\documentclass[final,3p,times]{elsarticle}

\usepackage{amssymb}
\usepackage{amsmath}
\usepackage{bm}
\usepackage{enumitem,amsthm,amsmath,amsfonts,amssymb}
\usepackage{indentfirst}
\usepackage{float}
\usepackage{natbib}
\usepackage{mathtools}
\usepackage{algorithmic}
\usepackage{algorithm}
\usepackage{xcolor}
\usepackage{graphicx}
\usepackage{subcaption}
\usepackage{tikz}
\usepackage{multirow}
\usepackage{multicol}
\usepackage{hyperref}
\usepackage{soul}
\usepackage{mathrsfs, fancyhdr}
\usepackage{amscd}
\usepackage{booktabs}

\newcommand{\Z}{\mathbb{Z}}

\newcommand{\N}{\mathbb{N}}

\newcommand{\D}{\mathcal{D}}
\newcommand{\E}{\mathbb{E}}

\newcommand{\R}{\mathbb{R}}
\newcommand{\F}{\mathcal{F}}

\newcommand{\calH}{\mathcal{H}}

\newcommand{\calB}{\mathcal{B}}

\newtheorem{cor}{Corollary}

\newtheorem{prop}{Proposition}
\newtheorem{rem}{Remark}

\def\bw{\mathbf{w}}
\def\O{\mathcal{O}}

\def\X{\mathcal{X}}
\def\Y{\mathcal{Y}}
\def\W{\mathcal{W}}

\def\Z{\mathcal{Z}}
\def\N{\mathcal{N}}
\def\L{\mathcal{L}}
\def\bx{\mathbf{x}}
\def\bz{\mathbf{z}}

\def\mbI{\mathbb{I}}

\def\S{\mathcal{S}}
\def\bw{\mathbf{w}}
\def\bW{\mathbf{W}}

\def\A{\mathcal{A}}

\def\gga{\gamma}
\def\ba{\mathbf{a}}

\def\0{\mathbf{0}}

\def\bH{\mathbf{H}}

\def\by{\mathbf{y}}

\def\bG{\mathbf{G}}

\newcommand{\bSigma}{\mathbf{\Sigma}}

\newtheorem{theorem}{Theorem}
\newtheorem{lemma}[theorem]{Lemma}
\newtheorem{proposition}[theorem]{Proposition}
\newtheorem{corollary}[theorem]{Corollary}
\theoremstyle{definition}
\newtheorem{definition}{Definition}

\newtheorem{assumption}{Assumption}
\newtheorem{example}{Example}
\theoremstyle{definition}
\newtheorem{property}{Property}
\newtheorem{remark}{Remark}

\def\begeqn{\begin{equation}}
\def\endeqn{\end{equation}}
\def\begth{\begin{theorem}}
\def\endth{\end{theorem}}
\def\begprop{\begin{proposition}}
\def\endprop{\end{proposition}}
\def\begcor{\begin{corollary}}
\def\endcor{\end{corollary}}
\def\begdef{\begin{definition}}
\def\enddef{\end{definition}}
\def\beglemm{\begin{lemma}}
\def\endlemm{\end{lemma}}
\def\begexm{\begin{example}}
\def\endexm{\end{example}}
\def\begrem{\begin{remark}}
\def\endrem{\end{remark}}

\def\begdef{\begin{definition}}
\def\enddef{\end{definition}}

\def\bw{\mathbf{w}}
\def\bz{z}

\def\O{\mathcal{O}}

\def\bE{\mathbf{E}}

\def\A{\mathcal{A}}
\def\Z{\mathcal{Z}}
\def\R{\mathbb{R}}

\def\X{\mathcal{X}}
\def\calE{\mathcal{E}}
\def\Y{\mathcal{Y}}
\def\Z{\mathcal{Z}}
\def\W{\mathcal{W}}

\def\mbI{\mathbb{I}}
\def\bD{\mathbf{D}}
\def\bV{\mathbf{V}}
\def\bfI{\mathbf{I}}

\def\ebb{\mathbb{E}}

\def\bK{\mathbf{K}}
\def\bxi{\boldsymbol{\xi}}
\def\bfnc{\mathbf{f}}
\def\H{\mathcal{H}}

\def\bS{\mathbf{S}}

\def\bL{\mathbf{L}}

\usepackage{amssymb}

\journal{}

\begin{document}

\begin{frontmatter}



\title{Optimal Rates for Generalization of Gradient Descent Methods with Deep Neural Networks}

\author{Junyu Zhou$^{1}$, Puyu Wang$^{2}$, Yunwen Lei$^{3}$, Yiming Ying$^{4*}$ and Ding-Xuan Zhou$^4$}

\address{$^{1}$  Mathematical Institute for Machine Learning and Data Science,  KU Eichstätt-Ingolstadt, Ingolstadt, Germany\\
$^{2}$ Department of Computer Science, RPTU Kaiserslautern-Landau, Kaiserslautern, Germany\\
$^{3}$ Department of Mathematics, University of Hong Kong, Hong Kong, China\\
$^{4}$ School of Mathematics and Statistics, University of Sydney, Sydney, Australia}

\begin{abstract}
Recent progress has been made in understanding the statistical generalization performance of gradient descent methods for overparameterized neural networks within the neural tangent kernel (NTK) regime. However, most of the existing work on regression problems is limited to shallow network architectures, leaving a notable gap in the theory of deep neural networks. This paper addresses this gap by presenting a comprehensive generalization analysis for deep ReLU networks trained using gradient descent (GD) and stochastic gradient descent (SGD). Specifically, we establish the first known minimax-optimal rates of excess population risk for both GD and SGD with deep ReLU networks, under the assumption that the network width scales polynomially with respect to the network depth and training sample size. Our results demonstrate that with sufficient width, gradient descent methods for deep ReLU networks can achieve optimal generalization rates on par with kernel methods.
\end{abstract}

\begin{keyword}
Deep ReLU networks \sep Gradient descent methods \sep Generalization analysis \sep  Neural tangent kernel \sep
Minimax-optimal rates


\end{keyword}

\end{frontmatter}



 
\section{Introduction}
Deep neural networks (DNNs) trained with gradient descent methods have achieved a remarkable success across a wide range of applications, including computer vision, natural language processing, and speech recognition \citep{bahdanau2014neural, hinton2012deep, krizhevsky2017imagenet, silver2016mastering}. 
Despite their highly nonconvex and overparameterized nature, DNNs can achieve a near-zero training error while still generalizing well to unseen data \citep{zhang2021understanding}.
To demystify this phenomenon, an extensive amount of work has been done to  understand the generalization and optimization properties of  gradient descent methods for training  DNNs.

The neural tangent kernel (NTK), introduced by \cite{jacot2018neural}, has emerged as a powerful framework for understanding the generalization performance of overparameterized neural networks trained using gradient descent methods. 
It reveals that, in the infinite-width limit, the training trajectory of a neural network with random initialization closely mirrors the behavior of its counterpart in the reproducing kernel Hilbert space (RKHS) associated with the NTK. 
This connection effectively bridges the gap between learning with DNNs and classical kernel methods, allowing insights from the kernel methods to inform our understanding of DNNs.

Following this perspective, the global convergence of gradient descent methods with DNNs has been extensively studied \cite{allen2019convergence,du2019gradient,zou2018stochastic}, while their generalization properties have only been investigated in a few works \cite{cao2019generalization,cao2020generalization,chen2019much,xu2024overparametrized}. 
Specifically, the appealing work  \cite{cao2019generalization} and \cite{cao2020generalization} developed algorithm-dependent misclassification error bounds for deep ReLU networks trained by gradient descent (GD) and stochastic gradient descent (SGD), respectively. 
\cite{chen2019much} relaxed the requirement of their network width for both GD and SGD. 
However, all of these works focused on classification problems under data separation assumptions. 
Very recently, the work \cite{xu2024overparametrized} studied one-pass SGD in the streaming (continuously coming) data setting with deep ReLU networks for regression problems and showed that the prediction error of one-pass SGD 
for deep ReLU networks can converge to zero in expectation, provided that the width of the network scales exponentially with the number of layers. 

On another important front, it is well-established in the kernel methods literature \citep{dieuleveut2016nonparametric, lin2017optimal, yao2007early, ying2006online} that for least squares regression,   GD and  SGD in RKHS can achieve the minimax-optimal rates in the sense that  the excess population risk is of the form $\mathcal{O}(n^{-\frac{2\beta}{2\beta+\gamma}})$, under standard regularity assumptions on the regression function and capacity assumptions associated with the RKHS \citep{caponnetto2007optimal, lin2017distributed}.
Here, $n$ is the size of the training data, $\beta>0$ is the smoothness of the target function $f_\rho$ and $\gamma\in[0,1]$ is a parameter that measures the capacity of the hypothesis space. 

\medskip

\noindent\textbf{Motivation.} Since the NTK perspective provides a close connection between the two learning processes of neural networks and kernel methods trained by gradient descent methods, it is natural to expect that neural networks trained by GD and SGD exhibit generalization performance (measured by excess population risk) comparable to their kernel-based counterparts.
This conjecture has been partially validated for shallow neural networks when the width is large enough. In particular,   \cite{nguyen2024many,nitanda2021optimal} demonstrated that GD and one-pass SGD for two-layer neural networks with {\em smooth activation functions} can replicate the classical results in the kernel setting, achieving the excess risk rates $\O(n^{-\frac{2\beta}{2\beta+\gamma}})$. 
However, a critical open question remains:
\begin{center}
\noindent\textit{Can DNNs trained by GD and SGD achieve minimax-optimal excess risk rates on par with their kernel-based counterparts?} 
\end{center}
In this work, we provide an affirmative answer to this question, significantly advancing the theoretical understanding of generalization of GD and SGD for training DNNs.

\begin{table*}[t]
{\small{ 
\begin{center}
\begin{tabular}{ ccccccc  } 
 \hline
 \multirow{2}{*}{\textbf{Work}} &\multirow{2}{*}{\textbf{Method}} &  \multirow{2}{*}{\textbf{Activation}}  &  \multirow{2}{*}{\textbf{Layer}}   &  \multirow{2}{*}{\textbf{Setting}}  &\multirow{2}{*}{\textbf{Excess risk}\!}  &  \multirow{2}{*}{\textbf{Width}} \\
 &&&&&\\
 \hline
 \multirow{2}{*}{\cite{nguyen2024many}}& \multirow{2}{*}{ GD }  & \multirow{2}{*}{ smooth } 
 &  \multirow{2}{*}{ shallow }  &  \multirow{2}{*}{$\beta > 0$ and $2\beta + \gamma  > 1$}&  \multirow{2}{*}{$ {\O}(  n^{-\frac{2\beta}{2\beta + \gamma}} )$} &\multirow{2}{*}{$\Omega(\text{Poly}(n,d))$}\\
   &&&&&\\
 \hline 
 \multirow{2}{*}{ \cite{kuzborskij2022learning} }& \multirow{2}{*}{ GD } &\multirow{2}{*}{ ReLU }& \multirow{2}{*}{ shallow } &   \multirow{2}{*}{$f_\rho$ is Lipschitz}&  \multirow{2}{*}{$ {\O} (n^{-\frac{2}{2+d}} )$}  &\multirow{2}{*}{$\Omega(\text{Poly}(n,\frac{1}{\lambda_0}) )$} \\
  &&&&&\\
   \hline 
 \multirow{2}{*}{ \cite{nitanda2021optimal} }&  \multirow{2}{*}{one-pass SGD} &\multirow{2}{*}{ smooth }& \multirow{2}{*}{ shallow } &   \multirow{2}{*}{$\beta\in[1/2,1]$}&  \multirow{2}{*}{$ {\O}(  n^{-\frac{2\beta}{2\beta + \gamma}} )$}  &\multirow{2}{*}{$\Omega(\exp(n))$} \\
  &&&&&\\
  \hline 
 \multirow{2}{*}{Ours}  & \multirow{2}{*}{GD/SGD}& \multirow{2}{*}{ReLU}& \multirow{2}{*}{deep} & \multirow{2}{*}{$\beta >  0$ and $2\beta + \gamma  > 1$} &  \multirow{2}{*}{$ {\O}(  n^{-\frac{2\beta}{2\beta + \gamma}} )$} & \multirow{2}{*}{$\Omega(\text{Poly}(L, n, d) )$}  \\
 &&&&&\\
  \hline 
\end{tabular}
\end{center}
}\caption{Results of GD and SGD for least square regression. Here, $f_\rho$ is the target function. 
\cite{kuzborskij2022learning} required the smallest eigenvalue of the NTK Gram matrix $\lambda_0 >0$.
\label{table:summary}} } 
\end{table*}

\medskip

\noindent {\bf Main contributions. } This paper extends the results for GD and SGD from shallow neural networks to DNNs while maintaining minimax-optimal excess risk rates under mild overparameterization conditions. Our contributions can be summarized as follows: 
\begin{itemize}[leftmargin=2mm]
\item We provide a comprehensive generalization analysis for deep ReLU networks trained with gradient descent methods for regression problems. 
For an $L$-layer ReLU network with a sufficiently large width $m$, we show that both GD and SGD can replicate the classical results in the kernel setting with the same gradient complexity under similar assumptions. 
Here, gradient complexity is the number of times the algorithm calculates the gradient. 
Specifically, we develop the minimax-optimal excess population risk rates $\O(n^{-\frac{2\beta}{2\beta+\gamma}})$ for both GD and SGD with deep ReLU networks when $m$ depends polynomially on $L, n$ and $d$ without imposing the commonly used assumptions on the Gram matrix of the NTK, where  $d$ is the dimension of the data. 

\item We improve the requirement of the network width in \cite{xu2024overparametrized} from exponential scaling in the number of layers $L$ to polynomial scaling. This relaxation has been achieved in \cite{zou2020gradient, chen2019much} for the classification setting by establishing favorable properties of the network at initialization. However, their methods cannot be directly applied to our setting, as we require these properties to hold uniformly over the entire input space, while their results are typically limited to the finite training dataset.

\item In particular, this is the first work to overcome the technical challenges of achieving the optimal excess risk rates from shallow to deep architectures within the NTK regime, demonstrating that deep ReLU networks trained with GD and SGD can achieve generalization performance on par with their kernel-based counterparts. Table~\ref{table:summary} summarizes the related results of GD and SGD for regression. 
\end{itemize}

\subsection{Technical Novelty }  
The minimax-optimal excess risk rates for shallow neural networks trained with gradient descent methods have been established in \cite{nitanda2021optimal,nguyen2024many}.
However, their approaches cannot extend directly to DNNs.
To analyze the excess population risk of $f_{\bW(T)}$, which represents the performance of a network $f_\bW$ at the output of GD/SGD with $T$ iterations,  \cite{nitanda2021optimal,nguyen2024many} employ the error decomposition 
$$\big\|f_{\bW(T)} - f_\rho\big\|^2_\rho \lesssim \big\|f_{\bW(T)} - f^{\text{lin}}_{\bW(T)}\big\|^2_\rho + \big\|f^{\text{lin}}_{\bW(T)}- g^m_T\big\|^2_\rho + \big\| g^m_T - h^m \big\|_\rho^2+\big\| h^m-f_\rho \big\|_\rho^2,$$
where $f^{\text{lin}}_{\bW(T)}$ is the linear approximation of $f_{\bW}$ at the Gaussian initialization $\bW(0)$, $g^m_T$ is GD/SGD associated with the finite-width NTK $K^m$ within the RKHS $\H_m$, $h^m$ is either the minimizer of the regularized population risk over $\H_m$ \cite{nitanda2021optimal} or the GD for the population risk in $\H_m$ \cite{nguyen2024many}, and $f_\rho$ is the target function.  
A critical step in their analysis is to control the term $\| g^m_T - h_m \|_\rho^2$ by $ n^{-\frac{2\beta}{2\beta + \tilde{\gamma}}}$, where $\tilde{\gamma}$ is the effective dimension of $\H_m$. 
To achieve minimax-optimal rates, it is essential to demonstrate that the effective dimension of $\H_m$ matches that of $\H_K$, i.e., $\tilde{\gamma} = \gamma$.
For $\gamma = 1$, this equivalence naturally holds since the integral operator associated with $K^m$ is a trace-class operator.  
For $\gamma<1$,  the argument is established by treating $K^m$ as a sum of i.i.d. random kernels with mean $K$ (see Proposition B in \cite{nitanda2021optimal} and Proposition A.18 in \cite{nguyen2023random}).
Specifically, for a two-layer ReLU network $f_\bW(\bx)= \frac{1}{\sqrt{m}} \sum_{r=1}^m a_r \sigma(\bw_r^\top \bx)$, a kernel has the  structure $K^m(\bx,\bx') =  \sum_{r=1}^m\langle  \partial_{\bw_r} f_{\bW(0)}(\bx), \partial_{\bw_r} f_{\bW(0)}(\bx') \rangle_2$. 
Since $ \partial_{\bw_r} f_{\bW(0)}(\bx) = m^{-\frac{1}{2}}a_r\sigma'(\bw_r(0)^\top\bx)$ depends only on the initial i.i.d. Gaussian weight $\bw_r(0)$,  $K^m$ is a sum of $i.i.d.$ random kernels.
However, this structure is not valid for deep ReLU networks. In deeper architectures, the gradient $\partial_{\bw^\ell_r} f_{\bW(0)}(\bx) $ is influenced not only by the weight $\bw^\ell_r(0)$ of the $\ell$-th layer but also by the  weights of all preceding and former layers. 
This interdependence makes a direct extension significantly more challenging.

To overcome this challenge and establish minimax-optimal rates for any  $\gamma \in [0,1]$, we adopt a refined error decomposition by introducing a new GD iterate $g_T$ associated with the infinite-width NTK $K$ in the RKHS $\H_K$.
Specifically, we decompose the error as
\begin{center}
    $\big\|f_{\bW(T)}-f_\rho\big\|_\rho^2\lesssim \big\|f_{\bW(T)} - f^{\text{lin}}_{\bW(T)}\big\|^2_\rho + \big\|f^{\text{lin}}_{\bW(T)}- g^m_T\big\|^2_\rho+ \big\|  g^m_T-g_T\big\|^2_\rho  +\big\| g_T-f_\rho \big\|_\rho^2.$
\end{center}
 
\noindent\textit{First term:} The first term $\|f_{\bW(T)}- f^{\text{lin}}_{\bW(T)} \|_\rho^2$ depends critically on forward and backward propagation estimates at random initialization. 
The work \cite{xu2024overparametrized} provided such estimates with upper bounds that scale exponentially with the network depth $L$ for deep ReLU networks. 
Applying these estimates to our setting will lead to the unexpected bound $\|f_{\bW(T)}- f^{\text{lin}}_{\bW(T)} \|_\rho^2\lesssim C^L m^{-\frac{1}{3}} \text{Poly}(\eta T)$ valid when $m\gtrsim C^L  \text{Poly}( \eta T,d)$.
Here, $C>1$ is an absolute constant and  $\eta>0$ is the step size.
In contrast, by extending the estimates of forward and backward propagation in \cite{zou2020gradient, chen2019much} from the finite training set $S$ to the full input space $\X$, 
we obtain the improved bound $\|f_{\bW(T)}- f^{\text{lin}}_{\bW(T)} \|_\rho^2\lesssim   m^{-\frac{1}{3}} \text{Poly}(L, \eta T)$ with a relaxed requirement $m \gtrsim \text{Poly}(L, \eta T, d)$,  reducing the dependence on the network depth from exponential to polynomial.  

\smallskip

\noindent\textit{Second and fourth terms: } The second term $\|f^{\text{lin}}_{\bW(T)}- g^m_T\|^2_\rho$ can be controlled by the first term and the gap between the gradients of $f_{\bW}$ at initialization $\bW(0)$ and at $\bW(T)$, while the final term, $\| g_T-f_\rho \|_\rho^2$, is bounded using standard results for kernel methods \cite{lin2017optimal}. 

\smallskip

\noindent\textit{Third term:} The \textit{most challenging} term,  $\| g^m_T - g_T \|_\rho^2$,   requires a more nuanced analysis.
 Since $g^m_T$ and $g_T$ lie in different RKHSs, directly controlling the $L_2$-norm of their difference is difficult. Then, we consider estimate their difference in the infinity norm by leveraging the uniform convergence properties of the NTK $K^m$. 
A key insight here is that the infinity norm between $g^m_T$ and $g_T$ can be controlled by that of the corresponding kernels $K^m$ and $K$, yielding  $\| g^m_T -  g_T\|_\rho^2\lesssim \| g^m_T -   g_T\|_\infty^2\lesssim (\eta T)^4\| K^m-K\|_\infty^2$.  
\cite{xu2024overparametrized} proved that $\|K^m-K\|_\infty\lesssim C^Lm^{-\frac{1}{6}}\sqrt{d}$ assuming exponential scaling of $m$ with $L$. 
By applying a more refined analysis (see Lemma~\ref{lem:concentrationNTK} in Section~{Proof}), we show $\|K^m-K\|_\infty\lesssim m^{-\frac{1}{6}}\sqrt{dL}$ under a relaxed condition $m\gtrsim dL^3$. 
This improvement significantly relaxes the overparameterization requirements, completing the estimation.
Further details can be found in Proposition~\ref{pro:gm-gT}. Note that if $f_{\bW(T)}$ is produced by SGD, the estimate strategy for the other three terms remains unchanged. There will be an additional error in the third due to the discrepancy between the SGD and GD iterates in the RKHS $\H_m$, which can be estimated using the results in \cite{lin2016optimal}.

The combination of the refined error decomposition and the key insights  effectively extend the analysis from shallow to deep neural networks.

\subsection{Other Related Work} 
 
In this subsection, we review some further works which are closely related to our paper.

There has been a large amount of literature studying the generalization performance of gradient descent methods for overparameterized neural networks in the NTK regime \citep{arora2019fine,ji2019polylogarithmic,nitanda2019gradient,nitanda2021optimal,nguyen2024many,xu2024overparametrized}.
The generalization performance of both GD and SGD have been well studied for the classification problems under some certainty assumptions on data distribution \cite{brutzkus2018sgd,cao2019generalization,cao2020generalization,li2018learning,ji2019polylogarithmic,nitanda2019gradient}. 
For example, \cite{brutzkus2018sgd} analyzed the misclassification error of SGD for training two-layer neural networks with Leaky ReLU activation.
\cite{li2018learning} focused on learning two-layer ReLU networks with SGD and demonstrated that small misclassification error can be achieved when the training data is drawn from mixtures of well-separated distributions.
More recently, \cite{ji2019polylogarithmic} and \cite{nitanda2019gradient} studied the generalization performance of two-layer networks with ReLU and smooth activations, respectively, showing that GD and SGD can achieve small misclassification error under separation margin assumptions.
For regression problems, \cite{arora2019fine} established data-dependent generalization bounds using the Rademacher complexity, under the assumption that the NTK Gram matrix has favorable properties. 
Building on this assumption, \cite{kuzborskij2022learning} derived a generalization bound of order $\O(n^{-\frac{2}{2+d}})$ for GD trained using two-layer ReLU networks when learning target functions with uniformly bounded, Lipschitz additive noise.  
Recently, \cite{xu2024overparametrized} proved that the prediction error of one-pass SGD for deep ReLU networks can converge to zero in expectation, while the explicit rate is not given. 
The works most related to ours are \cite{nguyen2024many,nitanda2021optimal}, where the minimax-optimal rates $\O(n^{-\frac{2\beta}{2\beta+\gamma}})$ of the excess population risk are provided for GD and one-pass SGD respectively.
While they only focused on shallow neural networks with smooth activation functions.

An important approach for analyzing the generalization bounds of neural networks is the uniform convergence framework, which uses capacity measures like Rademacher complexity and covering numbers to control the capacity of the hypothesis space \citep{bartlett2017spectrally,frei2023random,golowich2018size,neyshabur2015norm,parhi2022near,lei2026optimization}.
More recently, several studies have adopted algorithmic stability to assess the generalization performance of gradient descent methods for neural networks \cite{lei2022stability,taheri2024generalization,taheri2024sharper,wang2026optimization,wang2026population}. For example, the excess population risk bounds of order $\O(1/\sqrt{n})$ were established in \cite{lei2022stability,wang2025generalization} for both GD and SGD in shallow neural networks with least-squares regression.
The work~\cite{taheri2024generalization,taheri2024sharper} considered the classification setting, providing generalization analyses for GD in shallow and deep neural networks, respectively. \cite{wang2026population,wang2026population} extended their results to two-layer Kolmogorov–Arnold networks under the similar settings.
However, these works were limited to smooth activation functions.

\section{Problem Formulation}\label{sec:problem-formulation}
Let $\X\subseteq \R^d$ be the input space,   $\Y\subseteq \R$ be the output space, and $\Z=\X\times \Y$.
For simplicity, for any $\bx\in\X$ and $y\in\Y$,  we assume $\|\bx\|_2=1$ and $|y|\le 1$, where $\|\cdot\|_2$ is the standard Euclidean norm.  
Let $\rho$ be a probability measure on $\Z$. 
Denote by $S=\{z_i=(\bx_i,y_i):i=1,\ldots,n\}$ a training dataset drawn from the unknown distribution $\rho$. Based on $S$, we aim to build a predictor $f:\X \rightarrow \R$, whose performance is measured by the expected risk 
$ \L(f):=\frac{1}{2}\E_{(\bx,y)\sim \rho} [(y -f(\bx))^2 ].$ 
Since the distribution $\rho$ is unknown in practice, we instead minimize the empirical risk defined by
$\L_S(f):=\frac{1}{2n} \sum_{i=1}^n ( y_i-f(\bx_i)  )^2$. 
A minimizer of the expected risk over all measurable functions is the regression function $f_\rho(\bx):=\E[y|\bx]$, where $\E[\cdot|\bx]$ denotes
the conditional expectation given $\bx$. 

In this paper, we are interested in a prediction model $f$ parameterized by $\bW$ in some parameter space $\W$ with a neural network structure. 
In particular, we focus on  $L$-layer deep ReLU neural networks with width $m$ of the form
\begin{align}\label{eq:NN}
    f_\bW(\bx)=\ba^\top \sqrt{\frac{2}{m}} \sigma\,\Big(\, \bW^L \cdots \sqrt{\frac{2}{m}}\,\sigma\,\big( \bW^1 \bx\big)\,\Big),  
\end{align}
where $\bx\in \X$ is the input,  $\sigma(\cdot)=\max\{\cdot, 0\}$ is the ReLU activation, $\bW=\big( \bW^1,\ldots, \bW^L \big)\in\W$ with $\W:=\R^{m\times d} \times (\R^{m\times m})^{L-1}$ denoting the collection of weight matrices for all layers,  and $\ba=(a_1,\ldots,a_m)^\top\in\R^{m}$ is the weight vector of the output layer.
In the above formulation, $\bW^1\in\R^{m \times d}$ and $\bW^\ell\in \R^{m\times m} $ for $\ell=2,\ldots,L$ is the weight of the $\ell$-hidden layer.
We denote $(\bw^\ell_r)^\top$ the $r$-th row of $\bW^\ell$ for $\ell\in[L]:=\{1,\ldots,L\}$. For the simplicity of argument, we assume $m$ is even and 
use the notations $\L(\bW)=\L(f_\bW)$, $\L_S(\bW)=\L_S(f_\bW)$, and the loss $ l(\bW;z)=\frac{1}{2}(y-f_\bW(\bx))^2$.

In this paper, we are concerned with two notable algorithms to minimize the empirical risk, i.e.,  GD and SGD. We will consider symmetric initialization of GD and SGD, which are widely used in the theoretical analysis of neural networks \cite{kuzborskij2022learning,  nguyen2024many, nitanda2021optimal, xu2024overparametrized}.  
Especially, we adopt Gaussian initialization for all weights while the weights of the last layer are initialized additionally using  the symmetric weights and uniform initialization for the output layer weight defined as follows:
\begin{align}\label{eq:initialization}
       & \textit{for the first }L-1 \textit{ layer: } \ \ \bw_r^1(0) \overset{\text{i.i.d.}}{\sim} \mathcal{N}(0, \mathbf{I}_d)  \textit{ and }  \bw_r^\ell(0) \overset{\text{i.i.d.}}{\sim} \mathcal{N}(0,   \mathbf{I}_m)  \textit{  for all } r\in[m],
  \nonumber \\
      & \textit{for the last layer: } \ \ \bw_r^L(0) \overset{\text{i.i.d.}}{\sim} \mathcal{N}(0, \mathbf{I}_m) \textit{ for } r\in\big\{1,\ldots, {m}/{2}\big\}, \textit{ and }
     \bw_{r+\frac{m}{2}}^L(0)=\bw_r^L(0),\\
     & \textit{for the output layer:} \ \  a_r\overset{\text{i.i.d.}}{\sim} \textit{Unif }(\{-1,1\}) \textit{ for } r\in\big\{1,\ldots, {m}/{2}\big\},  \textit{ and } a_{r+\frac{m}{2}}=-a_r.\nonumber
    \end{align}
Symmetric initialization is mainly used to ensure that the initial function $f_{\bW(0)}(\bx) = 0$ for any $\bx\in\X$, which simplifies theoretical analysis. As noted in \cite{nguyen2024many,nitanda2021optimal}, 
this requirement can be relaxed by taking into account the additional error caused by non-symmetric initialization.  
Moreover, this symmetric trick does not affect the concentration properties of the NTK for deep ReLU networks (see the discussion in Section~\ref{sec:NTK}).
For a differentiable function $F$ on $\W$, we denote 
$$\partial F(\bW_0 ) =\frac{\partial F(\bW )}{\partial \bW }\Big|_{\bW=\bW_0} \ \text{ and } \ \partial_\ell F(\bW_0 ) =\frac{\partial F(\bW )}{\partial \bW^\ell }\Big|_{\bW=\bW_0} \ \text{ for all } \ \ell\in[L].$$
\begin{definition}[Gradient Descent]\label{GD}
    Let $\bW(0)\in\W$ be the initialization generated by \eqref{eq:initialization} and $\eta >0$ be the step size. 
    GD updates $\big\{\bW(k): k\in \mathbb N\big\}$ by
\begin{align}\label{eq:update-GD}
    \bW (k+1)&=\bW (k) -\eta  \partial  \L_S\big(\bW(k)\big) .
\end{align}  
\end{definition}

\begin{definition}[Stochastic Gradient Descent]\label{def:SGD}
Let $\bW(0)\in\W$ be the initialization generated by \eqref{eq:initialization} and $\eta >0$ be the step size.  
SGD updates $\big\{\bW(k): k\in \mathbb N\big\}$ by
\begin{equation}\label{eq:sgd-update}
    \bW(k+1)=\bW(k)-\eta \partial  l\big(\bW(k);z_{i_k}\big), \text{ where }  {i_k}  \text{ is uniformly drawn from } [n] .
\end{equation} 
\end{definition}

We are interested in the generalization performance of a model $f_\bW$ trained by GD and SGD with $T$ iterations, measured in terms of the \textit{excess population risk} 
$\varepsilon_{risk} (f_{\bW(T)} )=\L (f_{\bW(T)} ) - \L(f_\rho)$, i.e., the discrepancy between the expected risks of $f_{\bW(T)}$ and the target function $f_\rho$. 
For the least square regression, 
it has been shown in \cite{cucker2007learning} that $\varepsilon_{risk} (f_{\bW(T)} )$ can be further cast as 
$$\varepsilon_{risk} \big(f_{\bW(T)} \big) = \frac{1}{2} \big\|  f_{\bW(T)} - f_{\rho} \big\|_\rho^2.$$
Here,  $\|\cdot\|_\rho$ is the $L_2$-norm defined as $\|f\|_\rho=\big(\int_\X |f(\bx)|^2 d \rho_{\bx} (\bx)\big)^{1/2}$ where $\rho_\bx$ denotes the marginal distribution of $\rho$ on $\X$.

In the remainder of the paper, we focus on studying $ \|  f_{\bW(T)} - f_{\rho} \|_\rho^2$. The key idea of the analysis is to introduce kernel methods as a bridge between the neural network and the best model $f_\rho$. 
To this end,  we require the concept of the neural tangent kernel (NTK) \citep{jacot2018neural}. 
In our setting, the NTK $K:\X\times\X\to\R$  with symmetric initialization is defined, for any $\bx,\bx'\in\X$, by
\begin{align}\label{eq:K}
    K(\bx,\bx') =  {2}\ebb\big[\sigma(U^{L-1}(\bx))\sigma(U^{L-1}(\bx'))\big] q^L(\bx,\bx'),
\end{align}
where $ \big\{\big(U^{\ell}(\bx),U^{\ell}(\bx')\big)\big\}_{\ell=1}^{L-1} $ are pairs of bivariate normal variables defined iteratively by $$\big(U^{\ell}(\bx),U^{\ell}(\bx')\big)\sim \N\big(0,\Sigma^{\ell-1}(\bx,\bx')\big)$$
with  
\begin{align*}
    \Sigma^{\ell - 1} (\bx, \bx')  =  2 \left(\begin{aligned}
        &\ebb\big[\sigma^2(U^{\ell -  1} (\bx) ) \big] && \!\!\!\!\!\!\! \ebb\big[\sigma(U^{\ell - 1}(\bx) )\sigma(U^{\ell - 1}(\bx') ) \big]\\
         \ebb\big[& \sigma (U^{ \ell - 1} (\bx) ) \sigma( U^{ \ell - 1} (\bx') ) \big] &&\ \  \ebb\big[\sigma^2(U^{\ell -  1}(\bx') ) \big]
    \end{aligned}
     \right), \ \ \ 
    \Sigma^{0}(\bx, \bx') =   \left(\begin{aligned}
        &1 && \!\!\!\!\!\!\langle \bx, \bx'\rangle_2\\
         \langle \bx&, \bx'\rangle_2 &&  1
    \end{aligned}
    \!\right) 
\end{align*}
and 
\begin{align*}
    q^L(\bx,\bx') = \frac{\pi - \arccos\big(p^{L-1}(\bx,\bx')\big)}{\pi} \ \text{ with } \ p^{L-1}(\bx,\bx') =  \frac{\ebb\big[\sigma(U^{L-1}(\bx)) \sigma(U^{L-1}(\bx')) \big]}{\big(\ebb\big[\sigma^2(U^{L - 1}(\bx)) \big]\ebb\big[\sigma^2(U^{L - 1}(\bx')) \big]\big)^{{1}/{2}}}.
\end{align*}

Note that for all $\bx,\bx\in\X$ and $\ell\in[L]$, $\ebb[\sigma(U^\ell(\bx))\sigma(U^\ell(\bx')) ]$ is deterministic, and does not involve any randomness.
Then, from the definition we know that the NTK $K$ is also deterministic.

Let $\mathcal{H}_{K}$ be the RKHS associated with the kernel $K$, with inner product and induced norm denoted by $\langle \cdot, \cdot   \rangle_K$ and $\|\cdot\|_K$, respectively. Let $\mathcal{L}_{\rho_\bx}^2 = \big\{f:\X\to\R: \|f\|_\rho < \infty\big\}$ be the space of square-integrable functions on $\X$ with respect to $\rho_\bx$. 
We introduce the integral operator $\bL: \mathcal{L}^2_{\rho_\bx} \to \mathcal{L}^2_{\rho_\bx}$, defined by $\bL f= \int_{\X} K(\cdot, \bx) f(\bx) d\rho_\bx(\bx)$. 
One can show $\int_\X K(\bx,\bx)d\rho_\bx(\bx) \le 1$ (see Property~\ref{prop:bound-K} in  Section~\ref{sec:proof-conce}), hence $\bL$ is a compact, 
self-adjoint and positive operator, which has the eigen-decomposition $\bL  f=\sum_{i=1}^\infty \lambda_i \langle f, \Phi_i\rangle_{\mathcal{L}^2_{\rho_\bx}} \Phi_i$. Here, $\{(\lambda_i,\Phi_i)\}$ are pairs  of eigenvalues
and orthogonal eigenfunctions in $\mathcal{L}^2_{\rho_\bx}$ with $\lambda_1 \ge \lambda_2\ge\ldots\ge 0$, and $\{\Phi_i\}_{i=1}^\infty$ forms an orthonormal basis of $\mathcal{L}^2_{\rho_\bx}$. For $s\in \R$, we define the power $\bL^{s}$ as, for any $f\in \mathcal{L}^2_{\rho_\bx}$,  $ \bL^{s}(f) =\sum_{i=1}^\infty \lambda_i^s \langle f, \Phi_i\rangle_{\mathcal{L}^2_{\rho_\bx}} \Phi_i$.  For a bounded and positive linear operator $A$ on a separable Hilbert space $\mathcal{H}$ with orthonormal basis $\{e_i\}_{i=1}^\infty$, the trace of $A$ is defined by $tr(A) = \sum_{i=1}^{\infty}\langle Ae_i, e_i\rangle_{\mathcal{H}}$.

To analyze the performance of kernel methods, we impose the following standard assumptions on the capacity of the hypothesis space and the complexity of $f_\rho$.  
\begin{assumption}[Effective dimension]\label{ass:effec_dim}
For some $\gamma \in [0,1]$ and $c_\gamma \ge 1$, there holds 
     $ tr\big( \bL(\bL+\lambda\mathbf{I})^{-1} \big)=\sum_{i=1}^\infty\frac{\lambda_i}{\lambda_i+\lambda}\le c_\gamma\lambda^{-\gamma}  \text{ for all }  \lambda>0. $
\end{assumption}
In the above assumption, the quantity $tr\big( \bL(\bL+\lambda\mathbf{I})^{-1} \big)$ is called as the effective dimension \cite{caponnetto2007optimal} or the degrees of freedom \cite{zhang2005learning}. 
Note that $\bL$ is a trace class operator satisfying $tr(\bL)=\sum_i \lambda_i = \int_\X K(\bx,\bx)d\rho_\bx(\bx) \le 1$, then Assumption~\ref{ass:effec_dim} is always true for $\gamma=1$ and $c_\gamma=1$.
In this sense, the case of $\gga=1$ is often  referred  to as the capacity independent setting \cite{cucker2007learning}. Assumption~\ref{ass:effec_dim} holds true if $\bL$ is of finite rank (corresponds to $\gamma=0$) or the eigenvalues $\{\lambda_i\}$ satisfy a polynomial decaying condition $\lambda_i\lesssim i^{-1/\gamma}$ for $\gamma \in (0,1]$.
The specific decay rates of the eigenvalues have been studied for some specific settings \citep{bach2017breaking,bietti2019inductive,bietti2020deep,hu2021regularization,scetbon2021spectral}. 
For example, under the assumption that the input $\bx$ is uniformly distributed on a unit sphere,  \cite{hu2021regularization} showed that the eigenvalues of the NTK associated with two-layer ReLU networks decay as $\lambda_i\asymp i^{-\frac{d}{d-1}}$.  
\begin{assumption}[Source condition]\label{ass:frho_smth}
    There exist $\beta > 0$ and $B>0$, such that $\|\bL^{-\beta}f_{\rho}\|_\rho\le B$. 
\end{assumption}
Assumption~\ref{ass:frho_smth} is commonly used in nonparametric regression \cite{cucker2002mathematical}, which quantifies the smoothness (regularity) of the regression function $f_\rho$. The larger the value of $\beta$, the smoother $f_\rho$ becomes and, consequently, the more stringent the assumption. In particular, if $\beta=1/2$, then this assumption indicates the requirement of $f_\rho\in \H_K$, which implies that there exists at least one minimizer of population risk belonging to the RKHS $\H_K$.

\section{Main Results}
 
Before presenting our main results, we first introduce some necessary definitions and notations.  
Given the initialization $\bW(0)$, define the feature map $\Phi_m:\X \to \W$   by
$$\Phi_m(\bx) = \partial f_{\bW(0)}(\bx) = \big( \partial_1 f_{\bW(0)}(\bx),\ldots, \partial_L f_{\bW(0)}(\bx)  \big).$$ 
With this feature map, we define a PSD kernel $K^m : \X \times \X \to \R$ by
\begin{align}\label{eq:Km_original}
    K^m(\bx,\bx')=\big\langle  \Phi_m(\bx) ,  \Phi_m(\bx') \big\rangle_2 = \sum\nolimits_{\ell=1}^L \big\langle \partial_\ell f_{\bW(0)}(\bx),  \partial_\ell  f_{\bW(0)}(\bx') \big\rangle_2,
\end{align}
where $\langle\cdot,\cdot\rangle_2$ is the inner product of a vector or a matrix. Here, $K^m$ can be seen as a random feature approximation of the NTK $K$. 
According to Theorem 4.21 in \cite{steinwart2008support}, there exists a unique RKHS $\H_m$ associated with the kernel $K^m$ given by
$$\H_m = \big\{f: \X\to\R\ :\ \exists \ \bW \in \W \text{ such that } f(\bx) = \big\langle \bW, \Phi_m(\bx) \big\rangle_2 \text{ for all } \bx \in \X \big\},$$
whose corresponding norm is defined, for any $f\in \H_m$,  by 
  \begin{center}
      $  \|f\|_{\H_m} = \inf \Big\{ \big(\sum_{\ell=1}^L\|\bW^\ell\|_2^2\big)^{1/2} : \bW \in \W \text{ with } f(\bx) =  \big\langle \bW, \Phi_m(\bx) \big\rangle_2 \text{ for all } \bx \in \X \Big\}.$ 
\end{center}
We further define the linear approximation of $f_{\bW}$ at the initialization $\bW(0)$ by
$$f^{\text{lin}}_\bW(\bx) = f_{\bW(0)}(\bx) +  \big\langle \partial f_{\bW(0)}(\bx) , \bW - \bW(0)  \big\rangle_2.$$
Let $\bK = \big(K(\bx_i,\bx_j)\big)_{i,j=1}^n$ and $\bK^m = \big(K^m(\bx_i,\bx_j)\big)_{i,j=1}^n$ be the Gram matrices with kernels $K$ and $K^m$, respectively.
For a function $\psi$ on an arbitrary space $\Omega$, we define the inifinty norm of $\psi$ as $\|\psi\|_\infty := \sup_{\omega\in\Omega}|\psi(\omega)|$.

\subsection{Optimal Rates for Gradient Descent}\label{sec:gd-results}
Throughout the paper, we denote $M \lesssim M'$ if there exists a constant $c>0$ such that $M\le cM'$, and denote $M\asymp M'$ if both $M \lesssim M'$ and $M' \lesssim M$ hold. Define the functions $K^m_{\bx}\in\H_m,K_\bx\in \H_K$  by $K^m_{\bx}(\bx')=K^m(\bx,\bx')$ and $K_\bx(\bx') = K(\bx,\bx')$ for any $\bx,\bx'\in \X.$
If we regard the empirical risk $\L_S(\cdot)$ as a functional on the RKHS $\H_m$ and $\H_K$, the iteration of GD for least-square regression in $\H_m$ and $\H_K$ can be defined as
\begin{align}\label{eq:update-kernelGD}
    & g^m_{k+1} = g^m_k - \frac{\eta}{n} \sum\nolimits_{i=1}^n\big(g^m_k(\bx_i) - y_i\big)K^m_{\bx_i} \text{ for any $k\in\mathbb{N}$ with $g^m_0 = 0$}, \\\label{eq:kernel_GD_HK}
    &g_{k+1} = g_k - \frac{\eta}{n}\sum\nolimits_{i=1}^n\big(g_k(\bx_i) - y_i\big) K_{\bx_i} \text{ for any $k\in\mathbb{N}$ with $g_0 = 0.$}
\end{align}
Let $\bW(T)$, $g^m_T$ and $g_T$ be produced by \eqref{eq:update-GD},  \eqref{eq:update-kernelGD} and \eqref{eq:kernel_GD_HK} with $T$ iterations, respectively. 
Considering $f^{\text{lin}}_{\bW(T)}$, $g^m_T$ and $g_T$ as bridges connecting $f_{\bW(T)} $ and $ f_{\rho}$, we introduce the following error decomposition
\begin{align}\label{eq:decom-gd}
\varepsilon_{risk}\big(f_{\bW(T)}\big)  =\frac{1}{2}\big\|  f_{\bW(T)} - f_{\rho} \big\|_\rho^2 \lesssim  \big\|  f_{\bW(T)} - f^{\text{lin}}_{\bW(T)}\big\|_\rho^2  +  \big\|f^{\text{lin}}_{\bW(T)}  - \bS_m g^m_T\big\|_\rho^2  +  \big\|\bS_m g^m_T  -  \bS g_T\big\|_\rho^2  + \big\|\bS g_T  -  f_{\rho}\big\|_\rho^2, 
\end{align}
where $\bS_m:\H_m\xhookrightarrow{}\mathcal{L}^2_{\rho_\bx}$ and $\bS:\H_K\xhookrightarrow{}\mathcal{L}^2_{\rho_\bx}$ are the inclusion mappings that map $g^m_T\in \H_m$ to $\bS_m g^m_T\in \mathcal{L}^2_{\rho_\bx}$ and $g_T\in\H_K$ to $\bS g_T\in\mathcal{L}^2_{\rho_\bx}$, respectively. 
We will state the estimates for the above four terms in the subsequent context respectively and present the upper bounds of  $\varepsilon_{risk} (f_{\bW(T)} )$. 
We assume $\eta T\ge 1$ and denote $C\ge1$ as an absolute value which may differ from line to line.

We begin by estimating the term $\|f_{\bW(T)} - f^{\text{lin}}_{\bW(T)}\|_\rho^2$ on the right-hand side of \eqref{eq:decom-gd}. 
Since the population distribution $\rho$ is unknown, in the following proposition we employ the $\|\cdot\|_\infty$-norm to control the $\|\cdot\|_\rho$-norm of $f_{\bW(T)} - f^{\text{lin}}_{\bW(T)}$. 
In this sense, the established bound is the worst-case one which holds true for any population distribution $\rho$. 
The detailed proof is deferred to Section~\ref{proof:f-flin}.
\begin{proposition}\label{pro:f-flin}
    Let $\delta \in (0,1)$.
    Assume $\eta  \le  1/5$  and 
    \begin{align}\label{eq:m_condition}
    m\gtrsim L^{22}d^3(\eta T)^{7}\log^3(m/\delta).
    \end{align}
    Then, with probability at least $1-L\exp\big(\O(dL\log(m)) - \Omega(m^{\frac{1}{3}})\big) - \delta$ over the initialization $(\ba,\bW(0))$, there holds 
    $$ \big\|f_{\bW(T)} - f^{\text{lin}}_{\bW(T)} \big\|_\rho^2\le   \big\|f_{\bW(T)} - f^{\text{lin}}_{\bW(T)} \big\|_\infty^2 \lesssim  \frac{L^\frac{14}{3}(\eta T)^{\frac{4}{3}}}{m^{\frac{1}{3}}}.$$
\end{proposition}

\begin{remark}
The estimate of $\|f_{\bW(T)} - f^{\text{lin}}_{\bW(T)}\|_\rho^2$ is controlled by showing the trajectory of GD/SGD is always near the initialization, which critically depends on forward and backward propagation estimates at random initialization. 
The work \cite{xu2024overparametrized} provided such estimates with upper bounds depend exponentially on the number of layers $L$.
Applying these estimates to our approach leads to the unexpected bound  $\|f_{\bW(T)}- f^{\text{lin}}_{\bW(T)} \|_\rho^2\lesssim C^L m^{-\frac{1}{3}} \text{Poly}(\eta T)$ valid when $m\gtrsim C^L  \text{Poly}( \eta T,d)$. 
Meanwhile, \cite{zou2020gradient, chen2019much} conducted fine-grained analyses for forward and backward propagation for classification problems, significantly relaxing the required network width from an exponential to a polynomial scaling.
However, their approach cannot be directly applied to our setting, as we require these results to hold uniformly over the entire input space $\X$ to control the $\|\cdot\|_\infty$-norm, while their results are typically restricted to the training dataset $S$. 
We extend their results from the finite training set $S$ to the full input space $\X$, reducing the requirement of the width to a polynomial scaling. More details can be found in Lemmas~\ref{coro:bound-o} and \ref{lemma:o-o_0} in Section~\ref{proof}. 
\end{remark}

Proposition~\ref{pro:flin-gm} presents the estimation of $ \|f^{\text{lin}}_{\bW(T)} - \bS_mg^m_T \|_\rho^2$,  whose proof can be found in Section~\ref{proof:flin-gm}. 
\begin{proposition}\label{pro:flin-gm}
    Let $\delta\in(0,1)$.  
    Assume $\eta\le 1/5$, $  \eta T \le n(36\log(2n/\delta))^{-1}$ and  \eqref{eq:m_condition}. 
    Then, with probability at least $1-L\exp\big(\O(dL\log(m)) - \Omega(m^{\frac{1}{3}})\big) - \delta$ over the initialization $(\ba,\bW(0))$ and sampling, there holds
   $$ \big\|f^{\text{lin}}_{\bW(T)} - \bS_mg^m_T \big\|_\rho^2 \lesssim \frac{L^\frac{14}{3}(\eta T)^{\frac{7}{3}}}{ m^{\frac{1}{3}}}.$$
\end{proposition}

\begin{remark}
Propositions \ref{pro:f-flin} and \ref{pro:flin-gm} jointly demonstrate the almost equivalence of the GD trajectories for a deep ReLU network and for the corresponding NTK $K^m$ under overparameterization, i.e., $ \|f_{\bW(T)} - \bS_mg^m_T\|_\rho^2 \lesssim L^\frac{14}{3}(\eta T)^{\frac{7}{3}} m^{-\frac{1}{3}}$ under the condition $m\gtrsim \text{Poly}(L,d,\eta T)$. 
This implies the larger the width of the network $m$, the closer the two trajectories are and the more the behavior of $f_{\bW(T)}$ is similar to that of $g^m_T$. 
\cite{nitanda2021optimal} established this estimate for the trajectory of the SGD average stream. They showed that these two trajectories behave almost the same when the network width $m$ scales exponentially with $n$. Subsequently,  \cite{cao2024stochastic} reduced the requirement of $m$ to the polynomial degree. 
However, both of these two works are limited to two-layer networks with smooth activation. 
Our result demonstrates that even for deep networks with non-smooth ReLU activation, a polynomially large width is sufficient to ensure the alignment of the learning trajectories. 
\end{remark}

We estimate $ \|\bS_m g^m_T - \bS g_T \|_\rho^2$ in the following proposition.   
The proof is deferred to Section~\ref{proof:gm-gT}.
\begin{proposition}\label{pro:gm-gT}
Let $\delta\in(0,1)$.
    Assume  $\eta\le 1/4$ and \eqref{eq:m_condition}.
    Then, with probability at least $1-L\exp\big(\O(dL\log(m)) - \Omega(m^{\frac{1}{3}})\big) - \delta$ over the initialization $(\ba,\bW(0))$, there holds
    \begin{align*}
      \big\|\bS_mg^m_{T} -\bS g_{T}\big\|_\rho^2\le   \big\|g^m_{T} - g_{T}\big\|_\infty^2 \lesssim (\eta T)^4\big\|K^m-K\big\|^2_\infty \lesssim  \frac{L(\eta T)^4}{ m^{\frac{1}{3}}}.
    \end{align*}
\end{proposition}
\begin{remark}\label{rmk:distance two GDs} 
    The above proposition shows that the distance between the GD iterates in $\H_m$ and $\H_K$ can be controlled by the discrepancy between their respective kernels, $K^m$ and $K$.
    In fact, this result can be extended to any pair of RKHS with bounded kernels.
    Specifically, for arbitrary RKHS $\H_1,\H_2$ with bounded kernels $K^1,K^2$, let $g^1_T$ and $g^2_T$ denote the corresponding GD iterates (defined analogously to \eqref{eq:update-kernelGD} with $K^m$ replaced by $K^1$ and $K^2$), respectively.
    Then, it follows that $\|g^1_T - g^2_T\|_\infty \lesssim (\eta T)^2\|K^1-K^2\|_\infty$. In addition, the work \cite{xu2024overparametrized} proved that $\|K^m-K\|_\infty\lesssim C^Lm^{-\frac{1}{6}}\sqrt{d}$ valid when $m$ depends exponentially on $L$. 
We improved their bound to $\|K^m-K\|_\infty\lesssim m^{-\frac{1}{6}}\sqrt{dL}$ with the reduced condition $m\gtrsim \text{Poly}(L,n,d)$. More details  can be found in Lemma~\ref{lem:concentrationNTK} in Section~\ref{proof}. 
\end{remark}
Finally, we provide an estimate for the last term,  $\|\bS g_T  \!-\! f_{\rho}\|_\rho^2$, which captures the performance of GD within $\H_K$. The detailed proof is presented in Section~\ref{proof:gT-frho}. 
\begin{proposition}\label{pro:gT-frho}
    Suppose Assumptions \ref{ass:effec_dim} and  \ref{ass:frho_smth} hold.
    Assume $\eta \le 1$ and $  \eta T \le n(9\log(2n/\delta))^{-1}$. 
    Then, with probability at least $1-\delta$ over sampling, there holds  
        \begin{align*}
            \big\|\bS g_T - f_\rho\big\|_\rho^2 \lesssim\Big(   \frac{\eta T}{n^2}  +  \frac{(\eta T)^\gamma + (\eta T)^{1 - 2\beta}}{ n} \Big) \log^4\big( {T}/{\delta}\big) + \frac{1}{(\eta T)^{2\beta}}.
        \end{align*}
\end{proposition}
\begin{remark}\label{rmk:from shallow to deep}
Propositions \ref{pro:gm-gT} and \ref{pro:gT-frho} together provide an estimate for $\|\bS_mg^m_T -f_\rho\|_\rho^2$. 
In this remark, we highlight the technical novelty of our approach. 
For two-layer neural networks, previous work \cite{carratino2018learning,nitanda2019gradient} estimated this term by introducing an intermediate term $h^m\in \H_m$, separately bounding $\|\bS_m(g^m_T - h^m)\|_\rho^2$ and $\|\bS_mh^m - f_\rho\|_\rho^2$.
Here, $h^m$ is either the minimizer of the regularized population risk over $\H_m$ \cite{nitanda2021optimal} or the GD for the population risk in $\H_m$ \cite{nguyen2024many}.  
One can show that $\|\bS_m( g^m_T - h_m) \|_\rho^2\lesssim n^{-\frac{2\beta}{2\beta + \tilde{\gamma}}}$ with $\tilde{\gamma}$  the effective dimension of $\H_m$. 
Hence, to achieve optimal rates, it is essential to demonstrate that the effective dimension of $\H_m$ matches that of $\H_K$, i.e., $\tilde{\gamma} = \gamma$.
As discussed in the introduction,  this equivalence naturally holds for $\gamma = 1$. When $\gamma<1$, it is established by treating $K^m$ as a sum of i.i.d. random kernels with mean $K$.
However, this structure is not valid for deep architectures, as the gradient $\partial_{\bw^\ell_r} f_{\bW(0)}(\bx) $ is influenced not only by the weights $\bw^\ell_r(0)$ of the $\ell$-th layer but also by the  weights of all preceding layers. 
In contrast, we introduce $g_T\in\H_K$ as an intermediate term, and separately estimate $\|\bS_m g^m_T \!-\! \bS g_T\|_\rho^2$ and $\|\bS g_T \!- \!f_\rho\|_\rho^2$ in Propositions~\ref{pro:gm-gT} and \ref{pro:gT-frho}.
\end{remark}

Combining the above four propositions, we now present our main result on the excess population risk of GD for deep ReLU networks. The detailed proof is deferred to Section~\ref{proof:GD-results}. 
\begin{theorem}\label{thm:excess-relu}
Suppose Assumptions~\ref{ass:effec_dim} and \ref{ass:frho_smth} hold. 
 For any $\delta\in(0,1)$, assume $\eta \le 1/5$, $ \eta T \le n(36\log(8n/\delta))^{-1}$ and  \eqref{eq:m_condition} hold. 
Then, with probability at least $1-L\exp\big(\O(dL\log(m)) - \Omega(m^{\frac{1}{3}})\big) - \delta$ over initialization $(\ba,\bW(0))$ and sampling, there holds
\[ \L (f_{\bW(T)} ) - \L(f_\rho) \lesssim     \frac{L^\frac{14}{3}(\eta T)^4}{ m^{\frac{1}{3}}} + \Big(   \frac{\eta T}{n^2}  +  \frac{(\eta T)^\gamma + (\eta T)^{1 - 2\beta}}{ n} \Big) \log^4\big( {T}/{\delta}\big) + \frac{1}{(\eta T)^{2\beta}}. \]
\end{theorem}\vspace{-2mm}
We point out that our result does not need the widely adopted positivity assumption on the NTK Gram matrix $\bK^m$ to learn ReLU networks \cite{arora2019fine,du2018gradient}, i.e., the smallest eigenvalue of $\bK^m$ is strictly larger than $0$.  
Previous work~\cite{nitanda2021optimal, su2019learning} has shown that this assumption can be overly restrictive, as the smallest eigenvalue of $\bK^m$ tends to zero as the size of the training set increases.

The following corollary, derived from Theorem~\ref{thm:excess-relu}, shows that when the network width scales polynomially with the sample size $n$, dimension $d$, and the number of layers $L$, GD with a deep ReLU network can achieve the optimal excess risk rate $\O\big(n^{-\frac{2\beta}{2\beta + \gamma}}\big)$, with a gradient complexity of $\O\big(n^{1+\frac{1}{2\beta+\gamma}}\big)$. 

\begin{corollary}\label{cor:relu}
Suppose Assumptions~\ref{ass:effec_dim} and \ref{ass:frho_smth} hold and $2\beta +\gamma>1$.
For any $\delta\in(0,1)$, assume that $n \ge  \big(36(2\beta+\gamma)\beta^{-1} \big)^{\frac{2\beta+\gamma}{\beta}}\frac{16}{\delta}$ and $m \gtrsim L^{14}\max\big\{L^8d^3n^{\frac{7}{2\beta+\gamma}}\log^3(ndL/\delta), n^{\frac{6\beta+12}{2\beta+\gamma}}\big\}$.
  Choosing $T=\lceil n^{\frac{1}{2\beta + \gamma}} \rceil$ and $\eta\le1/5$ as a constant yields that, with a probability of at least $1 - \delta$ over initialization $(\ba,\bW(0))$ and sampling, there holds 
  \begin{center}
      $\L (f_{\bW(T)} ) - \L(f_\rho) \lesssim  n^{-\frac{2\beta}{2\beta+\gamma}}\log^4 ({n}/{\delta} ).$
  \end{center}
\end{corollary} 
Under Assumptions~\ref{ass:effec_dim} and \ref{ass:frho_smth}, the work \cite{lin2017optimal} proved that the optimal excess risk rate $\O\big(n^{-\frac{2\beta}{2\beta + \gamma}}\big)$ can be achieved by GD in the kernel setting, with a gradient complexity of $\O\big(n^{1+\frac{1}{2\beta+\gamma}}\big)$,    when $2\beta+\gamma>1$. 
Corollary~\ref{cor:relu} shows, provided that the network width scales polynomially with $n$, $d$, and $L$, GD with a deep ReLU network can replicate the classical results in the kernel setting.
It implies that the learning capability of GD with a deep ReLU network is competitive with that of the classical kernel regime.
Moreover, as $\beta$ and $\gamma$ increase, the required network width $m$ and the gradient complexity become less restrictive. In particular, in the capacity independent case, that is, $\gamma=1$, the optimal rate $\O(n^{-{2\beta\over 2\beta+1}})$ can be derived that matches the kernel setting \cite{ying2008online}.

\medskip

\noindent\textbf{Discussion with the existing work.} 
The study most relevant to our work on GD is \cite{nguyen2024many}, which provided the excess population risk bounds for two-layer neural networks with smooth activation. 
Specifically, \cite{nguyen2024many} established the optimal excess risk $\O\big(n^{-\frac{2\beta}{2\beta+\gamma}}\big)$ under Assumptions \ref{ass:effec_dim} and \ref{ass:frho_smth}, assuming the network width $m\gtrsim \text{Poly}(d, n)$. In their analysis, the smoothness of the activation function plays a crucial role especially for ensuring the boundedness of the second partial derivatives of $f_{\bW}$ at $\bW(0)$.  
In contrast, our work provides the first minimax-optimal excess risk rates for DNNs with non-smooth ReLU activation functions under the same assumptions, provided  $m\gtrsim \text{Poly}(L,d,n)$. 
Besides, \cite{kuzborskij2022learning} studied the generalization performance of GD for two-layer ReLU networks under the positive eigenvalue assumption of the NTK Gram matrix, focusing on learning target functions with additive noise that is uniformly bounded and Lipschitz.  
\cite{lai2023generalization} showed that gradient flow in two-layer ReLU networks can achieve a generalization bound of $\O(n^{-\frac{2}{3}})$ when $d=1$ and $\beta=1/2$. 
Table~\ref{table:summary} summarizes the comparison of our results with the related work.

\subsection{Optimal Rates for Stochastic Gradient Descent}
In this subsection, we present our main results for SGD. We begin by introducing the kernel SGD in the RKHS $\H_m$ based on the random feature approximation $K^m$
\begin{align}\label{eq:update-kernelSGD}
    f^m_{k+1} = f^m_k - \eta\big(f^m_k(\bx_{i_k}) - y_{i_k}\big)K^m_{\bx_{i_k}} \text{ with }f^m_0 = 0.
\end{align}
Let $\bW(T)$ and $f_T^m$ be produced by \eqref{eq:sgd-update} and \eqref{eq:update-kernelSGD} with $T$ iterations, respectively. We consider   
\begin{align}\label{eq:decom-sgd}
 \E_{ \A} \big[\varepsilon_{risk}\big( f_{\bW(T)} \big)\big] \lesssim  \E_{ \A} \Big[\big\|  f_{\bW(T)}- f^{\text{lin}}_{\bW(T)} \big\|_\rho^2\Big] +  \E_{ \A} \Big[\big\|f^{\text{lin}}_{\bW(T)} -  \bS_m f^m_T  \big\|_\rho^2\Big]  +   \E_{ \A} \Big[\big\|\bS_m f^m_T  -  f_{\rho} \big\|_\rho^2\Big], 
\end{align}
where $\bS_m:\H_m\xhookrightarrow{}\mathcal{L}^2_{\rho_\bx}$ is the inclusion mapping that maps $f^m_T\in \H_m$ to $\bS_m f^m_T\in \mathcal{L}^2_{\rho_\bx}$,  and $\E_{\A}[\cdot]$ denotes the expectation with respect to $\{i_k:k\in[T]\}$. In
the subsequent context, we will state the estimates for the three terms on the right-hand side of \eqref{eq:decom-sgd}.

First, we provide the estimate of the first term $\E_\A[ \|  f_{\bW(T)}  -  f^{\text{lin}}_{\bW(T)}  \|_\rho^2]$, whose proof can be found in Section~\ref{proof:SGD-f-flin_sgd}. 
Similar to GD, we use $\|\cdot\|_\infty$-norm to control the $\|\cdot\|_\rho$-norm of $f_{\bW(T)} - f^{\text{lin}}_{\bW(T)}$. 
\begin{proposition}\label{pro:sgd-f-flin}
Let $\delta\in(0,1)$.
Assume $\eta \le 1/5$, $\eta T \ge 1$ and
\begin{align}\label{eq:m_condition_sgd}
    m \gtrsim L^{26}d^3(\eta T)^7\log^3(m/\delta).
\end{align}
Then, with probability at least $1-L\exp\big(\O(dL\log(m)) - \Omega(m^{\frac{1}{3}})\big) - \delta$ over initialization $(\ba,\bW(0))$, there holds
  $$ \ebb_\A\Big[\big\|f_{\bW(T)} - f^{\text{lin}}_{\bW(T)} \big\|_\rho^2\Big] \le    \big\|f_{\bW(T)} - f^{\text{lin}}_{\bW(T)} \big\|_\infty^2 \lesssim \frac{L^\frac{14}{3}(\eta T)^{\frac{4}{3}}}{ m^{\frac{1}{3}}}.$$
\end{proposition}

The following proposition presents the estimation of the second term $\E_{ \A} [ \|f^{\text{lin}}_{\bW(T)} - \bS_m f^m_T  \|_\rho^2 ]$. 
Due to the randomness of SGD, the proof strategy of Proposition \ref{pro:flin-gm} cannot be directly extended to SGD.
Instead of estimating the $\|\cdot\|_\rho$-norm of the error term, we control the stronger $\|\cdot\|_\infty$-norm here.
The detailed proof is deferred to Section~\ref{proof:sgd-flin-fm}. 
\begin{proposition}\label{pro:sgd-flin-fm}
  Let $\delta\in(0,1)$.  
    Assume $\eta\le 1/5$, $\eta T \ge 1$ and condition \eqref{eq:m_condition_sgd} hold. 
    Then, with probability at least $1-L\exp\big(\O(dL\log(m)) - \Omega(m^{\frac{1}{3}})\big)-\delta$ over initialization $(\ba,\bW(0))$, there holds 
   $$ \ebb_\A\Big[\big\|f^{\text{lin}}_{\bW(T)} - \bS_mf^m_T \big\|_\rho^2\Big] \le  \big\|f^{\text{lin}}_{\bW(T)} - f^m_T \big\|_\infty^2 \lesssim \frac{L^\frac{20}{3}(\eta T)^{\frac{10}{3}}}{ m^{\frac{1}{3}}}.$$
\end{proposition}
Finally, we establish an upper bound for the last term on the right-hand side of \eqref{eq:decom-sgd} in the following proposition. To this end, we first estimate the distance between the SGD and GD iterates in the RKHS $\H_m$, i.e, $ f^m_T - g^m_T$.   This intermediate step, combined with the result of Proposition \ref{pro:gT-frho}, will allow us to complete the proof of the proposition, which is provided in Section~\ref{proof:sgd-fm-frho}. 
\begin{proposition}\label{pro:sgd-fm-frho}
    Suppose Assumptions \ref{ass:effec_dim} and \ref{ass:frho_smth} and \eqref{eq:m_condition_sgd} hold.
    Let $\delta \in (0,1)$ and $T\in\mathbb{N}$. Assume 
    $
        0 < \eta \le ( 32(\log (T) + 1))^{-1} \text{ and }   (\eta T)^{-1} \ge {n}^{-1}\log ( {6n}/{\delta} ).$
    Then, with probability at least $1-L\exp\big(\O(dL\log(m)) - \Omega(m^{\frac{1}{3}})\big)- \delta$ over initialization $(\ba, \bW(0))$ and sampling 
        \begin{align*}
          \E_\A  \Big[ \big\|\bS_m f^m_T -  f_\rho\big\|_\rho^2\Big]    \lesssim   \frac{L^\frac{20}{3} (\eta T)^4}{ m^{ \frac{1}{3}}}  +   \Big( \eta +  \frac{\eta T}{n^2}  +  \frac{(\eta T)^\gamma + (\eta T)^{1 - 2\beta}}{ n} \Big) \log^4\big( {T}/{\delta}\big) + \frac{1}{(\eta T)^{2\beta}}. 
        \end{align*}
\end{proposition}
Combining the above three propositions, we present our main result on the excess population risk of SGD with deep ReLU networks as follows.
The detailed proof is deferred to Section \ref{proof:SGD-results}.
\begin{theorem}\label{thm:excess-relu_sgd}
Suppose Assumptions~\ref{ass:effec_dim} and \ref{ass:frho_smth} and \eqref{eq:m_condition_sgd} hold. 
For any $\delta\in(0,1)$, assume  $   0 < \eta \le ( 32(\log (T) + 1))^{-1} $ and $1 \le \eta T \le n(36\log(12n/\delta))^{-1}$.
Then, with probability at least $1-L\exp\big(\O(dL\log(m)) - \Omega(m^{\frac{1}{3}})\big) - \delta$ over initialization $(\ba,\bW(0))$ and sampling, there holds
\begin{align*}
     \ebb_{ \A} \big[ \L (f_{\bW (T)} ) -  \L(f_{ \rho})\big]   \lesssim    \frac{L^\frac{20}{3} (\eta T)^4}{ m^{ \frac{1}{3}}}  +   \Big( \eta +  \frac{\eta T}{n^2}  +  \frac{(\eta T)^\gamma + (\eta T)^{1 - 2\beta}}{ n} \Big) \log^4\big( {T}/{\delta}\big) + \frac{1}{(\eta T)^{2\beta}}. 
\end{align*}
\end{theorem}

The following corollary, derived from Theorem~\ref{thm:excess-relu_sgd}, shows that when the network width scales polynomially with $n$, $d$ and $L$, SGD can achieve the optimal excess risk rate $\O\big(n^{-\frac{2\beta}{2\beta + \gamma}}\big)$ with lower computational cost (in terms of gradient complexity) than GD in Corollary~\ref{cor:relu}.  
\begin{corollary}\label{cor:relu_sgd}
    Suppose Assumptions~\ref{ass:effec_dim} and \ref{ass:frho_smth} hold and $2\beta +\gamma>1$.
For any $\delta\in(0,1)$, assume $n \ge (72(2\beta+\gamma))^{2(2\beta+\gamma)}(\frac{24}{\delta})$ and $m \gtrsim L^{20}\max\{L^6d^3n^{\frac{7}{2\beta+\gamma}}\log^3(ndL/\delta), n^{\frac{6\beta+12}{2\beta+\gamma}}\}$.
Choosing $T=\big\lceil n^{\frac{2\beta + 1}{2\beta + \gamma}} \big\rceil$ and $\eta=(72\log(24n/\delta))^{-1} n^{-\frac{2\beta}{2\beta+\gamma}} $ yields that,  with probability at least $1 - \delta$ over initialization $(\ba,\bW(0))$ and sampling, there holds 
  \begin{center}
      $\ebb_\A \big[\L (f_{\bW(T)} ) - \L(f_\rho) \big] \lesssim  n^{-\frac{2\beta}{2\beta+\gamma}} \log^2(n)\log^{2\beta}( {n}/{\delta}).$
  \end{center} 
\end{corollary}
Our results suggest that, provided a sufficiently wide network width, SGD with deep ReLU networks can recover the classical results of SGD \citep{dieuleveut2016nonparametric,lin2017optimal} in the kernel setting with the same gradient complexity under similar assumptions.

\medskip

\noindent\textbf{Discussion with the existing work.} Several works studying generalization performance of deep ReLU networks trained by SGD in the NTK regime \cite{cao2019generalization,chen2019much,zou2020gradient,xu2024overparametrized}. 
However, most of them focus on classification problems \cite{cao2020generalization,chen2019much,zou2020gradient}. For regression problems, \cite{xu2024overparametrized} studied one-pass SGD in the streaming data setting for deep ReLU networks and demonstrated that the average prediction error $\ebb_S[(\varepsilon_{risk} (f_{\bW(T)} ))^\frac{1}{2}]$ can converge to zero in expectation, provided that the width of the network scales exponentially with the number of layers $L$. The precise convergence rate was not specified in \cite{xu2024overparametrized}. 
Under Assumptions \ref{ass:effec_dim} and \ref{ass:frho_smth}, \cite{nitanda2021optimal} established minimax-optimal rates for one-pass SGD in two-layer neural networks with smooth activations, assuming the network width $m$ scales exponentially with $n$. 
We significantly extend their results, showing that SGD for DNNs can achieve the optimal rates under the relaxed condition $m\gtrsim \text{Poly}(L,n,d)$.

\section{Proofs of the Main Results}\label{proof}
In this section, we provide detailed proofs for our main results. 
Section~\ref{sec:proof-conce} introduces the uniform concentration of the NTK.  Section~\ref{sec:technicallemma} presents some necessary lemmas. Sections~\ref{sec:proof-GD} and \ref{sec:proof-SGD} give all proofs for both GD and SGD.

\subsection{Proof for Concentration of the NTK}\label{sec:proof-conce}
 In this subsection, we provide the uniform concentration of the NTK in our setting. 
Denote by $\mathbb{I}\{\cdot\}$ the indicator function (i.e., taking the value 1 if the argument holds true, and 0 otherwise).  
Given an input $\bx\in\X$, the $L$-layer ReLU network can be expressed as the following specific form
\begin{align}\label{eq:ReLU}
    f_\bW(\bx)=\ba^\top \sqrt{\frac{2}{m}} \bD^L(\bx) \bW^L \cdots \sqrt{\frac{2}{m}} \bD^1(\bx) \bW^1 \bx,  
\end{align}
where $\bD^\ell(\bx)$ with $\ell\in[L]$ is the diagonal sign matrix defined by
\begin{align}\label{eq:D-ell}
    \bD^\ell(\bx)=\text{diag}\big( \mathbb{I}\{ \langle \bw^\ell_r, o^{\ell-1}(\bx)\rangle_2 \ge 0  \} \big)_{r=1}^m \in \R^{m\times m}
\end{align}
with   $o^0(\bx)=\bx$ and  
\begin{align}\label{eq:o}
    o^{\ell-1}(\bx)= \sqrt{\frac{2}{m}} \bD^{\ell-1} (\bx) \bW^{\ell-1} \cdots \sqrt{\frac{2}{m}} \bD^1(\bx) \bW^1 \bx \text{ for } \ell=2,\ldots,L.  
\end{align} 
Here, $o^{\ell-1}(\bx)$ can be regarded as the output of the  $(\ell-1)$-th layer. 
By further defining $\big(V^L_L(\bx)\big)^\top = \sqrt{\frac{2}{m}}\bD^L(\bx)$  and
\begin{align}\label{eq:VL}
        \big(\bV^\ell_L(\bx)\big)^\top =  \sqrt{\frac{2}{m}} \bD^L(\bx) \bW^L \cdots \sqrt{\frac{2}{m}} \bD^\ell(\bx) \text{ for }\ell\in[L-1],
\end{align}
we can rewrite $f_{\bW}(\bx)$ as
\begin{align*}
    f_{\bW}(\bx)=\ba^\top  \big(\bV^\ell_L(\bx)\big)^\top  \bW^\ell o^{\ell-1}(\bx)=\big\langle  \bV^\ell_L(\bx) \ba \big(o^{\ell-1}(\bx)\big)^\top, \bW^\ell\big\rangle_2. 
\end{align*}
The above observation implies that
\begin{align*}
    \frac{\partial f_\bW(\bx)}{\partial \bW^\ell} = \bV^\ell_L(\bx) \ba \big(o^{\ell-1}(\bx)\big)^\top. 
\end{align*}

Denote $\|\cdot\|_{op}$ the operator norm of a matrix or an operator. 
For any $\bW,\widetilde{\bW}\in\W$, let $\|\bW-\widetilde{\bW}\|_{op,\infty}=\max_{\ell\in[L]} \|\bW^\ell-\widetilde{\bW}^\ell\|_{op}$, and, for any $R>0$, $   \mathcal{B}_R(\widetilde{\bW})=\big\{\bW\in \W: \|\bW-\widetilde{\bW}\|_{op,\infty}\le R\big\}. $

Let $\bD_0^\ell(\bx)$,  $o^\ell_0(\bx)$ and $\bV^\ell_{L,0}$ be defined as \eqref{eq:D-ell}, \eqref{eq:o} and \eqref{eq:VL}  with $\bW=\bW(0)$ for all $\ell\in[L]$. 
The following lemma shows that only the performance of the last layer plays a role in defining $K^m$ under the symmetric initialization. 
\begin{lemma}\label{prop:symm=0}
   For any $\bx\in\X$, there holds
   \[   \ba^\top\bD^L_0(\bx)\bW^L(0) = 0   \   
 \text{ and } \ 
      \frac{\partial f_{\bW(0)}(\bx)}{\partial \bW^\ell(0)} = 0\; \text{ for any } \ell\in[L-1].
 \]
    Further,  
\begin{align*} 
        K^m(\bx,\bx') = \Big\langle\frac{\partial f_{\bW(0)}(\bx)}{\partial \bW^L(0)},\frac{\partial f_{\bW(0)}(\bx')}{\partial \bW^L(0)}\Big\rangle_2 \; \text{ for all }\bx,\bx'\in\X.
    \end{align*}
\end{lemma}
\begin{proof}
    Note the $r$-th row of $\bD_0^L(\bx)\bW^L(0)$ is $\mbI\{\langle \bw^L_r(0), o^{L-1}_0(\bx) \rangle_2\ge 0\} (\bw^L_r(0))^\top$.
    Since $a_r = -a_{r+\frac{m}{2}}$ and $\bw^L_r(0) = \bw^L_{r+\frac{m}{2}}(0)$ for all $r\in[\frac{m}{2}]$, there holds
    \begin{align*}
         \ba^\top \bD^L_0(\bx)\bW^L(0) 
       &= \sum_{r=1}^m a_r\mbI\{\langle \bw^L_r(0), o^{L-1}_0(\bx)\rangle_2 \ge 0\} (\bw^L_r(0))^\top\\
        &= \sum_{r=1}^{\frac{m}{2}}a_r\mbI\{\langle \bw^L_r(0), o^{L-1}_0(\bx)\rangle_2 \ge 0\} (\bw^L_r(0))^\top + \sum_{r=1}^{\frac{m}{2}}a_{r+\frac{m}{2}}\mbI\{\langle \bw^L_{r+\frac{m}{2}}(0), o^{L-1}_0(\bx)\rangle_2 \ge 0\} (\bw^L_r(0))^\top\\
        &= \sum_{r=1}^{\frac{m}{2}}a_r\mbI\{\langle \bw^L_r(0), o^{L-1}_0(\bx)\rangle_2 \ge 0\} (\bw^L_r(0))^\top - \sum_{r=1}^{\frac{m}{2}}a_r\mbI\{\langle \bw^L_r(0), o^{L-1}_0(\bx)\rangle_2 \ge 0\} (\bw^L_r(0))^\top = 0.
    \end{align*}
   It further implies
    \begin{align*}     \big(\bV_{L,0}^\ell(\bx)\ba\big)^\top = \ba^\top\sqrt{\frac{2}{m}} \bD^L_0(\bx) \bW^L(0) \cdots \sqrt{\frac{2}{m}} \bD_0^\ell(\bx) = 0.
    \end{align*}    
   Combining this observation with $\frac{\partial f_{\bW(0)}(\bx)}{\partial \bW^\ell(0)} = \bV^\ell_{L,0}(\bx) \ba \big(o_0^{\ell-1}(\bx)\big)^\top$, we know $\frac{\partial f_{\bW(0)}(\bx)}{\partial \bW^\ell(0)} = 0$ for any $ \ell\in[L-1]$.  
   The first two results of the lemma are proved.

    Finally, 
    from \eqref{eq:Km_original} we get
     \begin{align}        K^m(\bx,\bx') = \sum_{\ell=1}^L \Big\langle\frac{\partial f_{\bW(0)}(\bx)}{\partial \bW^\ell(0) } , \frac{\partial  f_{\bW(0)}(\bx')}{\partial \bW^\ell(0) }  \Big\rangle_2=\Big\langle\frac{\partial f_{\bW(0)}(\bx)}{\partial \bW^L(0)},\frac{\partial f_{\bW(0)}(\bx')}{\partial \bW^L(0)}\Big\rangle_2,
     \end{align}
    which completes the proof.
\end{proof}

The following lemma shows that the initial weights $\bW^\ell(0)$ are bounded by $\O(\sqrt{m})$ with high probability.  
\begin{lemma}[Theorem 4.4.5 in \cite{vershynin2018high}]\label{lemma:oprt_norm}
    With probability at least $1-L\exp(-Cm)$ over the random choice of $\bW(0)$, there exists an absolute constant $c_0>1$ such that for any $\ell\in[L]$, there holds
\begin{align}\label{eq:optr_nrm_W}
    \|\bW^\ell(0)\|_{op} \le c_0\sqrt{m}. 
\end{align}
\end{lemma}
In the rest of the proofs, we will assume that the event $\{\|\bW^\ell(0)\|_{op} \le c_0\sqrt{m} \text{ for all } \ell \in [L]\}$ holds unless otherwise specified.

We require the following useful lemma which can be found in \cite{zou2018stochastic} (Corollary A.2, Lemmas A.8, B.1 and B.3 with $m_L = m/2, m_{L-1} = \cdots=m_1 = m$).
We note that in the following lemma,  Assumptions 3.4 and 3.5 in \cite{zou2018stochastic} are removed and the training dataset $S$ is replaced by a finite subset $\D$ of $\X$.
Denote $\|\cdot\|_0$ the $\ell^0$-norm which is the number of nonzero entries of a matrix or a vector. 
\begin{lemma}\label{lemma:initi_sample}
    Let $\D\subset\X$ be a finite subset of $\X$ with cardinality $|\D| = p$.
    For any $\delta\in(0,1)$, the following statements hold with probability at least $1-\delta$ over the random choice of $\bW(0)$ for all $\hat{\bx}\in\D$.
    \begin{enumerate}[label=(\alph*), leftmargin=*]
        \item Assume $m \ge C\log(pL/\delta).$ For all $\ell\in[L],$ there holds
        $$\big|\|o^\ell_0(\hat{\bx})\|_2 - 1\big| \le C\ell\sqrt{\frac{\log(pL/\delta)}{m}}.$$
        \item Assume $m \ge C\log(pL^2/\delta).$ For all $1\le \ell_1<\ell_2\le L,$ there holds
        $$\Big\|\sqrt{\frac{2}{m}}\bW^{\ell_2}(0)\prod_{h = \ell_1}^{\ell_2-1}\sqrt{\frac{2}{m}}\bD^h_0(\hat{\bx})\bW^h(0)\Big\|_{op} \le CL.$$
        \item Let $R_{op} \ge 1$ and $s\in\mathbb{N}$ with $s\le m$.
        Assume $m \ge CL^6\max\{s\log(m),R_{op}^2\}$ and $s \ge C\log(pL^2/\delta)$.
        Then, for any $\widehat{\bW}\in\W$ satisfying $\|\widehat{\bW}\|_{op,\infty} \le R_{op}$, and for all $\hat{\bx}\in\D,\ell\in[L]$ and any diagonal matrices $\widehat{\bD}^\ell\in\R^{m\times m}$ satisfying $\|\widehat{\bD}^\ell\|_0\le s$ and $\widehat{\bD}^\ell, \bD^\ell_0(\hat{\bx})+ \widehat{\bD}^\ell\in[-1,1]^{m\times m}$, there holds
        \begin{align}\label{eq:bound_Vhat_sample}
            \Big\|\prod_{h = \ell_1}^{\ell_2}\sqrt{\frac{2}{m}} \big(\bD^h_0(\hat{\bx}) + \widehat{\bD}^h\big)\big(\bW^h(0) + \widehat{\bW}^h\big)\Big\|_{op} \le CL \ \text{ for all $1 \le \ell_1 < \ell_2 \le L.$ }
        \end{align}
        \item Let $R_{op} \ge 1$.
        Assume $m \ge C\max\{L^{22}dR_{op}^2\log^3(m), L^3\log^3(pL/\delta)\}.$
        Then, for any $\bW \in \W$ satisfying $\|\bW - \bW(0)\|_{op,\infty} \le R_{op}$ and all $\hat{\bx}\in\D$, there holds
        \begin{align}\label{eq:o-o0_sample}
            \|o^\ell(\hat{\bx}) - o^\ell_0(\hat{\bx})\|_2 \le \frac{C\ell LR_{op}}{\sqrt{m}}.
        \end{align}
    \end{enumerate}
\end{lemma}
Lemma~\ref{lemma:initi_sample} applies only to the finite subset $\D$ of $\X$. In the following lemma, we extend their results to the entire space $\X$.  
\begin{lemma}\label{coro:bound-o}
    Let $\delta\in(0,1)$.
    The following statements hold with probability at least $1-\delta$ over initialization $\bW(0)$ for all $\ell\in[L].$
    \begin{enumerate}[label=(\alph*), leftmargin=*]
        \item Assume $m\gtrsim dL\log(\frac{m}{\delta}),$ there holds $\sup_{\bx\in\X}\big|\|o^\ell_0(\bx)\|_2 - 1\big| \le C\ell\sqrt{\frac{dL\log(m/\delta)}{m}}.$
        \item Assume $m\gtrsim dL\log(\frac{m}{\delta}),$ there holds $\sup_{\bx\in\X}\|\bV^\ell_{L,0}(\bx)\|_{op} \le \frac{CL}{\sqrt{m}}.$
        \item Assume $m\gtrsim dL^3\log(\frac{m}{\delta}),$ there holds $\sup_{\bx\in\X}\|\frac{\partial f_{\bW(0)}(\bx)}{\partial \bW^L(0)}\|_{2} \le 2.$
    \end{enumerate}
\end{lemma}
\begin{proof}
    We first prove part $(a)$.
    Let $\D$ be a $m^{-\frac{1}{2}}(\sqrt{2}c_0)^{-L}$-net of $\X$.
    We know for any $\bx\in\X$, there exists $\hat{\bx}\in\D$ such that $\|\bx-\hat{\bx}\|_2\le m^{-\frac{1}{2}}(\sqrt{2}c_0)^{-L}.$
    Then, for $\ell\in[L]$, there holds
    \begin{align*}
         \|o^\ell_0(\bx) - o^\ell_0(\hat{\bx})\|_2 
        &= \sqrt{\frac{2}{m}}\big\|\sigma\big(\bW^\ell(0) o^{\ell-1}_0(\bx)\big) - \sigma\big(\bW^\ell(0) o^{\ell-1}_0(\hat{\bx})\big)\big\|_2\le \sqrt{\frac{2}{m}}\big\|\bW^\ell(0)(o^{\ell-1}_0(\bx)-o^{\ell-1}_0(\hat{\bx}))\big\|_2\\
        &\le \sqrt{\frac{2}{m}}\big\|\bW^\ell(0)\big\|_{op} \|o^{\ell-1}_0(\bx)-o^{\ell-1}_0(\hat{\bx})\|_2 \le \sqrt{2}c_0\|o^{\ell-1}_0(\bx)-o^{\ell-1}_0(\hat{\bx})\|_2,
    \end{align*}
    where we have used $1$-Lipschitzness of the ReLU and $\|\bW^\ell(0)\|_{op}\le c_0\sqrt{m}$.
    
    Applying the above inequality recursively on $\ell$, we know $\|o^\ell_0(\bx) - o^\ell_0(\hat{\bx})\|_2 \le (\sqrt{2}c_0)^\ell\|\bx-\hat{\bx}\|_2 \le \frac{1}{\sqrt{m}}.$
    Note $\X = S^{d-1}$ is the unit sphere and $\D$ is a $m^{-\frac{1}{2}}(\sqrt{2}c_0)^{-L}$-net of $\X$.
    From Corollary 4.2.13 in \cite{vershynin2018high}, we know the covering number of $\X$ satisfy $|\D|\le (3\sqrt{m})^d(\sqrt{2}c_0)^{dL}.$
    Combining part $(a)$ of Lemma \ref{lemma:initi_sample} with $p=(3\sqrt{m})^d(\sqrt{2}c_0)^{dL}$ and the condition $m\gtrsim dL\log(\frac{m}{\delta})$, we know with probability at least $1-\delta$, there holds
    $$\big|\|o^\ell_0(\hat{\bx})\|_2 - 1\big| \le C\ell\sqrt{\frac{dL\log(m/\delta)}{m}} \text{ for all }\hat{\bx}\in\D.$$
    Combining this with the above inequality $\|o^\ell_0(\bx) - o^\ell_0(\hat{\bx})\|_2 \le \frac{1}{\sqrt{m}}$, there holds $|\|o^\ell_0(\bx)\|_2 - 1| \le \|o^\ell_0(\bx) - o^\ell_0(\hat{\bx})\|_2 + |\|o^\ell_0(\hat{\bx})\|_2 - 1| \le C\ell\sqrt{\frac{dL\log(m/\delta)}{m}}$ for all $\bx\in\X$.
    The first part is proved.

    Now, we turn to prove part $(b)$.
    From Lemma 32 in \cite{xu2024overparametrized} we know the cardinality of the set $\{(\bD_0^1(\bx),\ldots,\bD_0^L(\bx)) \in \R^{L\times m\times m}:\bx\in\X\}$ is at most $m^{dL}$.
    Therefore, there exists a subset $\D\subset\X$ with $|\D| \le m^{dL}$ such that
    $$\sup_{\bx\in\X}\|\bV^\ell_{L,0}(\bx)\|_{op} = \sup_{\bx\in\D}\|\bV^\ell_{L,0}(\bx)\|_{op} \text{ for all } \ell\in[L].$$
   Note part $(b)$ of Lemma \ref{lemma:initi_sample} with $p = m^{dL}, \ell_2 = L, \ell_1 = \ell + 1$ and the condition $m\gtrsim dL\log(\frac{m}{\delta})$ implies that with probability at least $1-\delta$,
    \begin{align*}
         \sup_{\bx\in\D}\|\bV^\ell_{L,0}(\bx)\|_{op}  &= \sup_{\bx\in\D}\Big\|\sqrt{\frac{2}{m}} \bD^L_0(\bx) \bW^L(0) \cdots \sqrt{\frac{2}{m}} \bD^{\ell+1}_0(\bx) \bW^{\ell+1}(0)\sqrt{\frac{2}{m}} \bD^\ell_0(\bx)\Big\|_{op}\\
        &\le \sup_{\bx\in\D}\big\|\bD^L_0(\bx)\big\|_{op}\Big\|\sqrt{\frac{2}{m}}\bW^{L}(0)\prod_{h = \ell + 1}^{L}\sqrt{\frac{2}{m}}\bD^h_0(\bx)\bW^h(0)\Big\|_{op}\Big\|\sqrt{\frac{2}{m}}\bD^\ell_0(\bx)\Big\|_{op} \le \frac{CL}{\sqrt{m}}.
    \end{align*}
    Hence, 
    \[  \sup_{\bx\in\X}\|\bV^\ell_{L,0}(\bx)\|_{op} \le \frac{CL}{\sqrt{m}}, \]
 which completes the proof of part $(b)$.

    It remains to prove the last part.
    Note
    \begin{align*}
      &\sup_{\bx\in\X}\Big\|\frac{\partial f_{\bW(0)}(\bx)}{\partial \bW^L(0)}\Big\|_{2}   = \sup_{\bx\in\X} \Big\|\bV^L_{L,0}(\bx) \ba \big(o^{L-1}_0(\bx)\big)^\top\Big\|_2 = \sup_{\bx\in\X} \Big\|\sqrt{\frac{2}{m}}\bD^L_{0}(\bx) \ba \big(o^{L-1}_0(\bx)\big)^\top\Big\|_2\\
     &\le\sqrt{2}\sup_{\bx\in\X}\|o^{L-1}_0(\bx)\|_{2}  \le \sqrt{2}\big(\sup_{\bx\in\X}\big|\|o^{L-1}_0(\bx)\|_{2} - 1\big| + 1\big) \le \sqrt{2}\Big(CL\sqrt{\frac{dL\log(m/\delta)}{m}} + 1\Big) \le 2,
    \end{align*}
    where the last second inequality follows from the first part of this lemma, and the last inequality used the condition $m \gtrsim dL^3\log(m/\delta)$.
    The proof of the lemma is completed.
\end{proof}

The following property shows that $K(\bx,\bx')$ is bounded for any $\bx,\bx'\in\X$. 
\begin{property}\label{prop:bound-K}
   For any $\bx,\bx'\in\X$, there holds $|K (\bx,\bx')|\le 1.$
\end{property}
\begin{proof}
    By the definition of $U^\ell(\bx)$, we know $\ebb[\sigma^2(U^{\ell}(\bx))] = \frac{1}{2}\ebb[(U^{\ell}(\bx))^2] = \ebb[\sigma^2(U^{\ell-1}(\bx))]$.
Recursively applying this equality, we have $\ebb[\sigma^2(U^{\ell}(\bx))] = \ebb[\sigma^2(U^{1}(\bx))] = \ebb_{w\sim\N(0,1)}[\sigma^2(w)] =  1/2$.
Then, according to Cauchy-Schwarz inequality, for all $\bx,\bx'\in\X$ and $\ell\in[L]$, there holds
\begin{align}\label{eq:bound_Uell}
    \big|2\ebb[\sigma(U^{\ell}(\bx))\sigma(U^{\ell}(\bx'))]\big| \le \sqrt{2\ebb[\sigma^2(U^{\ell}(\bx))]}\sqrt{2\ebb[\sigma^2(U^{\ell}(\bx'))]} = 1.
\end{align}
Further, according to the definition of $q^{\ell}(\bx,\bx')$, there holds $|q^{\ell}(\bx,\bx')| \le 1$.
Then, for $\bx,\bx'\in\X$, $K(\bx,\bx')$ can be uniformly bounded by
\begin{align*}
    |K (\bx,\bx')| = \big|2\ebb[\sigma(U^{L-1}(\bx))\sigma(U^{L-1}(\bx'))]\big
| |q^L(\bx,\bx')| \le \sqrt{2\ebb[\sigma^2(U^{L- 1}(\bx))]}\sqrt{2\ebb[\sigma^2(U^{L - 1}(\bx'))]} = 1.
\end{align*}
This completes the proof.
\end{proof}

The work \cite{du2019gradient} provided the concentration of the NTK for deep ReLU networks over the training data. 
  i.e., $\sup_{i,j\in[n]}|K^m(\bx_i,\bx_j) - K(\bx_i,\bx_j)|\rightarrow 0$ as $m\rightarrow \infty$. \cite{xu2024overparametrized} extended their result and showed the concentration uniformly over $\X$, i.e,  $\| K^m-K \|_\infty \lesssim C^Lm^{-\frac{1}{6}}\sqrt{d}$ assuming exponential scaling of $m$ with $L$. 
  In the following lemma, we improve their results with relaxed condition on $m$, which is pivotal for reducing the requirement on $m$ from $C^L \text{Poly}(n,d)$ to $\text{Poly}(n,d,L)$ in Corollaries~\ref{cor:relu} and \ref{cor:relu_sgd}. 
\begin{lemma}\label{lem:concentrationNTK}
    Let $\delta\in(0,1)$.
    Assume $m \gtrsim dL^3\log(\frac{m}{\delta})$.
    With probability at least $1-L\exp\big(\O(dL\log(m)) \!-\! \Omega(m^{\frac{1}{3}})\big) - \delta$ over the random choice of $(\ba,\bW(0))$, there holds
    \begin{align*}
        \|K^{m} - K \|_\infty \lesssim  \sqrt{L} m^{-\frac{1}{6}}    + \sqrt{ dL\log(m) m^{-1}}.
    \end{align*}
\end{lemma}
\begin{proof}
Note \eqref{eq:bound_Uell} and  $\big|(\bV^L_{L,0}(\bx)\ba)^\top\bV^L_{L,0}(\bx')\ba\big| = \frac{1}{m}|\ba^\top\bD^L_0(\bx)\bD^L_0(\bx')\ba| \le 1$. 
From the definitions of $K^m$ and $K$, there holds
\begin{align*}
     \|K^{m} - K\|_\infty 
    &  = \sup_{\bx,\bx'\in\X}\! \Big|\langle o^{L-1}_0(\bx), \! o^{L-1}_0(\bx')\rangle_2(\bV^L_{L,0}(\bx)\ba)^{\!\top}\!\bV^L_{L,0}(\bx')\ba \!-\!2 \ebb[\sigma(U^{L-1}(\bx))\sigma(U^{L-1}(\bx'))] q^L(\bx,\bx') \Big| \nonumber\\
    & \le \sup_{\bx,\bx'\in\X}\big|\langle o^{L-1}_0(\bx), o^{L-1}_0(\bx')\rangle_2 - 2\ebb[\sigma(U^{L-1}(\bx))\sigma(U^{L-1}(\bx'))]\big| \cdot \big|(\bV^L_{L,0}(\bx)\ba)^\top\bV^L_{L,0}(\bx')\ba\big|\nonumber\\
    &\quad + \sup_{\bx,\bx'\in\X}\big|2\ebb[\sigma(U^{L-1}(\bx))\sigma(U^{L-1}(\bx'))]\big| \cdot \Big|(\bV^L_{L,0}(\bx)\ba)^\top\bV^L_{L,0}(\bx')\ba - tr\big(\bV^L_{L,0}(\bx)^\top\bV^L_{L,0}(\bx')\big)\Big|\nonumber\\
    &\quad + \sup_{\bx,\bx'\in\X}\big|2\ebb[\sigma(U^{L-1}(\bx))\sigma(U^{L-1}(\bx'))]\big| \cdot \Big|tr\big(\bV^L_{L,0}(\bx)^\top\bV^L_{L,0}(\bx')\big) - q^L(\bx,\bx')\Big|\nonumber\\
    & \le   \sup_{\bx,\bx'\in\X}\big|\langle o^{L-1}_0(\bx), o^{L-1}_0(\bx')\rangle_2 - 2\ebb[\sigma(U^{L-1}(\bx))\sigma(U^{L-1}(\bx'))]\big| \nonumber\\
    &\quad + \sup_{\bx,\bx'\in\X} \Big|(\bV^L_{L,0}(\bx)\ba)^\top\bV^L_{L,0}(\bx')\ba - tr\big(\bV^L_{L,0}(\bx)^\top\bV^L_{L,0}(\bx')\big)\Big|  + \sup_{\bx,\bx'\in\X} \Big|tr\big(\bV^L_{L,0}(\bx)^\top\bV^L_{L,0}(\bx')\big) - q^L(\bx,\bx')\Big|\nonumber\\
    &\!=: \calE_1 + \calE_2 + \calE_3, 
\end{align*}
The estimates of the above three terms $\calE_1, \calE_2 , \calE_3$ are given as follows.

\noindent\textbf{Estimate of $\calE_1$}: The estimate of $\calE_1$ follows the same proof steps as in Lemma 6 in \cite{xu2024overparametrized}. 
According to Lemma 6 in \cite{xu2024overparametrized}, one can get that  $\calE_1\lesssim  {LC^L}{m^{-\frac{1}{3}}}$. 
We improve this estimate  from $ {LC^L}{m^{-\frac{1}{3}}}$ to $ {L}{m^{-\frac{1}{3}}}$ by using more finer estimates of initialization terms. 
Specifically, instead of using their estimate $\sup_{\bx}\|o^\ell_0(\bx)\|_2 \le c_0^\ell$ in Lemma 30, we apply the tight estimate  $\sup_{\bx}\|o^\ell_0(\bx)\|_2 \le \sup_{\bx}|\|o^\ell_0(\bx)\|_2 - 1| + 1 \le C$ according to part (a) of Lemma \ref{coro:bound-o} and the condition $m\gtrsim dL^3\log(m/\delta)$.
    In addition, we set $V_0$ to be a $c_0^{-L}m^{-2}$-net of the $S^{d-1}$ rather than a $m^{-2}$-net.
    Then, following the same steps of the proof of Lemma 6 in \cite{xu2024overparametrized}, with probability at least $1-L\exp(O(dL\log(m)) - \Omega(m^{\frac{1}{3}})))$ over initialization $\bW(0)$, there holds
\begin{align}
        \calE_1 \lesssim  {L}{m^{-\frac{1}{3}}}.\nonumber
    \end{align}

\noindent\textbf{Estimate of $\calE_2$}:
Similar to the proof of the estimate of $\calE_1$, by using more finer estimates $\sup_{\bx}\|o^\ell_0(\bx)\|_2 \le C,\sup_{\bx\in\X}\|\bV^L_{L,0}(\bx)\|_{op} \le m^{-\frac{1}{2}}$, following the same proof steps of Lemma 7 in \cite{xu2024overparametrized}, we can show that
\begin{align*}
    \calE_2 \lesssim {m^{-\frac{1}{3}}}.
\end{align*}

\noindent\textbf{Estimate of $\calE_3$}: 
Similar to the above arguments, we use the estimates $\sup_{\bx}\|o^\ell_0(\bx)\|_2 \le C$ and $\sup_{\bx\in\X}\|\bV^L_{L,0}(\bx)\|_{op} \le m^{-\frac{1}{2}}$ to improve the proof of Lemma 8 in \cite{xu2024overparametrized} and get
\begin{align*}
    \calE_3   \lesssim \frac{\sqrt{L}}{m^{\frac{1}{6}}} + \sqrt{\frac{dL\log(m)}{m}}.
\end{align*}
    
Combining the above estimates of $\calE_1,\calE_2,\calE_3$ completes the proof of this lemma. 
\end{proof}

\subsection{Useful Lemmas}\label{sec:technicallemma}
In this subsection, we present some useful lemmas for proving main results of both GD and SGD with deep ReLU networks.

Recall that  $\bD_0^\ell(\bx)$,  $o^\ell_0(\bx)$ and $\bV^\ell_{L,0}(\bx)$ are defined as \eqref{eq:D-ell}, \eqref{eq:o} and \eqref{eq:VL}  with $\bW=\bW(0)$ for all $\ell\in[L]$, and $   \mathcal{B}_R(\widetilde{\bW})=\big\{\bW\in \W: \|\bW-\widetilde{\bW}\|_{op,\infty} = \max_{\ell\in[L]} \|\bW^\ell-\widetilde{\bW}^\ell\|_{op} \le R\big\} $.  
\begin{lemma}\label{lemma:o-o_0}
    Let $\delta\in(0,1)$.
    Assume $m \gtrsim L^{22}d^3R_{op}^2\log^3(m/\delta)$ and $R_{op} \ge 1$.
    Then, with probability at least $1-\delta$ over initialization $\bW(0)$, the following statement holds for any $\bW \in \mathcal{B}_{R_{op}}\big(\bW(0)\big)$.
    \begin{align*}
        \sup_{\bx\in\X}\big\|o^\ell(\bx) - o^\ell_0(\bx)\big\|_2 \lesssim \frac{\ell LR_{op}}{\sqrt{m}}.
    \end{align*}
\end{lemma}
\begin{proof}
    Let $\D$ be a $\frac{1}{C^L\sqrt{m}}$-net of $\X$.
    We know for any $\bx\in\X$, there exists $\hat{\bx}\in\D$ such that $\|\bx-\hat{\bx}\|_2\le \frac{1}{C^L\sqrt{m}}.$
    Note the condition for $m$ implies $R_{op} \le \sqrt{m}$.
    Then, similar to the proof of part $(a)$ of Lemma \ref{coro:bound-o}, we know $\|o^\ell(\bx) - o^\ell(\hat{\bx})\|_2 \le (\sqrt{2}+\sqrt{2}c_0)^\ell \|\bx-\hat{\bx}\|_2 \le \frac{1}{\sqrt{m}}$ and $\|o^\ell_0(\bx) - o^\ell_0(\hat{\bx})\|_2 \le \frac{1}{\sqrt{m}}$.
    Note $\X = S^{d-1}$ is the unit sphere.
    From Corollary 4.2.13 in \cite{vershynin2018high}, it holds that $|\D|\le (3\sqrt{m})^dC^{dL}.$
    Then, applying part $(d)$ of Lemma \ref{lemma:initi_sample} with $p = (3\sqrt{m})^dC^{dL}$, there holds
    \begin{align*}
        \big\|o^\ell(\bx) - o^\ell_0(\bx)\big\|_2 &\le \big\|o^\ell(\bx) - o^\ell(\hat{\bx})\big\|_2 + \big\|o^\ell(\hat{\bx}) - o^\ell_0(\hat{\bx})\big\|_2 + \big\|o^\ell_0(\hat{\bx}) - o^\ell_0(\bx)\big\|_2\\
        &\lesssim  \frac{1}{\sqrt{m}} + \frac{\ell LR_{op}}{\sqrt{m}} + \frac{1}{\sqrt{m}} \lesssim \frac{\ell LR_{op}}{\sqrt{m}}.
    \end{align*}
    This completes the proof. 
\end{proof}
 
\begin{lemma}\label{lem:flip}
    Let $\delta\in(0,1)$.
    Assume $m \gtrsim L^{22}d^3R_{op}^2\log^3(m/\delta)$ and $ R_{op} \ge 1$.
    For any $\bW\in\mathcal{B}_{R_{op}}(\bW(0))$, with probability at least $1 - L\exp(\O(dL\log(m)) - \Omega(m^{\frac{1}{3}})) - \delta$ over initialization $(\ba,\bW(0))$,  for any $\ell\in[L]$, there holds
    \begin{align*}
        \sup_{\bx\in\X} \big\|\bD^\ell(\bx) - \bD^\ell_0(\bx)\big\|_0 \le  \big(LmR_{op}\big)^{\frac{2}{3}}.
    \end{align*}
\end{lemma}
\begin{proof}
    Let $R'>0$ which will be chosen later.
    For $\bx\in\X$ and $\ell\in[L]$, define the diagonal matrix
    \[ \bE^\ell(\bx)= diag\big\{\mbI\big\{|\langle \bw_r^\ell(0), o^{\ell-1}_0(\bx)  \rangle_2| \le R'\big\} \big\}_{r=1}^m \in \{0,1\}^{m\times m}.  \]
Note $\sup_\bx \|\bD^\ell(\bx)-\bD^\ell_0(\bx)\|_0 \le \sup_\bx \|(\bD^\ell(\bx)-\bD^\ell_0(\bx))(\bfI - \bE^\ell(\bx))\|_0 + \sup_\bx \|\bE^\ell(\bx)\|_0$. We will estimate $\sup_\bx \|(\bD^\ell(\bx)-\bD^\ell_0(\bx))(\bfI - \bE^\ell(\bx))\|_0$ and $\sup_\bx \|\bE^\ell(\bx)\|_0$ separately in the following proof.

\noindent \textbf{Estimate of $\sup_\bx \big\|(\bD^\ell(\bx)-\bD^\ell_0(\bx))(\bfI - \bE^\ell(\bx))\big\|_0$:}  
From the definition, if the absolute value of $(r,r)$-th entry of the diagonal matrix $(\bD^\ell(\bx)-\bD^\ell_0(\bx))(\bfI - \bE^\ell(\bx))$ is $1$, then $| \langle \bw_r^\ell(0), o_0^{\ell-1}(\bx) \rangle_2|> R'$ and $\mbI\{\langle\bw_r^\ell, o^{\ell-1}(\bx)\rangle_2 \ge 0\} \neq \mbI\{\langle \bw_r^\ell(0), o_0^{\ell-1}(\bx)\rangle_2 \ge 0 \}$.
Then, there holds
\[| \langle \bw_r^\ell, o^{\ell-1}(\bx) \rangle_2 -  \langle \bw_r^\ell(0), o_0^{\ell-1}(\bx) \rangle_2 |^2 \ge | \langle \bw_r^\ell(0), o_0^{\ell-1}(\bx) \rangle_2|^2 > (R')^2.\] Therefore, we have
\begin{align*}
     &\sup_\bx \big\|(\bD^\ell(\bx)-\bD^\ell_0(\bx))(\bfI - \bE^\ell(\bx))\big\|_0  \le \frac{1}{(R')^2} \sup_\bx \sum_{r=1}^m \big( \langle \bw_r^\ell, o^{\ell-1}(\bx) \rangle_2 -  \langle \bw_r^\ell(0), o_0^{\ell-1}(\bx) \rangle_2 \big)^2\nonumber\\
    &= \frac{1}{(R')^2}  \sup_\bx \big\| \bW^\ell o^{\ell-1}(\bx) -\bW^\ell(0)o^{\ell-1}_0(\bx) \big\|_2^2\nonumber\\
    &\le \frac{1}{(R')^2}  \sup_\bx \Big(\big\|  \bW^\ell- \bW^\ell(0)\big\|_{op} \big\|o^{\ell-1}(\bx)- o^{\ell-1}_0(\bx) +   o^{\ell-1}_0(\bx)\big\|_2  + \big\|\bW^\ell(0)\big\|_{op} \big\|o^{\ell-1}(\bx) - o^{\ell-1}_0(\bx) \big\|_2\Big)^2\nonumber\\
    &\le \frac{1}{(R')^2}  \sup_\bx \Big(R_{op} \big(\big\|o^{\ell-1}(\bx) - o^{\ell-1}_0(\bx)\big\|_2 + C \big) +  c_0 \sqrt{m}  \big\|o^{\ell-1}(\bx) - o^{\ell-1}_0(\bx) \big\|_2\Big)^2,
\end{align*}
where in the last inequality we have used $\sup_\bx \|o^{\ell-1}_0(\bx)\|_2\le C $ implied by part $(a)$ of Lemma \ref{coro:bound-o} and the condition for $m$, and $\|\bW^\ell(0)\|_{op}\le c_0\sqrt{m}$.

Combining the above inequality with Lemma \ref{lemma:o-o_0} and noting the condition $m \gtrsim L^{22}d^3R_{op}^2\log^3(m/\delta)$, we get
\begin{align*} 
    &\sup_\bx \big \|(\bD^\ell(\bx)-\bD^\ell_0(\bx))(\bfI - \bE^\ell(\bx))\big\|_0\lesssim \frac{1}{(R')^2}\Big(R_{op}\Big(\frac{L^2R_{op}}{\sqrt{m}} + C \Big) + L^2R_{op} \Big) \lesssim
    \frac{L^2 R_{op}^2   }{(R')^2}. 
\end{align*}

  \noindent \textbf{Estimates of $\sup_\bx \big\| \bE^\ell(\bx) \big\|_0$:}  
The proof is similar to that of Lemma 11 in \cite{xu2024overparametrized}, we give the proof here for the sake of completeness.

Denoting the function class $\F = \{\mbI\{|\langle \ \cdot \ , o^{\ell-1}_0(\bx)\rangle_2| \le R'\}: \bx\in \X \},$ there holds
$$\sup_\bx \frac{1}{m}\big\|\bE^\ell(\bx)\big\|_0 = \sup_\bx \frac{1}{m} \sum_{r=1}^m \mathbb{I}\Big\{ \big| \langle \bw_r^\ell(0),o_0^{\ell-1}(\bx) \rangle_2\big| \le R'\Big\} = \sup_{f\in\F}\frac{1}{m}\sum_{r=1}^mf(\bw^\ell_r(0)).$$
To control the right hand side of the above equality, we need to estimate the VC-dimension of $\F$.
We first fixed $(\bW^1(0),\ldots,\bW^{\ell-1}(0))$.
Denote $\mathcal{D}^{\ell-1} = \{(\bD^1_0(\bx), \ldots,\bD^{\ell-1}_0(\bx)): \bx\in\X\}\subset \R^{{(\ell-1)}\times m\times m}$.
From Lemma 32 in \cite{xu2024overparametrized} we know the cardinality of $\mathcal{D}^{\ell-1}$ is less than $m^{d(\ell-1)}$, i.e., $|\mathcal{D}^{\ell-1}| \le m^{d(\ell-1)}$.
Then,  there exists a disjoint partition of $\X$ such that $\X = \bigcup_{j\in[|\mathcal{D}^{\ell-1}|]} U_j$, where $U_i\cap U_j = \varnothing$ for $i\neq j$ and the tuple $(\bD^1_0(\bx), \ldots,\bD^{\ell-1}_0(\bx))\in\R^{(\ell-1)\times m \times m}$ is a fixed matrix sequence on each $U_j$.
Therefore, $o^{\ell-1}_0(\bx) = \sqrt{\frac{2}{m}} \bD^{{\ell-1} }_0 (\bx) \bW^{{\ell-1} }(0) \cdots \sqrt{\frac{2}{m}} \bD^1_0(\bx) \bW^1(0) \bx$ lies in a $d$-dimensional subspace of $\R^m$ on each $U_j$.
Let $V_j$ and $V$ be the VC-dimension of the classes $\F_j = \{\mbI\{|\langle \ \cdot\ , o^{\ell-1}(\bx)\rangle_2| \le R'\}: \bx\in U_j \}$ and $\F$, respectively.
By Theorem 9.5 in \cite{gyorfi2006distribution}, the VC-dimension of the class of indicators of half spaces in $\R^d$ is $d + 1$.
Further, note that $o^{\ell-1}_0(\bx)$ lies in a $d$-dimensional subspace of $\R^m$ on each $U_j$ and the indicator function $\mbI\{|\langle \bw^\ell, o^{\ell-1}(\bx)\rangle_2| \le R'\}$ can be written as the multiplication of two indicators of half space, i.e., $\mbI\{|\langle \bw^\ell, o^{\ell-1}(\bx)\rangle_2| \le R'\} = \mbI\{\langle \bw^\ell, o^{\ell-1}(\bx)\rangle_2 \le R'\}\mbI\{\langle \bw^\ell, o^{\ell-1}(\bx)\rangle_2 \ge -R'\}$.
Then, from Lemma 3.2.3 in \cite{blumer1989learnability} with $s=2$ we know $V_j \le 10(d + 1)$ for any $j$.
By further applying Lemma 23 in \cite{xu2024overparametrized} with $N = |\mathcal{D}^{\ell-1}| \le m^{d(\ell-1)}$, there holds $V \lesssim \max(d\log(d),\log(|\mathcal{D}^{\ell-1}|)) \lesssim d\ell\log(m)$.

Now, we turn to control $\sup_{f\in\F}\frac{1}{m}\sum_{r=1}^mf(\bw^\ell_r(0))$, which can be regarded as a function on $(\bw^\ell_1(0),\ldots,\bw^\ell_m(0))$.
One can check that the value of this function can change by at most $\frac{1}{m}$ under an arbitrary change of the $r$-th coordinate.
Then, by McDiarmid's inequality, we know with probability at least $1-\exp(-2m^{\frac{1}{3}})$ over $\bW^\ell(0)$, there holds 
\begin{align*}
    \sup_{f\in\F}\frac{1}{m}\sum_{r=1}^mf(\bw^\ell_r(0)) \le m^{-\frac{1}{3}} + \ebb_{\bW^\ell(0)}\Big[\sup_{f\in\F} \frac{1}{m} \sum_{r=1}^m f(\bw^\ell_r(0))\Big].
\end{align*}
    Now, we estimate the right-hand side of the above inequality.
    There holds
\begin{align*}
     \ebb_{\bW^\ell(0)}\Big[\sup_{f\in\F} \frac{1}{m} \sum_{r=1}^m f(\bw^\ell_r(0))\Big]  
    &\le  \ebb_{\bW^\ell(0)}\Big[\sup_{f\in\F} \Big|\frac{1}{m} \sum_{r=1}^m f(\bw^\ell_r(0)) - \ebb[f(\bw^\ell(0))]\Big|\Big] + \sup_{f\in\F}\ebb[f(\bw^\ell(0))]\\
    &\le  \sqrt{\frac{V}{m}} + \sup_{f\in\F}\ebb[f(\bw^\ell_r(0))] \le \sqrt{\frac{d\ell\log(m)}{m}} + \sup_{\bx}\ebb\big[\mathbb{I}\big\{ | \langle \bw^\ell(0),o_0^{\ell-1}(\bx)\rangle_2|\le R'\}\big]\\
    &\le \sqrt{\frac{d\ell\log(m)}{m}} + \sup_{\bx}\int_{-R'/\|o^{\ell-1}_0(\bx)\|_2}^{R'/\|o^{\ell-1}_0(\bx)\|_2}\frac{1}{\sqrt{2\pi}}e^{-\frac{t^2}{2}}dt \le   \sqrt{\frac{d\ell\log(m)}{m}} + \sup_{\bx}\frac{\sqrt{2}R'}{\sqrt{\pi}\|o^{\ell-1}_0(\bx)\|_2},
\end{align*}
where  the second inequality is according to Theorem 8.3.23 in \cite{vershynin2018high},  the third inequality follows from $V \le d\ell\log(m)$, in the last second inequality we have used $\langle \bw^\ell(0), o^{\ell-1}_0(\bx)/\|o^{\ell-1}_0(\bx)\|_2\rangle_2 \sim \N(0,1)$, and in the last inequality we have used $e^{-\frac{t^2}{2}} \le 1.$
It remains to estimate $\sup_{\bx}\frac{\sqrt{2}R'}{\sqrt{\pi}\|o^\ell_0(\bx)\|_2}.$
For the case $\ell=1$, there holds $\|o^0(\bx)\|_2 = \|\bx\|_2 = 1$.
For the case $\ell \ge 2$, for any $\bx\in\X$, from part $(a)$ of Lemma \ref{coro:bound-o} we have
\begin{align*}
    \big\|o^{\ell-1}_0(\bx)\big\|_2 &= 1 - \Big(1 - \big\|o^{\ell-1}_0(\bx)\big\|_2\Big) \ge 1 - \Big|\big\|o^{\ell-1}_0(\bx)\big\|_2 - 1\Big| \\
    &\ge 1 - CL\sqrt{\frac{dL\log(m/\delta)}{m}} \ge \frac{1}{2},
\end{align*}
where the last inequality follows from the condition $m \gtrsim L^{22}d^3R_{op}^2\log^3(m/\delta).$

Combining the above estimates we obtain
\begin{align*}
    \sup_\bx \frac{1}{m}\big\|\bE^\ell(\bx)\big\|_0 = \sup_\bx \frac{1}{m} \sum_{r=1}^m \mathbb{I}\Big\{ \big| \langle \bw_r^\ell(0),o_0^{\ell}(\bx) \rangle_2\big| \le R'\Big\} \lesssim \frac{1}{m^{\frac{1}{3}}} + \sqrt{\frac{d\ell\log(m)}{m}} + R'.
\end{align*}

Further, combining the estimates of $\sup_\bx \big\|(\bD^\ell(\bx)-\bD^\ell_0(\bx))(\bfI - \bE^\ell(\bx))\big\|_0$ and $\sup_\bx \|\bE^\ell(\bx)\|_0$, there holds
\begin{align*}
    \sup_\bx \big\| \bD^\ell(\bx) - \bD^\ell_0(\bx)\big\|_0 &\lesssim \frac{  L^2 R_{op}^2   }{(R')^2} + R' m + 2m^{\frac{2}{3}} + \sqrt{md\ell\log(m)}.
\end{align*}
Setting $R' \asymp (LR_{op})^{\frac{2}{3}}m^{-\frac{1}{3}}$. Noting that $m \gtrsim L^{22}d^3R_{op}^2\log^3(m/\delta)$ and $R_{op} \ge 1$, we have
\begin{align*}
    \sup_\bx \big\| \bD^\ell(\bx) - \bD^\ell_0(\bx)\big\|_0 &\lesssim  \big(LmR_{op}\big)^{\frac{2}{3}} + 2m^{\frac{2}{3}} + \sqrt{md\ell\log(m) } \lesssim \big(LmR_{op}\big)^{\frac{2}{3}}.
\end{align*} 
The proof of the lemma is completed.
\end{proof}

Recall that
\begin{align*}
    \bV^\ell_{L,0}(\bx) = \sqrt{\frac{2}{m}}\bD^L_0(\bx)\bW^L(0) \cdots\sqrt{\frac{2}{m}}\bD^{\ell+1}_0(\bx)\bW^{\ell+1}(0)\sqrt{\frac{2}{m}}\bD^{\ell}_0(\bx).
\end{align*}
For any $\ell \in[L]$, let $\widehat{\bW}^\ell$ and the diagonal matrix $\widehat{\bD}^\ell$ be the matrices with the same size of $\bW^\ell(0)$ and $\bD_0^\ell(\bx)$, respectively. Define,  for  $k\in[L-1]$ and $\ell<k$, 
\begin{align}\label{eq:hat_VL}
    \widehat{\bV}^\ell_k(\bx) =  \sqrt{\frac{2}{m}}&\big(\bD^k_0(\bx)+ \widehat{\bD}^k\big) \big(\bW^k(0) + \widehat{\bW}^k\big)\cdots \sqrt{\frac{2}{m}}\big(\bD^{\ell+1}_0(\bx)+ \widehat{\bD}^{\ell+1}\big)\big(\bW^{\ell+1}(0) + \widehat{\bW}^{\ell+1}\big) \sqrt{\frac{2}{m}}\big(\bD^{\ell}_0(\bx)+\widehat{\bD}^\ell\big) 
\end{align}
and $\widehat{\bV}^\ell_\ell(\bx) = \sqrt{\frac{2}{m}}\big(\bD^\ell_0(\bx)+ \widehat{\bD}^\ell\big)$ for all $\ell\in[L]$.

\begin{lemma}\label{lemma:bound_V}
    Let $\delta\in(0,1)$ and $\widehat{\bV}^\ell_k(\bx)$ with  $k\in[L]$ and $\ell<k$ be the matrix defined in \eqref{eq:hat_VL}. 
    Let $R_{op} \ge 1$ and $s\in[m]$.
    Assume $m \ge CL^6\max\{s\log(m),R_{op}^2\}$ and $s \ge CdL\log(m/\delta)$.
    Then, with probability at least $1-\delta$ over initialization $\bW(0)$, for any matrices satisfying $\|\widehat{\bW}\|_{op,\infty} \le R_{op}$ and diagonal matrices satisfying $\|\widehat{\bD}^\ell\|_0 \le s$ and $\widehat{\bD}^\ell, \bD^\ell_0(\bx) + \widehat{\bD}^\ell\in[-1,1]^{m\times m}$ for all $\ell\in[L]$ and $\bx\in\X$, there holds
    \begin{align*}      \sup_{\bx\in\X}\big\|\widehat{\bV}^\ell_k(\bx)\big\|_{op} \le \frac{CL}{\sqrt{m}}.
    \end{align*}
\end{lemma}
\begin{proof}
    Similar to the proof of part $(b)$ of Lemma \ref{coro:bound-o}, we know there exists a finite subset $\D\subset\X$ with $|\D| \le m^{dL}$ such that
    $$\sup_{\bx\in\X}\big\|\widehat{\bV}^\ell_{k}(\bx)\big\|_{op} = \sup_{\bx\in\D}\big\|\widehat{\bV}^\ell_{k}(\bx)\big\|_{op} \text{ for all } 1\le\ell<k\le L.$$
    Then, part $(c)$ of Lemma \ref{lemma:initi_sample} with $p = m^{dL}$ implies that
    \begin{align*}
        \sup_{\bx\in\D}\|\widehat{\bV}^\ell_k(\bx)\|_{op} &\le \Big\|\prod_{h = \ell+1}^{k}\sqrt{\frac{2}{m}} \big(\bD^h_0(\bx) + \widehat{\bD}^h\big)\big(\bW^h(0) + \widehat{\bW}^h\big)\Big\|_{op} \Big\|\sqrt{\frac{2}{m}}\big(\bD^\ell_0(\bx) + \widehat{\bD}^\ell\big)\Big\|_{op} \le \frac{CL}{\sqrt{m}}.
    \end{align*}
    This completes the proof.
\end{proof}

\begin{lemma}\label{lemma:bound_diff_V}
   Let $\delta\in(0,1)$ and $\widehat{\bV}^\ell_L(\bx)$ with  $\ell\in[L]$ be the matrix defined in \eqref{eq:hat_VL}.
    Let $R_{op} \ge 1$ and $s\in [m]$.
    Assume $\|\widehat{\bW}\|_{op,\infty} = \max_{\ell\in[L]}\|\widehat{\bW}^\ell\|_{op} \le {R_{op}}$ and $\sup_{\ell\in[L]}\|\widehat{\bD}^\ell\|_0 \le s$ and $\widehat{\bD}^\ell, \bD^\ell_0(\bx) + \widehat{\bD}^\ell\in[-1,1]^{m\times m}$ for all $\bx\in\X$.
    Suppose $m \ge CL^6\max\{s\log(m),R_{op}^2\}$ and $s \ge CdL\log(m/\delta)$.
    Then, with probability at least $1-\delta$ over the random choice of the initialization $\bW(0)$, there holds for all $\ell\in[L]$
    \begin{align*}
        \sup_{\bx\in\X}\big\|\ba^\top&\big(\widehat{\bV}^\ell_L(\bx) - \bV^\ell_{L,0}(\bx) \big)\big\|_2 \lesssim \frac{L\big(\sqrt{s} + R_{op}\big)}{\sqrt{m}}.     
    \end{align*}
\end{lemma}
\begin{proof}
For the case $\ell =L$, according to definitions of $\widehat{\bV}^\ell_L(\bx) $ and $\bV^\ell_{L,0}(\bx)$  we know     \begin{align}\label{eq:sparse-s}      \big\|\ba^\top(\widehat{\bV}^L_L(\bx) - \bV^L_{L,0}(\bx))\big\|_2 = \sqrt{\frac{2}{m}}\big\|\ba^\top\widehat{\bD}^L\big\|_2 \lesssim \frac{\sqrt{s}}{\sqrt{m}},
    \end{align}
    where the inequality is due to $a_r\in\{-1,1\}$ for $r\in[m]$ and $\|\widehat{\bD}^L\|_0 \le s$.
    This completes the proof of the case $\ell=L$.

 For the case $\ell \in [L-1]$, noting that $\ba^\top\sqrt{\frac{2}{m}}\bD^L_0(\bx)\bW^L(0) = 0$ (see Lemma~\ref{prop:symm=0}), we know
    \begin{align*}       &\ba^\top\big(\widehat{\bV}^\ell_L(\bx) - \bV^\ell_{L,0}(\bx)\big)\\
    &=  \ba^\top\Big(\sqrt{\frac{2}{m}}\big(\bD^L_0(\bx)\!+\! \widehat{\bD}^L(\bx)\big)\big(\bW^L(0) \!+\! \widehat{\bW}^L\big)\widehat{\bV}^\ell_{L-1}(\bx) \!-\! \sqrt{\frac{2}{m}}\bD^L_0(\bx)\bW^L(0)\bV^\ell_{L-1,0}(\bx)\Big)\\
        &=\ba^\top\sqrt{\frac{2}{m}}\bD^L_0(\bx)\bW^L(0)\big(\widehat{\bV}^\ell_{L-1}(\bx) \!-\! \bV^\ell_{L-1,0}(\bx)\big) + \ba^\top\sqrt{\frac{2}{m}}\widehat{\bD}^L(\bx)\big(\bW^L(0) \!+\! \widehat{\bW}^L\big)\widehat{\bV}^\ell_{L-1}(\bx) + \ba^\top\sqrt{\frac{2}{m}}\bD^L_0(\bx)\widehat{\bW}^L\widehat{\bV}^\ell_{L-1}(\bx)\\
        &= \ba^\top\sqrt{\frac{2}{m}}\widehat{\bD}^L(\bx)\big(\bW^L(0) + \widehat{\bW}^L\big)\widehat{\bV}^\ell_{L-1}(\bx) + \ba^\top\sqrt{\frac{2}{m}}\bD^L_0(\bx)\widehat{\bW}^L\widehat{\bV}^\ell_{L-1}(\bx).
    \end{align*} 
    According to Lemma \ref{lemma:bound_V}, we know $\sup_{\bx\in\X}\|\widehat{\bV}^\ell_{L}(\bx)\|_{op} \le \frac{CL}{\sqrt{m}}$.
    Then, there holds    \begin{align*}
      &  \big\| \ba^\top\big(\widehat{\bV}^\ell_L(\bx) - \bV^\ell_{L,0}(\bx) \big)\big\|_2\nonumber\\
        &\le \sqrt{\frac{2}{m}}\big\|\ba^{\!\top} \widehat{\bD}^L\!(\bx)\big\|_2 \!\big\|\bW^L\!(0)\!+\! \widehat{\bW}^L\big\|_{op}\big\|\widehat{\bV}^\ell_{\!L\!-\!1}(\bx)\big\|_{op} \!\!\!+\! \sqrt{\frac{2}{m}}\|\ba\|_2\big\| \bD^L_0(\bx)\big\|_{op}\big\|\widehat{\bW}^L\big\|_{op}\big\|\widehat{\bV}^\ell_{\!L\!-\!1}(\bx)\big\|_{op}\nonumber\\
        &\le  \big\|\ba^\top \widehat{\bD}^L(\bx)\big\|_2 \frac{\sqrt{2}(c_0\sqrt{m}+R_{op})}{\sqrt{m}}\frac{CL}{\sqrt{m}} + \sqrt{2}R_{op}\frac{CL}{\sqrt{m}} \lesssim \frac{L\big(\sqrt{s} + R_{op}\big)}{\sqrt{m}},
    \end{align*}
    where the second inequality used the assumption $\|\widehat{\bW}\|_{op,\infty}\le{R_{op}}$ and $\|\bW^L(0)\|_{op} \le c_0\sqrt{m}$, and the last inequality used \eqref{eq:sparse-s} and $R_{op} \le  \sqrt{m}$ by noting $m \ge CL^6R_{op}^2$.
    This completes the proof of the lemma.
\end{proof}

\begin{lemma}[Claim 11.2 and Proposition 11.3 in \cite{allen2019convergence}]\label{lem:o-o'}
   For any $\bW ,\widetilde{\bW}\in \mathcal{B}_{R_{op}}(\bW(0))$ and $\bx\in\X$,
   there exist a series of diagonal matrices $\{(\bD'')^\ell\in \R^{m\times m}\}_{\ell\in[L]}$ with entries in $[-1,1]$ such that for any $\ell\in[L]$, there holds
   \begin{enumerate}[label=(\alph*), leftmargin=*]
       \item $ o^\ell(\bx) - \tilde{o}^\ell(\bx) = \sum_{h=1}^\ell \Big[\prod_{j=h+1}^\ell \sqrt{\frac{2}{m}}\big( \widetilde{\bD}^j(\bx) + (\bD'')^j \big)\widetilde{\bW}^j \Big] \sqrt{\frac{2}{m}}\big( \widetilde{\bD}^h(\bx) + (\bD'')^h \big)\big( \bW^h - \widetilde{\bW}^h\big) o^{h-1}(\bx). $
       \item $\big\|(\bD'')^\ell\big\|_0\le \big\|\bD^\ell(\bx) - \widetilde{\bD}^\ell(\bx)\big\|_0$ and $\widetilde{\bD}^\ell(\bx) + (\bD'')^\ell$ has entries in $[0,1]$.
   \end{enumerate} 
\end{lemma}
 The following lemma shows that the neural network is almost linear in terms of its weights and the loss is locally almost smooth near the initialization. 
\begin{lemma}\label{lem:almost-convex}
  Assume $R_{op} \ge 1$ and $m \ge CL^{22}d^3R_{op}^2\log^3(m/\delta)$.
  For any $\bW , \widetilde{\bW}\in \mathcal{B}_{R_{op}}(\bW(0))$,  with probability at least $1-L\exp(\O(dL\log(m)) - \Omega(m^{\frac{1}{3}})) - \delta$ over initialization $(\ba,\bW(0))$, for any $z = (\bx,y)\in\Z$, there holds
\begin{align}\label{eq:linear_approx} 
    \Big|f_{\widetilde{\bW}}(\bx) - f_{\bW }(\bx) - \Big\langle \frac{\partial f_{\bW }(\bx)}{\partial \bW}, \widetilde{\bW}-\bW \Big\rangle_2\Big| \lesssim L^{\frac{7}{3}}\|\widetilde{\bW}-\bW\|_{op,\infty}R_{op}^{\frac{1}{3}}m^{-\frac{1}{6}},
    \end{align}
    \begin{align}\label{eq:semi-smth}
         l(\widetilde{\bW};z) -  l(\bW;z) \ge \Big\langle \frac{\partial l(\bW ;z)}{\partial \bW } , \widetilde{\bW}-\bW \Big\rangle_2 - |f_{\bW}(\bx) - y|\cdot \epsilon,
    \end{align}
    with $\epsilon \lesssim L^{\frac{7}{3}}\|\widetilde{\bW}-\bW\|_{op,\infty}R_{op}^{\frac{1}{3}}m^{-\frac{1}{6}}$, 
    and   \begin{align}\label{eq:diff_derivative}
        \Big\|\frac{\partial f_{\bW}(\bx)}{\partial \bW^\ell} - \frac{\partial f_{\bW(0)}(\bx)}{\partial \bW^\ell(0)}\Big\|_{2} \lesssim L^{\frac{4}{3}}R_{op}^{\frac{1}{3}}m^{-\frac{1}{6}}.
    \end{align}
\end{lemma}
\begin{proof}
We first prove that the neural network $f$ is almost linear in terms of its weights near the initialization. From the definition of $f$, we know
\begin{align*}
    &\Big|f_{\widetilde{\bW}}(\bx) - f_{\bW }(\bx) - \Big\langle \frac{\partial f_{\bW }(\bx)}{\partial \bW}, \widetilde{\bW}-\bW \Big\rangle_2\Big|\nonumber\\
    &=\Big| \ba^\top \tilde{o}^{L}(\bx) - \ba^\top o^L(\bx) - \sum_{\ell=1}^L \ba^\top \bigg[\prod_{h=\ell+1}^L \sqrt{\frac{2}{m}}\bD^h(\bx) \bW^h \bigg] \sqrt{\frac{2}{m}}\bD^\ell(\bx)\big( \widetilde{\bW}^\ell-\bW^\ell \big) o^{\ell-1}(\bx)\Big|,
\end{align*}
where we used the conventional notation $\prod_{L+1}^L = \bfI$.

Lemma~\ref{lem:o-o'} with $\ell=L$ implies there exist a series of diagonal matrices $\{(\bD'')^\ell\in \R^{m\times m}\}_{\ell\in[L]}$ with entries in $[-1,1]$ such that
\begin{align*}
    o^L(\bx) -  \tilde{o}^L(\bx)  =  \sum_{\ell=1}^L  \bigg[ \prod_{h=\ell+1}^L  \sqrt{\frac{2}{m}}\big( \widetilde{\bD}^h(\bx) +  (\bD'')^h \big)\widetilde{\bW}^h \bigg] \sqrt{\frac{2}{m}}\big( \widetilde{\bD}^\ell(\bx)  +  (\bD'')^\ell \big)\big( \bW^\ell- \widetilde{\bW}^\ell\big) o^{\ell-1}(\bx).
\end{align*}
Hence,
\begin{align}\label{eq:lemma:semismth}
    &\Big|f_{\widetilde{\bW}}(\bx) - f_{\bW }(\bx) - \Big\langle \frac{\partial f_{\bW }(\bx)}{\partial \bW}, \widetilde{\bW}-\bW \Big\rangle_2\Big|\nonumber\\
    &\le \sum_{\ell=1}^L\Big| \ba^\top  \bigg[\prod_{h=\ell+1}^L \sqrt{\frac{2}{m}}\big( \widetilde{\bD}^h(\bx) + (\bD'')^h \big)\widetilde{\bW}^h \bigg] \sqrt{\frac{2}{m}}\big( \widetilde{\bD}^\ell(\bx) + (\bD'')^\ell \big)\big(  \widetilde{\bW}^\ell - \bW^\ell\big) o^{\ell-1}(\bx) \nonumber\\
    &\qquad\quad  -  \ba^\top \bigg[\prod_{h=\ell+1}^L \sqrt{\frac{2}{m}}\bD^h(\bx) \bW^h \bigg] \sqrt{\frac{2}{m}}\bD^\ell(\bx)\big( \widetilde{\bW}^\ell-\bW^\ell \big) o^{\ell-1}(\bx)\Big|\nonumber\\
    &=: \sum_{\ell=1}^L\Big| U_\ell^L(\bx) \big(  \widetilde{\bW}^\ell - \bW^\ell\big) o^{\ell-1}(\bx)\Big|\le \sum_{\ell=1}^L\|U^L_\ell(\bx)\|_{2}\big\|\widetilde{\bW}^\ell - \bW^\ell\big\|_{op}\|o^{\ell-1}(\bx)\|_2,
\end{align}
where $U_\ell^L(\bx) = \ba^\top\big[\prod_{h=\ell+1}^L \sqrt{\frac{2}{m}}\big( \widetilde{\bD}^h(\bx) + (\bD'')^h \big)\widetilde{\bW}^h \big] \sqrt{\frac{2}{m}}\big( \widetilde{\bD}^\ell(\bx) + (\bD'')^\ell \big) - \ba^\top\big[\prod_{h=\ell+1}^L\sqrt{\frac{2}{m}}\bD^h(\bx)\bW^h\big]\sqrt{\frac{2}{m}}\bD^\ell(\bx)$.

We first consider estimating the term $\|U_\ell^L(\bx)\|_2$. 
We begin by showing that $\widetilde{\bD}^\ell(\bx) + (\bD'')^\ell,\widetilde{\bD}^\ell(\bx) + (\bD'')^\ell - \bD^\ell_0(\bx)\in[-1,1]^{m\times m}$ for all $\ell\in[L]$ and $\bx\in\X$.
Indeed, according to part $(b)$ of Lemma \ref{lem:o-o'}, we know $\widetilde{\bD}^\ell(\bx) + (\bD'')^\ell\in[0,1]^{m\times m}$.
Then, there holds $\widetilde{\bD}^\ell(\bx) + (\bD'')^\ell - \bD^\ell_0(\bx)\in[-1,1]^{m\times m}$ by noting $\bD^\ell_0(\bx)\in\{0,1\}^{m\times m}$.

Note $\bW,\widetilde{\bW}\in\mathcal{B}_{R_{op}}(\bW(0))$, then Lemma \ref{lem:flip} implies that $\|\bD^\ell(\bx) - \bD^\ell_0(\bx)\|_0, \|\widetilde{\bD}^\ell(\bx) - \bD^\ell_0(\bx)\|_0 \lesssim (LmR_{op})^{\frac{2}{3}}$ with probability at least $1 - L\exp(\O(dL\log(m)) - \Omega(m^{\frac{1}{3}})) - \delta$ over initialization $(\ba,\bW(0))$ for all $\ell\in[L]$.
Then, from part $(b)$ of Lemma \ref{lem:o-o'}, we know
\begin{align*}
     \big\|\widetilde{\bD}^\ell(\bx) + (\bD'')^h - \bD^\ell_0(\bx)\big\|_0  
    &\le \big\|\widetilde{\bD}^\ell(\bx) - \bD^\ell_0(\bx)\big\|_0 + \big\|\bD''\big\|_0 \le \big\|\widetilde{\bD}^\ell(\bx) - \bD^\ell_0(\bx)\big\|_0 + \big\|\widetilde{\bD}^\ell(\bx) - \bD^\ell(\bx)\big\|_0\\
    &\le 2\big\|\widetilde{\bD}^\ell(\bx) - \bD^\ell_0(\bx)\big\|_0 + \big\|\bD^\ell(\bx) - \bD^\ell_0(\bx)\big\|_0 \lesssim (LmR_{op})^{\frac{2}{3}} .
\end{align*}
Setting $s = (LmR_{op})^{\frac{2}{3}}$, the condition $m \gtrsim L^{22}d^3R_{op}^2\log^3(m/\delta)$ implies the conditions $m \gtrsim L^6\max\{s\log(m),R_{op}^2\}$ and $s \gtrsim dL\log(m/\delta)$ in Lemma \ref{lemma:bound_diff_V}.
Then, by further noting that $\bW , \widetilde{\bW} \in \calB_{R_{op}}(\bW(0))$, and $\widetilde{\bD}^\ell(\bx) + (\bD'')^\ell, \widetilde{\bD}^\ell(\bx) + (\bD'')^\ell - \bD^\ell_0(\bx)\in[-1,1]^{m\times m}$ and $\bD^\ell(\bx),\bD^\ell(\bx) - \bD^\ell_0(\bx)\in[-1,1]^{m\times m}$, we apply Lemma \ref{lemma:bound_diff_V} twice with $s = (LmR_{op})^{\frac{2}{3}}$, $\widehat{\bW}^\ell = \widetilde{\bW}^\ell - \bW^\ell(0), \widehat{\bD}^\ell(\bx) = \widetilde{\bD}^\ell(\bx) + (\bD'')^\ell - \bD^\ell_0(\bx)$ and $\widehat{\bW}^\ell= \bW^\ell - \bW^\ell(0), \widehat{\bD}^\ell(\bx)=\bD^\ell(\bx) - \bD^\ell_0(\bx)$, respectively, and there holds
\begin{align*}
    \big\|U^L_\ell(\bx)\big\|_{2} &\le \bigg\|\ba^\top\Big[\prod_{h=\ell+1}^L \sqrt{\frac{2}{m}}\big( \widetilde{\bD}^h(\bx) + (\bD'')^h \big)\widetilde{\bW}^h \Big] \sqrt{\frac{2}{m}}\big( \widetilde{\bD}^\ell(\bx) + (\bD'')^\ell \big) - \ba^\top V^\ell_{L,0}(\bx)\bigg\|_2\\
    &\qquad + \Big\|\ba^\top V^\ell_{L,0}(\bx) - \ba^\top\Big[\prod_{h=\ell+1}^L \sqrt{\frac{2}{m}}\bD^h(\bx) \bW^h \Big] \sqrt{\frac{2}{m}}\bD^\ell(\bx)\Big\|_2\\
    &\lesssim \frac{L\big((LmR_{op})^{\frac{1}{3}} + R_{op}\big)}{\sqrt{m}} \lesssim L^{\frac{4}{3}}R_{op}^{\frac{1}{3}}m^{-\frac{1}{6}},
\end{align*}
where the last inequality used $R_{op} \le (LmR_{op})^{\frac{1}{3}}$ by noting $m \gtrsim CL^{22}d^3R_{op}^2\log^3(m/\delta)$.

The term $\|o^\ell(\bx)\|_2$ can be controlled by using part $(a)$ of Lemma~\ref{coro:bound-o}, Lemma~\ref{lemma:o-o_0} and $m \gtrsim L^{22}d^3R_{op}^2\log^3(m/\delta)$ by
\begin{align*}
    \big\|o^\ell(\bx)\big\|_2 \le \big\|o^\ell(\bx)- o^\ell_0(\bx)\big\|_2 + \big\|o^\ell_0(\bx)\big\|_2 \lesssim \frac{\ell L R_{op}}{\sqrt{m}}  + C \lesssim C.
\end{align*}
Plugging the above two inequalities back into \eqref{eq:lemma:semismth}, we obtain
\begin{align}\label{eq:flinear}
    \Big|f_{\widetilde{\bW}}(\bx) - f_{\bW }(\bx) - \Big\langle \frac{\partial f_{\bW }(\bx)}{\partial \bW}, \widetilde{\bW}-\bW \Big\rangle_2\Big| &\lesssim L^{\frac{4}{3}}\sum_{\ell=1}^L\big\|\widetilde{\bW}^\ell-\bW^\ell\big\|_{op}R_{op}^{\frac{1}{3}}m^{-\frac{1}{6}}\lesssim L^\frac{7}{3}\big\|\widetilde{\bW}-\bW\big\|_{op,\infty}R_{op}^{\frac{1}{3}}m^{-\frac{1}{6}}.
\end{align}
The first part of the lemma is proved.

Now, we show the loss $ l$ is locally almost smooth near the initialization. 
    From the convexity of $ l(\bW;z)$ (with respect to $f_{\bW}$), we know
\begin{align*}
     l(\widetilde{\bW};z ) -  l(\bW ;z )&\ge \frac{\partial l(\bW ;z)}{\partial f_{\bW }} \big(f_{\widetilde{\bW}}(\bx) - f_{\bW }(\bx)\big) =  \big(f_{\bW }(\bx) - y \big) \big(f_{\widetilde{\bW}}(\bx)- f_{\bW }(\bx)\big). 
\end{align*}
Then, according to the chain rule, we get
\begin{align}\label{eq:weaklyconvex}
      l(\widetilde{\bW};z ) -  l(\bW;z ) 
    & \ge  \big(f_{\bW }(\bx)- y \big) \Big(f_{\widetilde{\bW}}(\bx) -f_{\bW }(\bx) - \Big\langle \frac{\partial f_{\bW }(\bx)}{\partial \bW}, \widetilde{\bW}-\bW  \Big\rangle_2 + \Big\langle \frac{\partial f_{\bW }(\bx)}{\partial \bW}, \widetilde{\bW}-\bW  \Big\rangle_2\Big)\nonumber\\
    &= \big(f_{\bW }(\bx)- y \big) \Big\langle \frac{\partial f_{\bW }(\bx)}{\partial \bW}, \widetilde{\bW}-\bW \Big\rangle_2 + \big(f_{\bW }(\bx)- y \big)  \Big(f_{\widetilde{\bW}}(\bx) - f_{\bW }(\bx)   - \Big\langle \frac{\partial f_{\bW }(\bx)}{\partial \bW}, \widetilde{\bW}-\bW  \Big\rangle_2\Big)\nonumber\\
    &\ge\Big\langle \frac{\partial  l(\bW ;z )}{\partial \bW}, \widetilde{\bW}-\bW  \Big\rangle_2 - \big|f_{\bW }(\bx)- y\big|\Big|f_{\widetilde{\bW}}(\bx) - f_{\bW }(\bx) - \Big\langle \frac{\partial f_{\bW }(\bx)}{\partial \bW}, \widetilde{\bW}-\bW \Big\rangle_2\Big|. 
\end{align}

Plugging \eqref{eq:flinear} back into \eqref{eq:weaklyconvex}, there holds
\begin{align*}      l(\widetilde{\bW};z ) -  l(\bW;z ) \ge\Big\langle \frac{\partial  l(\bW ;z )}{\partial \bW}, \widetilde{\bW}-\bW  \Big\rangle_2 - |f_{\bW}(\bx) - y|\cdot \epsilon
\end{align*}
with $\epsilon \lesssim L^{\frac{7}{3}}\|\widetilde{\bW}-\bW\|_{op,\infty}R_{op}^{\frac{1}{3}}m^{-\frac{1}{6}}.$
The second part of the lemma is proved.

Finally, we turn to prove the last part of the lemma. From the above estimates we already know $\|o^\ell(\bx) - o^\ell_0(\bx)\|_2 \lesssim \ell LR_{op} m^{-\frac{1}{2}}$, $\|o^\ell(\bx)\|_2 \lesssim C$ and $\big\|\ba^\top V^\ell_{L,0}(\bx) - \ba^\top[\prod_{h=\ell+1}^L \sqrt{\frac{2}{m}}\bD^h(\bx) \bW^h ] \sqrt{\frac{2}{m}}\bD^\ell(\bx)\big\|_2  \lesssim L^\frac{4}{3}R_{op}^{\frac{1}{3}}m^{-\frac{1}{6}}$ for all $\ell\in[L]$ and $\bx\in\X$.
Then, combining these estimates with Lemma \ref{lemma:bound_V}, there holds
\begin{align*}
        &\Big\|\frac{\partial f_{\bW}(\bx)}{\partial \bW^\ell} - \frac{\partial f_{\bW(0)}(\bx)}{\partial \bW^\ell(0)}\Big\|_{2}\\
        &= \Big\|o^{\ell-1}(\bx) \ba^\top\Big[\prod_{h=\ell+1}^L \sqrt{\frac{2}{m}}\bD^h(\bx) \bW^h \Big] \sqrt{\frac{2}{m}}\bD^\ell(\bx) -  o^{\ell-1}_0(\bx) \ba^\top V^\ell_{L,0}(\bx)\Big\|_{2}\\
        &\le \big\|o^{\ell-1}(\bx)\big\|_2\Big\|\ba^\top\Big[\prod_{h=\ell+1}^L \sqrt{\frac{2}{m}}\bD^h(\bx) \bW^h \Big] \sqrt{\frac{2}{m}}\bD^\ell(\bx) - \ba^\top V^\ell_{L,0}(\bx)\Big\|_2 + \big\|o^{\ell-1}(\bx) - o_0^{\ell-1}(\bx)\big\|_2   \big\|\ba^\top V^\ell_{L,0}(\bx)\big\|_2\\
        &\lesssim L^{\frac{4}{3}}R_{op}^{\frac{1}{3}}m^{-\frac{1}{6}} +  \ell L^2R_{op}m^{-\frac{1}{2}}\lesssim L^{\frac{4}{3}}R_{op}^{\frac{1}{3}}m^{-\frac{1}{6}},
    \end{align*}
    where the last inequality used $\ell L^2R_{op}m^{-\frac{1}{2}}\lesssim L^{\frac{4}{3}}R_{op}^{\frac{1}{3}}m^{-\frac{1}{6}}$ by noting $m \gtrsim L^{22}d^3R_{op}^2\log^3(m/\delta)$.
The proof is completed.
\end{proof}

\subsection{Proofs for Gradient Descent}\label{sec:proof-GD}
In this subsection, we give all proofs for GD. 
Sections~\ref{proof:f-flin} and \ref{proof:flin-gm} present the proofs of Propositions~\ref{pro:f-flin} and \ref{pro:flin-gm}, respectively. 
Sections~\ref{proof:gm-gT} and \ref{proof:gT-frho} present the proofs of Propositions~\ref{pro:gm-gT} and \ref{pro:gT-frho}, respectively. 
Section~\ref{proof:GD-results} provides detailed proofs for Theorem~\ref{thm:excess-relu} and Corollary~\ref{cor:relu}.

\subsubsection{Proof of Proposition~\ref{pro:f-flin}}\label{proof:f-flin}
For notational convenience,  define
 $\bfnc_{\bW(k)} = (f_{\bW(k)}(\bx_1),\ldots,f_{\bW(k)}(\bx_n))^\top\in\R^n$ and $\by=(y_1,\ldots,y_n)^\top\in\R^n.$ The following lemma shows that the trajectory of GD during the training process is always near the initialization. Note that we make no assumption on the data distribution and the NTK Gram matrix.  
\begin{lemma}\label{lemma:wk-w0}
Let $\delta\in(0,1)$ and $\{\bW(k)\}$ be produced by \eqref{eq:update-GD} with $\eta \le 1/5$.
Assume \eqref{eq:m_condition} holds. 
Then, with probability at least $1-L\exp(\O(dL\log(m)) - \Omega(m^{\frac{1}{3}})) - \delta$ over initialization $(\ba,\bW(0))$, for any $k\in[T]$, there holds 
\[ \big\| \bW(k) - \bW(0) \big\|_{op,\infty}^2 \le \big\|\bW(k) - \bW(0) \big\|_2^2 \le 4\eta k \]
and
\[\|\bfnc_{\bW(k)} - \by\|_2 \le 2\|\bfnc_{\bW(0)} - \by\|_2. \]
\end{lemma}
\begin{proof}
The lemma is proved by induction.  It's obvious that $  \| \bW(k) - \bW(0) \|^2_2 \le 0$ and $\|\bfnc_{\bW(k)} - \by\|_2 \le 2\|\bfnc_{\bW(0)} - \by\|_2$ hold with $k=0$. Assume, for all $t\in [k]$ with $k\le T-1$, $\| \bW(t) - \bW(0) \|^2_2 \le 4\eta t$ and $\|\bfnc_{\bW(t)} - \by\|_2 \le 2\|\bfnc_{\bW(0)} - \by\|_2$ hold.
We will show that  $\| \bW(k+1) - \bW(0) \|^2_2 \le 4\eta(k+1)$ and $\|\bfnc_{\bW(k+1)} - \by\|_2 \le 2\|\bfnc_{\bW(0)} - \by\|_2$.

    From the update rule \eqref{eq:update-GD}, we know
    \begin{align}\label{eq:wk-w0}
       &\big\|\bW(k+1) - \bW(0)  \big\|_2^2  = \Big\|\bW(k) - \bW(0) -  \frac{\eta}{n} \sum_{i=1}^n \frac{\partial  l(\bW(k);z_i)}{\partial \bW(k)} \Big\|_2^2\nonumber\\
       &\le \big\|\bW(k) - \bW(0)\big\|_2^2 +  \eta^2 \Big(\frac{1}{n}\sum_{i=1}^n\Big\|(f_{\bW(k)}(\bx_i) - y_i)\frac{\partial f_{\bW(k)}(\bx_i)}{\partial{\bW(k)}}\Big\|_2\Big)^2   +  \frac{2\eta}{n} \Big\langle \bW(0) -  \bW(k),  \sum_{i=1}^n \frac{\partial  l(\bW(k);z_i)}{\partial \bW(k)}    \Big\rangle_2\nonumber\\
       &\le \big\|\bW(k) - \bW(0)\big\|_2^2 + \frac{\eta^2\|\bfnc_{\bW(k)} - \by\|_2^2}{n^2}\sum_{i=1}^n\Big\|\frac{\partial f_{\bW(k)}(\bx_i)}{\partial{\bW(k)}}\Big\|_2^2 + \frac{2\eta}{n} \Big\langle \bW(0) -  \bW(k),  \sum_{i=1}^n \frac{\partial  l(\bW(k);z_i)}{\partial \bW(k)}   \Big\rangle_2 \nonumber\\
       &= \big\|\bW(k) - \bW(0)\big\|_2^2 + \frac{2\eta^2L_S(\bW(k))}{n}\sum_{i=1}^n\Big\|\frac{\partial f_{\bW(k)} (\bx_i)}{\partial{\bW(k)}}\Big\|_2^2 \nonumber\\
       & \le \frac{2\eta}{n} \Big\langle \bW(0) - \bW(k),  \sum_{i=1}^n \frac{\partial  l(\bW(k);z_i)}{\partial \bW(k)}   \Big\rangle_2,
    \end{align}
    where in the last inequality we have used the Cauchy-Schwarz inequality, and in the last equality we have used $\frac{\|\bfnc_{\bW(k)} - \by\|_2^2}{n}= 2L_S(\bW(k)).$

Now, we turn to control $\|\frac{\partial f_{\bW(k)} (\bx_i)}{\partial{\bW(k)}}\|_2$ and $\frac{2\eta}{n}\big\langle \bW(0)\!-\! \bW(k),  \sum_{i=1}^n \frac{\partial  l(\bW(k);z_i)}{\partial \bW(k)} \big\rangle_2$.
Setting $R_{op} = 2\sqrt{\eta T}$.
By the induction assumption, there holds $\bW(k) \in\calB_{R_{op}}(\bW(0))$.
Then from Lemma \ref{prop:symm=0} (if $\ell<L$) and part $(c)$ of Lemma \ref{coro:bound-o} (if $\ell=L$) and  \eqref{eq:diff_derivative} in Lemma \ref{lem:almost-convex} with $R_{op} = 2\sqrt{\eta T}$ and $\bW = \bW(k)$, there holds
\begin{align}\label{eq:bound_partial_wk}
    \Big\|\frac{\partial f_{\bW(k)} (\bx_i)}{\partial{\bW(k)}}\Big\|_2 &\le   \Big\|\frac{\partial f_{\bW(k)} (\bx_i)}{\partial{\bW(k)}}- \frac{\partial f_{\bW(0)} (\bx_i)}{\partial{\bW(0)}}\Big\|_2 +\Big\|\frac{\partial f_{\bW(0)} (\bx_i)}{\partial{\bW^L(0)}}\Big\|_2 \nonumber\\
    &\le \sqrt{L}\max_{\ell\in[L]}\Big\|\frac{\partial f_{\bW(k)} (\bx_i)}{\partial{\bW^\ell(k)}}- \frac{\partial f_{\bW(0)} (\bx_i)}{\partial{\bW^\ell(0)}}\Big\|_2 + 2   \le \epsilon_3 + 2
\end{align}
with $\epsilon_3 \lesssim L^\frac{7}{3}(\eta T)^{\frac{1}{6}}m^{-\frac{1}{6}}$. 

According to \eqref{eq:semi-smth} with $R_{op} = 2\sqrt{\eta T}$, $\bW = \bW(k)$, $\widetilde{\bW} = \bW(0)$ and $\|\bW(k)-\bW(0)\|_{op,\infty} \le 2\sqrt{\eta T}$, we know
\begin{align*}
    \frac{2\eta}{n} \Big\langle   \bW(0)- \bW(k),   \sum_{i=1}^n   \frac{\partial  l(\bW(k); z_i)}{\partial \bW(k)}   \Big\rangle_2  \le  2\eta\big(L_S(\bW(0)) - L_S(\bW(k))\big)  +  2\eta \epsilon_2 \sum_{i=1}^n \frac{|f_{\bW(k)}(\bx_i) -y_i|}{n},
\end{align*}
with $\epsilon_2 \lesssim L^{\frac{7}{3}}(\eta T)^{\frac{2}{3}}m^{-\frac{1}{6}}.$

Plugging the above two estimates back into \eqref{eq:wk-w0}, we get
\begin{align*}
     &\big\|\bW(k+ 1) - \bW(0)\big\|_2^2 \\
    &\le \big\|\bW(k) - \bW(0)\big\|_2^2 + 2\eta^2( \epsilon_3 + 2)^2 L_S(\bW(k))  + 2\eta\big(L_S(\bW(0))  - L_S(\bW(k))\big) + 2\eta \epsilon_2\sum_{i=1}^n\frac{1}{n}|f_{\bW(k)}(\bx_i) - y_i|\\
    &\le \big\|\bW(k) - \bW(0)\big\|_2^2 + 2\eta^2( \epsilon_3 + 2)^2 L_S(\bW(k))  + 2\eta\big(L_S(\bW(0))  - L_S(\bW(k))\big) + 2\eta\epsilon_2\frac{1}{\sqrt{n}}\|\bfnc_{\bW(k)} - \by\|_2\\
    &\le \big\|\bW(k) - \bW(0)\big\|_2^2 + 2\eta^2( \epsilon_3 + 2)^2 L_S(\bW(k))  + 2\eta\big(L_S(\bW(0))  - L_S(\bW(k))\big) + 4\eta\epsilon_2\\
    &\le \big\|\bW(k) - \bW(0)\big\|_2^2 + 2\eta L_S(\bW(k))\big( \eta(\epsilon_3 + 2)^2 - 1 \big) + \eta + 4\eta\epsilon_2\\
    &\le \big\|\bW(k) - \bW(0)\big\|_2^2 + 2\eta L_S(\bW(k))\big( 5\eta - 1 \big) + \eta + 2\eta \le 4\eta k + \eta + 2\eta < 4\eta (k+1),
\end{align*}
where in the second inequality we have used Cauchy-Schwarz inequality, and in the third inequality we have used the induction assumption $\|\bfnc_{\bW(k)} - \by\|_2 \le 2\|\bfnc_{\bW(0)} - \by\|_2 = 2\|\by\|_2 \le 2\sqrt{n}$ by noting $f_{\bW(0)} = 0$, and in the last third inequality we have used $L_S(\bW(0)) = \frac{1}{2n}\sum_{i=1}^ny_i^2 \le \frac{1}{2}$, and the last second inequality used $\epsilon_2 \le \frac{1}{2}$ and $\epsilon_3 \le \sqrt{5} - 2$ by condition \eqref{eq:m_condition}, and the last inequality follows from the induction assumption and $\eta \le \frac{1}{5}$ and $L_S(\bW(k)) \ge 0$.

Now, we turn to estimate $\|\bfnc_{\bW(k+1)} - \by\|_2$.
Let $\xi_i(k) = f_{\bW(k+1)}(\bx_i) - f_{\bW(k) }(\bx_i) - \big\langle \frac{\partial f_{\bW(k) }(\bx_i)}{\partial \bW(k)}, \bW(k+1)-\bW(k) \big\rangle_2$, there holds for all $i\in[n]$ that
\begin{align*}
     f_{\bW(k+1)}(\bx_i) - y_i  &= f_{\bW(k+1)}(\bx_i) - f_{\bW(k)}(\bx_i) + f_{\bW(k)}(\bx_i) - y_i= \Big\langle \frac{\partial f_{\bW(k) }(\bx_i)}{\partial \bW(k)}, \bW(k+1)-\bW(k) \Big\rangle_2 + \xi_i(k) + f_{\bW(k)}(\bx_i) - y_i\\
    &= -\frac{\eta}{n}\sum_{j=1}^n \Big\langle \frac{\partial f_{\bW(k) }(\bx_i)}{\partial \bW(k)}, \frac{\partial f_{\bW(k) }(\bx_j)}{\partial \bW(k)} \Big\rangle_2(f_{\bW(k)}(\bx_j) - y_j) + \xi_i(k) + f_{\bW(k)}(\bx_i) - y_i,
\end{align*}
where in the last equality we used the update rule \eqref{eq:update-GD}.
Define the matrix $\bH(k)\in\R^{n\times n}$ by $(\bH(k))_{i,j}:= \big\langle \frac{\partial f_{\bW(k) }(\bx_i)}{\partial \bW(k)}, \frac{\partial f_{\bW(k) }(\bx_j)}{\partial \bW(k)} \big\rangle_2$.
Denote $\bxi(k) = (\xi_1(k),\ldots,\xi_n(k))^\top\in\R^n$.
Then, the above observation implies 
\begin{align*}
    \bfnc_{\bW(k+1)} - \by &= \bfnc_{\bW(k)} - \by - \frac{\eta}{n}\bH(k)(\bfnc_{\bW(k)} - \by) + \bxi(k)\\
    &= \Big(\bfI - \frac{\eta}{n}\bH(k)\Big)(\bfnc_{\bW(k)} - \by) + \bxi(k).
\end{align*}
Applying the above equality recursively, we get
\begin{align*}
    \bfnc_{\bW(k+1)} - \by = \sum_{s=0}^k\prod_{u=s+1}^{k}\Big(\bfI - \frac{\eta}{n}\bH(u)\Big)\bxi(s) - \prod_{s=0}^k\Big(\bfI - \frac{\eta}{n}\bH(s)\Big)\by,
\end{align*}
where we used the conventional notation $\prod_k^{k-1} = \bfI.$ Then, there holds
\begin{align}\label{eq:fwk-y}
    \|\bfnc_{\bW(k+1)} - \by\|_2 \le  \sum_{s=0}^k\prod_{u=s+1}^{k}\Big\|\bfI - \frac{\eta}{n}\bH(u)\Big\|_{op}\|\bxi(s)\|_2 + \prod_{s=0}^k\Big\|\bfI - \frac{\eta}{n}\bH(s)\Big\|_{op}\|\by\|_2.
\end{align}

Now, we turn to estimate $\|\bfI - \frac{\eta}{n}\bH(s)\|_{op}$ and $\|\bxi(k)\|_2$.
According to \eqref{eq:bound_partial_wk} and \eqref{eq:m_condition}, there holds $\epsilon_3 + 2 \le \sqrt{5}$.
By further noting that $\eta \le \frac{1}{5}$, we know for all $s\in[k]$
\begin{align*}
\Big\|\frac{\eta}{n}\bH(s)\Big\|_{op}^2 &\le \Big\|\frac{\eta}{n}\bH(s)\Big\|_{2}^2 = \frac{\eta^2}{n^2}\sum_{i,j=1}^n\Big\langle \frac{\partial f_{\bW(k) }(\bx_i)}{\partial \bW(k)}, \frac{\partial f_{\bW(k) }(\bx_j)}{\partial \bW(k)} \Big\rangle_2^2 \\
&\le \frac{\eta^2}{n^2}\sum_{i,j=1}^n\Big\| \frac{\partial f_{\bW(k) }(\bx_i)}{\partial \bW(k)}\Big\|^2_2 \Big\|\frac{\partial f_{\bW(k) }(\bx_j)}{\partial \bW(k)} \Big\|_2^2 \le 25\eta^2 \le 1.
\end{align*}
Since $\frac{\eta}{n}\bH(s)$ is a PSD matrix whose operator norm is not larger than $1$, then $\|\bfI - \frac{\eta}{n}\bH(s)\|_{op} \le 1$.

Note we already showed that $\bW(s+1)\in\calB_{R_{op}}(\bW(0))$ with $R_{op} = 2\sqrt{\eta T}$ for all $s\le k$.
From \eqref{eq:linear_approx} in Lemma \ref{lem:almost-convex}, we get
\begin{align*}
    \|\bxi(s)\|_2 &= \Big(\sum_{i=1}^n \xi_i(s)^2\Big)^{\frac{1}{2}} \lesssim L^{\frac{7}{3}}\sqrt{n}\big\|\bW(s+1) - \bW(s)\big\|_{op,\infty}(\eta T)^{\frac{1}{6}}m^{-\frac{1}{6}}\\
    &= L^{\frac{7}{3}}\Big\|\frac{\eta}{\sqrt{n}}\sum_{i=1}^n\frac{\partial f_{\bW(s)}(\bx_i)}{\partial \bW(s)}(f_{\bW(s)}(\bx_i) - y_i)\Big\|_{op,\infty}(\eta T)^{\frac{1}{6}}m^{-\frac{1}{6}}\\
    &\lesssim L^{\frac{7}{3}}\Big[\sup_{\ell\in[L]}\frac{\eta}{\sqrt{n}}\sum_{i=1}^n\Big\|\frac{\partial f_{\bW(s)}(\bx_i)}{\partial \bW^\ell(s)}\Big\|_{op}|f_{\bW(s)}(\bx_i) - y_i|\Big](\eta T)^{\frac{1}{6}}m^{-\frac{1}{6}}\\
    &\lesssim L^{\frac{7}{3}}\eta\|\bfnc_{\bW(s)} - \by\|_2\sup_{\ell\in[L],i\in[n]}\Big\|\frac{\partial f_{\bW(s)}(\bx_i)}{\partial \bW^\ell(s)}\Big\|_{op}(\eta T)^{\frac{1}{6}}m^{-\frac{1}{6}}\\
    &\lesssim L^{\frac{7}{3}}\eta\|\bfnc_{\bW(s)} - \by\|_2(\eta T)^{\frac{1}{6}}m^{-\frac{1}{6}} \lesssim L^{\frac{7}{3}}\eta\|\bfnc_{\bW(0)} - \by\|_2(\eta T)^{\frac{1}{6}}m^{-\frac{1}{6}},
\end{align*}
where the last third inequality used Cauchy-Schwarz inequality, and in the last second inequality we have used \eqref{eq:bound_partial_wk} with $\epsilon_3+2\lesssim C$ by noting \eqref{eq:m_condition}, and in the last inequality we have used the induction assumption $\|\bfnc_{\bW(s)} - \by\|_2 \le 2\|\bfnc_{\bW(0)} - \by\|_2$.

Plugging the estimates $\|\bfI - \frac{\eta}{n}\bH(s)\|_{op} \le 1$ and the above inequality back into \eqref{eq:fwk-y}, and noting the condition \eqref{eq:m_condition} and $f_{\bW(0)} = 0$, there holds
\begin{align*}
    \|f_{\bW(k+1)} - \by\|_2 &\le CL^{\frac{7}{3}}\|\bfnc_{\bW(0)} - \by\|_2(\eta T)^{\frac{7}{6}}m^{-\frac{1}{6}} + \|\by\|_2 \\
    &= CL^{\frac{7}{3}}\|\bfnc_{\bW(0)} - \by\|_2(\eta T)^{\frac{7}{6}}m^{-\frac{1}{6}} + \|f_{\bW(0)}  - \by\|_2 \le 2\|\bfnc_{\bW(0)} - \by\|_2.
\end{align*}
This completes the proof of the lemma.
\end{proof}

Based on Lemma~\ref{lemma:wk-w0}, we present the proof of Proposition~\ref{pro:f-flin} as follows. 
\begin{proof}[Proof of Proposition~\ref{pro:f-flin}] 
Setting $R_{op} = 2\sqrt{\eta T}$.
From  Lemma \ref{lemma:wk-w0}, we know with probability at least $1-L\exp(\O(dL\log(m)) - \Omega(m^{\frac{1}{3}})) -\delta$ over initialization $(\ba,\bW(0))$ that  $\bW(k)\in\calB_{R_{op}}(\bW(0))$.
    Then, \eqref{eq:linear_approx} in Lemma \ref{lem:almost-convex} with $\widetilde{\bW} = \bW(k)$ and $\bW = \bW(0)$ implies
    \[\Big|f_{\bW(k)}(\bx) - f_{\bW(0)}(\bx) - \Big\langle \frac{\partial f_{\bW(0)}(\bx)}{\partial \bW(0)}, \bW(k)-\bW(0) \Big\rangle_2\Big| \lesssim L^\frac{7}{3}(\eta T)^{\frac{2}{3}}m^{-\frac{1}{6}}.\]
    Then, there holds
    \[ \big\|f_{\bW(T)} - f^{\text{lin}}_{\bW(T)}\big\|_\rho^2\le   \big\|f_{\bW(T)} - f^{\text{lin}}_{\bW(T)}\big\|_\infty^2 \lesssim \frac{L^\frac{14}{3}(\eta T)^{\frac{4}{3}} }{m^{\frac{1}{3}}} .\]
    Since we assume that we are under the event $\{\|\bW(0)\|_{op,\infty}\le c_0\sqrt{m}\}$, whose probability is at least $1 - L\exp(-Cm)$ according to Lemma \ref{lemma:oprt_norm}.
    By further noting that $1-L\exp(\O(dL\log(m)) - \Omega(m^{\frac{1}{3}})) -\delta \le 1 - L\exp(-Cm)$, the proof is completed.
\end{proof}

\subsubsection{Proof of Proposition~\ref{pro:flin-gm}}\label{proof:flin-gm}

Let $H$ be a separable Hilbert space.
For $f\in H$, we define the operator $f\otimes f:H\to H$ by $(f\otimes f) g = \langle f, g\rangle_H f$.
To estimate  $\big\|  f^{lin}_{\bW(T)}  - \mathbf{S}_m g^m_T \big\|_\rho^2$, we introduce the following useful lemma. 
\begin{lemma}[Lemma 3 in \cite{carratino2018learning}]\label{lemma:<2}
    Let $\lambda > 0$, $\Gamma \in \mathbb{N}$ and $\delta \in (0,1)$. Let $\zeta_1,\dots,\zeta_\Gamma$ be independent and identically distributed random vectors bounded by $\kappa > 0$. Let $Q_\Gamma = \frac{1}{\Gamma}\sum_{i=1}^\Gamma \zeta_i \otimes \zeta_i$ and $Q$ be the expectation of $Q_\Gamma$.
    Then, for any $\lambda \geq \frac{9\kappa^2}{\Gamma} \log \frac{\Gamma}{\delta}$, with probability at least $1-\delta$ over sampling, there holds
$$\| (Q + \lambda \bfI)^{1/2}(Q_\Gamma + \lambda \bfI)^{-1/2}\|_{op}^{2} = \|(Q_\Gamma + \lambda \bfI)^{-1/2} (Q + \lambda \bfI)^{1/2}\|_{op}^{2} \leq 2.$$
\end{lemma}

Now, we give the estimate of the second term $\big\|  f^{lin}_{\bW(T)}  - \mathbf{S}_m g^m_T \big\|_\rho^2$ as follows.  
\begin{proof}[Proof of Proposition~\ref{pro:flin-gm}]
According to Lemma \ref{lemma:oprt_norm}, we know $\|\bW(0)\|_{op,\infty} \le c_0\sqrt{m}$ holds with probability at least $1- L\exp(-Cm)$ over the random choice of $\bW(0)$.

    Let $F_k = f^{\text{lin}}_{\bW(k) } - g_k^m \in \H_m$ and $\epsilon^1_k = f_{\bW(k+1)}^{\text{lin}} - f_{\bW(k)}^{\text{lin}} + \frac{\eta }{n}  \sum_{i=1}^n (f_{\bW(k) }^{\text{lin}}(\bx_i) - y_i)K^m_{\bx_i} \in\H_m$.
Define the self-adjoint positive operator $\widehat{\bSigma}_m = \frac{1}{n}\sum_{i=1}^n K^m_{\bx_i}\otimes K^m_{\bx_i}:\H_m\to\H_m$. 
From the update rule of $ g_{k}^{m}$ \eqref{eq:update-kernelGD}, we know
\begin{align}\label{eq:recursive_FT}
    F_{k+1}      & = \Big(f_{\bW(k)}^{\text{lin}} \!-\! \frac{\eta }{n}  \sum_{i=1}^n (f_{\bW(k) }^{\text{lin}}(\bx_i) \!-\! y_i)K^m_{\bx_i}   + \epsilon^1_k\Big) - \Big(g^m_k \!-\! \frac{\eta }{n}  \sum_{i=1}^n (g^m_{ k }(\bx_i) \!-\! y_i)K^m_{\bx_i}\Big) \nonumber\\
        &= \big(f_{\bW(k)}^{\text{lin}} - g^m_k\big) - \frac{\eta}{n}\sum_{i=1}^n\big(f_{\bW(k)}^{\text{lin}}(\bx_i) - g^m_k(\bx_i)\big)K^m_{\bx_i} + \epsilon^1_k \nonumber\\
        &= F_k - \frac{\eta }{n}  \sum_{i=1}^n \langle F_k,K^m_{\bx_i}\rangle_{\H_m} K^m_{\bx_i} + \epsilon^1_k
        = \big(\mathbf{I} - \eta \widehat{\bSigma}_m\big)F_k  + \epsilon^1_k, 
\end{align}
where the last second equality follows from the fact $F_k = f^{\text{lin}}_{\bW(k) } - g_k^m \in \H_m$ and the reproducing kernel property that $f_{\bW(k)}^{\text{lin}}(\bx_i) - g^m_k(\bx_i) = \langle f_{\bW(k)}^{\text{lin}} - g^m_k,K^m_{\bx_i}\rangle_{\H_m} = \langle F_k,K^m_{\bx_i}\rangle_{\H_m}$.

Applying the above equality recursively, we get
\begin{equation}\label{eq:recursive-F}
    F_{k+1}= \sum_{s=0}^{k}\big(\mathbf{I} - \eta \widehat{\bSigma}_m\big)^s\epsilon^1_{k-s}.
\end{equation}

Define the mean of $\widehat{\bSigma}_m$ by $\bSigma_m := \ebb[\widehat{\bSigma}_m] = \int_\X K^m_\bx\otimes K^m_\bx d\rho_\bx(\bx):\H_m\to\H_m$. Mercer's Theorem \citep{steinwart2008support} implies $\|\bS_mf\|_\rho = \|\bSigma_m^{\frac{1}{2}}f\|_{\H_m}$ for any $f\in \H_m$.
From Lemmas \ref{prop:symm=0} and \ref{coro:bound-o} we know $\|K^m_{\bx_i}\|_{\H_m} = \sqrt{K^m(\bx_i,\bx_i)} \le \sqrt{\|K^m\|_\infty} \le \sup_{\bx\in\X}\|\frac{\partial f_{\bW(0)}(\bx)}{\partial \bW^L(0)}\|_2 \le 2$.
Therefore, Lemma \ref{lemma:<2}  with $\zeta_i = K^m_{\bx_i}$, $\Gamma = n$ and $\kappa = 2$ yields $\|(\bSigma_m+\lambda \mathbf{I})^{\frac{1}{2}}(\widehat{\bSigma}_m+\lambda \mathbf{I})^{-\frac{1}{2}}\|_{op}\le 2$ with probability at least $1-\delta/2$ over the sampling if $\lambda > \frac{36}{n}\log(\frac{2n}{\delta})$.
Then, according to \eqref{eq:recursive-F}, we get
\begin{align}\label{eq:bound_h_T}
    &\|\bS_mF_{k}\|_{\rho} = \big\|\bSigma_m^{\frac{1}{2}}F_k\big\|_{\H_m} \le \big\|(\bSigma_m+\lambda \mathbf{I})^{\frac{1}{2}}F_k\big\|_{\H_m}\nonumber\nonumber\\
    &\le \big\|(\bSigma_m+\lambda \mathbf{I})^{\frac{1}{2}}(\widehat{\bSigma}_m+\lambda \mathbf{I})^{-\frac{1}{2}}\big\|_{op}\big\|(\widehat{\bSigma}_m+\lambda \mathbf{I})^{\frac{1}{2}}F_k\big\|_{\H_m}
    \le 2\big\|\widehat{\bSigma}_m^{\frac{1}{2}}F_k\big\|_{\H_m} + 2\sqrt{\lambda}\big\|F_k\big\|_{\H_m}\nonumber\\
    &=2\eta^{-\frac{1}{2}}\bigg\|\sum_{s=0}^{k-1}\big(\eta \widehat{\bSigma}_m\big)^\frac{1}{2}\big(\mathbf{I} - \eta \widehat{\bSigma}_m\big)^s\epsilon^1_{k-s-1}\bigg\|_{\H_m} + 2\sqrt{\lambda}\bigg\|\sum_{s=0}^{k-1}\big(\mathbf{I} - \eta \widehat{\bSigma}_m\big)^s\epsilon^1_{k-s-1}\bigg\|_{\H_m}\nonumber\\
    &\le 2\eta^{-\frac{1}{2}}\sum_{s=0}^{k-1}\Big\|\big(\eta \widehat{\bSigma}_m\big)^\frac{1}{2}\big(\mathbf{I} - \eta \widehat{\bSigma}_m\big)^s\Big\|_{op}\big\|\epsilon^1_{k-s-1}\big\|_{\H_m}+2\sqrt{\lambda}\sum_{s=0}^{k-1}\Big\|\big(\mathbf{I} - \eta \widehat{\bSigma}_m\big)^s\Big\|_{op}\big\|\epsilon^1_{k-s-1}\big\|_{\H_m}.
\end{align}

For any $a\in[0,1)$ and any $s\in\mathbb{N}$, it can be easily computed that $\sup_{t\in[0,1]}t^a(1-t)^s \le \big(\frac{a}{a+s}\big)^a$.
Here, we take notation $0^0 = 1$.
From \eqref{eq:bound_Km} and $\eta \le 1/5$ we know $\eta\|\widehat{\bSigma}_m\|_{op} \le \frac{\eta}{n}\sum_{j=1}^n\|K^m_{\bx_i}\otimes K^m_{\bx_i}\|_{op} = \frac{\eta}{n}\sum_{j=1}^n\|K^m_{\bx_i}\|^2_{\H_m} \le \eta \|K^m\|_\infty \le \eta \sup_{\bx\in\X}\|\frac{\partial f_{\bW(0)}(\bx)}{\partial \bW^L(0)}\|_2^2 \le 1$. Then, there holds
\begin{align*}
    &\sum_{s=0}^{k-1}\big\|\big(\eta \widehat{\bSigma}_m\big)^a\big(\mathbf{I} - \eta \widehat{\bSigma}_m\big)^s\big\|_{op} \le \sum_{s=0}^{k-1}\sup_{t\in[0,1]}t^a(1-t)^s \le \sum_{s=0}^{k-1}\Big(\frac{a}{a+s}\Big)^a= 1 + a^a\sum_{s=1}^{k-1}\Big(\frac{1}{a+s}\Big)^a\\
    &\le1 +a^a\sum_{s=1}^{k-1}\int_{s-1}^{s}\Big(\frac{1}{a+x}\Big)^adx = 1 +a^a\int_0^{k-1}\Big(\frac{1}{a+x}\Big)^adx \\
    &\le 1 + \frac{a^a}{1-a}\big((k+a-1)^{1-a} - a^{1-a}\big) \le 1 + \frac{(k+a-1)^{1-a}}{1-a}.
\end{align*}
Combining \eqref{eq:bound_h_T} and the above inequality with $a=\frac{1}{2}$ and $a=0$, respectively,  we have
\begin{align}\label{eq:estimate_F_T}
    \big\|\bS_mF_k\big\|_\rho & \le \bigg(2\eta^{-\frac{1}{2}}\!\sum_{s=0}^{k-1}\Big\|\big(\eta \widehat{\bSigma}_m\big)^\frac{1}{2}\big(\mathbf{I} - \eta \widehat{\bSigma}_m\big)^s\Big\|_{op}\!\!+ 2\sqrt{\lambda}\sum_{s=0}^{k-1}\Big\|\big(\mathbf{I} - \eta \widehat{\bSigma}_m\big)^s\Big\|_{op}\bigg)\!\max_{s\in[k-1]}\|\epsilon^1_{s}\|_{\H_m}\nonumber\\
    &\le \Big(2\eta^{-\frac{1}{2}}\big(1 + 2\sqrt{k}\big) + 2\sqrt{\lambda} k \Big)\max_{s\in[k-1]}\|\epsilon^1_{s}\|_{\H_m}.
\end{align}

It remains to  estimate $\|\epsilon^1_k\|_{\H_m}$. From the definition of $f^{\text{lin}}_{\bW }$ and the update rule of GD \eqref{eq:update-GD}, there holds
    \begin{align*}
        &\epsilon_k^1(\bx) = f^{\text{lin}}_{\bW(k+1) }(\bx)  - f^{\text{lin}}_{\bW(k)} (\bx) + \frac{\eta}{n}\sum_{i=1}^n\big(f^{\text{lin}}_{\bW(k)} (\bx_i) - y_i\big)K^m_{\bx_i}(\bx)\nonumber\\
        &= \Big\langle \frac{\partial f_{\bW(0)}(\bx)}{\partial\bW(0)}, \bW(k+1)- \bW(k) \Big\rangle_2 + \frac{\eta}{n}\sum_{i=1}^n\big(f^{\text{lin}}_{\bW(k)} (\bx_i) - y_i\big)\Big\langle \frac{\partial f_{\bW(0)}(\bx_i)}{\partial\bW(0)}, \frac{\partial f_{\bW(0)}(\bx)}{\partial\bW(0)} \Big\rangle_2 \nonumber\\
        &= \Big\langle \frac{\partial f_{\bW(0)}(\bx)}{\partial\bW^L(0)}, \bW^L(k+1)- \bW^L(k) \Big\rangle_2 + \frac{\eta}{n}\sum_{i=1}^n\big(f^{\text{lin}}_{\bW(k)} (\bx_i) - y_i\big)\Big\langle \frac{\partial f_{\bW(0)}(\bx_i)}{\partial\bW^L(0)}, \frac{\partial f_{\bW(0)}(\bx)}{\partial\bW^L(0)} \Big\rangle_2 \nonumber\\
        & = \frac{\eta}{n}\sum_{i=1}^n\bigg[ - \big(f_{\bW(k)}(\bx_i) -  y_i\big)\Big\langle \frac{\partial f_{\bW(k)}(\bx_i)}{\partial\bW^L(k)}, \frac{\partial f_{\bW(0)}(\bx)}{\partial\bW^L(0)} \Big\rangle_2  + \big(f^{\text{lin}}_{\bW(k)}(\bx_i)  -  y_i\big) \Big\langle \frac{\partial f_{\bW(0)}(\bx_i)}{\partial\bW^L(0)}, \frac{\partial f_{\bW(0)}(\bx)}{\partial\bW^L(0)} \Big\rangle_2\bigg] \nonumber\\
        &= \bigg\langle \frac{\eta}{n}\sum_{i=1}^n  \Big[\big(y_i - f_{\bW(k)}(\bx_i)\big)\frac{\partial f_{\bW(k)}(\bx_i)}{\partial\bW^L(k)} + \big(f^{\text{lin}}_{\bW(k)}(\bx_i) - y_i\big)\frac{\partial f_{\bW(0)}(\bx_i)}{\partial\bW^L(0)}\Big] , \frac{\partial f_{\bW(0)}(\bx)}{\partial\bW^L(0)}\bigg\rangle_2 \nonumber\\
        & =: \Big\langle \Delta(k), \frac{\partial f_{\bW(0)}(\bx)}{\partial\bW^L(0)} \Big\rangle_2,
    \end{align*}
    where the second equality is due to $K^m_{\bx_i}(\bx) = K^m(\bx_i,\bx)$, the third equality used $\frac{\partial f_{\bW(0)}}{\partial\bW^\ell(0)} = 0$ for $\ell\in[L-1]$ according to Lemma \ref{prop:symm=0}, the fourth equality is according to the update rule \eqref{eq:update-GD}, and in the last equality we define
    \begin{align*}
        &\Delta(k) : = \frac{\eta}{n}\sum_{i=1}^n  \Big[\big(y_i - f_{\bW(k)}(\bx_i)\big)\frac{\partial f_{\bW(k)}(\bx_i)}{\partial\bW^L(k)} + \big(f^{\text{lin}}_{\bW(k)}(\bx_i) - y_i\big)\frac{\partial f_{\bW(0)}(\bx_i)}{\partial\bW^L(0)}\Big]\\
        &\!=\! \frac{\eta}{n}\!\sum_{i=1}^n\!  \Big[\big(y_i \!-\! f_{\bW(k)}(\bx_i)\big)\Big(\frac{\partial f_{\bW(k)}(\bx_i)}{\partial\bW^L(k)} \!-\! \frac{\partial f_{\bW(0)}(\bx_i)}{\partial\bW^L(0)}\Big)\!+\! \big(f^{\text{lin}}_{\bW(k)}(\bx_i) \!-\! f_{\bW(k)}(\bx_i)\big)\frac{\partial f_{\bW(0)}(\bx_i)}{\partial\bW^L(0)}\Big].
    \end{align*}
    
    Let $\boldsymbol{\Delta}(k) = (0,\ldots,0,\Delta(k))\in\W$, then we know $\epsilon^1_k(\bx) = \langle \boldsymbol{\Delta}(k),\Phi_m(\bx)\rangle_2$.
    Note that for any $f\in \H_m$,  $
    \|f\|^2_{\H_m} = \inf\big\{\sum_{\ell=1}^L\|\bW^\ell\|_2^2 : \bW \in \W \text{ with } f(\bx) = \big\langle \bW, \Phi_m(\bx)\big\rangle_2\big\}.$ We control $\|\epsilon^1_k\|^2_{\calH_m}$ as follows
    \begin{align}\label{eq:epsilon1}
         \|\epsilon^1_k\|_{\calH_m}& \le \|\Delta(k)\|_2 \le  \frac{\eta}{n} \sum_{i=1}^n   \Big[\big|y_i  -  f_{\bW(k)}(\bx_i)\big|\Big\|\frac{\partial f_{\bW (k)}(\bx_i)}{\partial\bW^L(k)} - \frac{\partial f_{\bW (0)}(\bx_i)}{\partial\bW^L(0)}\Big\|_2  +  \big|f^{\text{lin}}_{\bW (k)}(\bx_i) - f_{\bW (k)}(\bx_i)\big| \Big\|\frac{\partial f_{\bW (0)}(\bx_i)}{\partial\bW^L(0)}\Big\|_2\Big] \nonumber\\
        & \le  \eta\Big(  \frac{\|\bfnc_{\bW (k)} - \by\|_2}{\sqrt{n}}  \sup_{\bx\in\X} \Big\|\frac{\partial f_{\bW\!(k)}(\bx)}{\partial\bW^L(k)} - \frac{\partial f_{\bW (0)}(\bx)}{\partial\bW^L(0)}\Big\|_2  + \frac{1}{n} \sum_{i=1}^n \big|f^{\text{lin}}_{\bW (k)}(\bx_i) - f_{\bW (k)}(\bx_i)\big| \Big\|\frac{\partial f_{\bW (0)}(\bx_i)}{\partial\bW^L(0)}\Big\|_2\Big) \nonumber\\
        & \le  \eta \Big( \frac{\|\bfnc_{\bW(k)}-\by\|_2}{\sqrt{n}} \sup_{\bx\in\X}\Big\|\frac{\partial f_{\bW(k)}(\bx)}{\partial\bW^L(k)}  -  \frac{\partial f_{\bW(0)}(\bx)}{\partial\bW^L(0)}\Big\|_2 + \big\|f^{\text{lin}}_{\bW(k)}  -  f_{\bW(k)}\big\|_\infty\sup_{\bx\in\X}\Big\|\frac{\partial f_{\bW(0)}(\bx)}{\partial\bW^L(0)}\Big\|_2\Big),
    \end{align}
    where in the last second inequality we have used Cauchy-Schwarz inequality.

    From part $(c)$ of Lemma \ref{coro:bound-o}, we know $\sup_{\bx\in\X}\|\frac{\partial f_{\bW(0)}(\bx)}{\partial\bW^L(0)}\|_2 \le 2$.
    Note we choose $R_{op} = 2\sqrt{\eta T}$. From Lemma \ref{lemma:wk-w0} we know $\|\bfnc_{\bW(k)} - \by\|_2 \le 2\|\bfnc_{\bW(0)} - \by\|_2 = 2\|\by\|_2 \le 2\sqrt{n}$ and $\bW(k) \in \calB_{R_{op}}(\bW(0))$ for any $k\in[T]$.
    Combining this with Lemma \ref{lem:almost-convex}, there holds
    \begin{align*}
        \Big\|\frac{\partial f_{\bW(k)}(\bx)}{\partial\bW^L(k)} - \frac{\partial f_{\bW(0)}(\bx)}{\partial\bW^L(0)}\Big\|_2 \lesssim L^\frac{4}{3}\Big(\frac{\eta T}{m}\Big)^{\frac{1}{6}}.
    \end{align*}
    According to Proposition~\ref{pro:f-flin}, we know $\|f_{\bW(k)} - f_{\bW(k)}^{\text{lin}}\|_\infty \lesssim L^\frac{7}{3}(\eta T)^{\frac{2}{3}}m^{-\frac{1}{6}}$.
    Plugging the above observations back into \eqref{eq:epsilon1}, we know with probability at least $1-L\exp(\O(dL\log(m)) - \Omega(m^{\frac{1}{3}})) - \delta/2$ over initialization $(\ba,\bW(0))$, there holds
    \begin{align*}
     \|\epsilon^1_k\|_{\H_m}\lesssim \frac{L^\frac{7}{3}\eta^{\frac{5}{3}}T^{\frac{2}{3}}}{m^{\frac{1}{6}}} .
    \end{align*}
    
Putting the estimate of $\|\epsilon^1_s\|_{\H_m}$ back into \eqref{eq:estimate_F_T} and setting $\lambda = (\eta T)^{-1}$, with a little abuse of notation (we regard $f^{\text{lin}}_{\bW(k)}$ as a function in $\mathcal{L}^2_{\rho_\bx}$ in the following first term), the following inequality holds with probability at least $1-L\exp(\O(dL\log(m)) - \Omega(m^{\frac{1}{3}})) - \delta$ over the initialization $(\ba,\bW(0))$ and sampling
\begin{align*}
    \|f^{\text{lin}}_{\bW(k) } - \bS_mg_k^m\|_\rho&=\|\bS_mF_k\|_\rho \le 7\eta^{-\frac{1}{2}}\sqrt{T}\max_{s\in[T-1]}\|\epsilon^1_s\|_{\H_m} \lesssim \frac{L^\frac{7}{3}(\eta T)^{\frac{7}{6}}}{m^{\frac{1}{6}}}.
\end{align*}
Since we assume that we are under the event $\{\|\bW(0)\|_{op,\infty}\le c_0\sqrt{m}\}$, whose probability is at least $1 - L\exp(-Cm)$ according to Lemma \ref{lemma:oprt_norm}.
Squaring both sides of the above inequality and combining all the high probability events complete the proof of the proposition. 
\end{proof}

\subsubsection{Proof of Proposition~\ref{pro:gm-gT}}\label{proof:gm-gT}

For any $k\in\mathbb{N}$, we denote $\bG^m_k = (g^m_k(\bx_1),\ldots,g^m_k(\bx_n))^\top\in\R^n$, $\bG_k = (g_k(\bx_1),\ldots,g_k(\bx_n))^\top\in\R^n$ and $\by = (y_1,\ldots,y_n)^\top \in\R^n$.
Recall that $\bK = (K(\bx_i,\bx_j))_{i,j=1}^n$ and $\bK^m = (K^m(\bx_i,\bx_j))_{i,j=1}^n$ are the Gram matrices with kernels $K$ and $K^m$, respectively.
The following lemma shows that $\|\bG^m_k - \bG_k\|_2 \to 0$ as $m\to\infty$ for any $k\in[T]$. 
\begin{lemma}\label{prop:gmk-fk}
Let $\delta\in(0,1)$.
    Assume $m\gtrsim dL^3\log(m/\delta)$ and $\eta \le 1/4$.
    Then, with probability at least $1 - \delta$ over the initialization $(\ba,\bW(0))$, for any $k\in[T]$ there holds
    \begin{align*}
        \big\|\bG^m_{k} - \bG_{k}\big\|_2 \le \eta T\sqrt{n}\|K^m-K\|_\infty. 
    \end{align*}
\end{lemma}
\begin{proof}
    According to \eqref{eq:update-kernelGD} and \eqref{eq:kernel_GD_HK}, we know for any $k\in[T-1]$,
    \begin{align}\label{eq:update-Gmk-Fk}
        \bG^m_{k+1} = \bG^m_k - \frac{\eta}{n}\bK^m(\bG^m_k-\by) \text{ and } \bG_{k+1} = \bG_k - \frac{\eta}{n}\bK(\bG_k-\by).
    \end{align}
    Then, there holds
    \begin{align}
        \bG^m_{k+1} - \bG_{k+1} &= \bG^m_k - \bG_k - \frac{\eta}{n}\big(\bK^m(\bG^m_k-\by) - \bK(\bG_k - \by)\big)\nonumber\\
        &= \bG^m_k - \bG_k - \frac{\eta}{n}\big(\bK^m(\bG^m_k-\bG_k) - (\bK - \bK^m)(\bG_k - \by)\big)\nonumber\\
        &= \Big(\bfI - \frac{\eta}{n}\bK^m\Big)(\bG^m_k - \bG_k) + \frac{\eta}{n}(\bK - \bK^m)(\bG_k - \by)\nonumber.
    \end{align}
    Applying the above equality recursively, we have
    \begin{align}\label{eq:Gmk-Fk}
        \big\|\bG^m_{k+1} - \bG_{k+1}\big\|_2 &= \Big\|\frac{\eta}{n}\sum_{s=0}^k\Big(\bfI - \frac{\eta}{n}\bK^m\Big)^s(\bK-\bK^m)(\bG_{k-s} - \by)\Big\|_2\nonumber\\
        &\le \frac{\eta}{n}\sum_{s=0}^k\Big\|\bfI - \frac{\eta}{n}\bK^m\Big\|_{op}^s\|\bK-\bK^m\|_{op}\|\bG_{k-s} - \by\|_2.
    \end{align}

    From Lemma \ref{prop:symm=0} we know that
    \begin{align}\label{eq:bound_Km}
        \|K^m\|_\infty = \sup_{\bx,\bx'\in\X}\Big|\Big\langle\frac{\partial f_{\bW(0)}(\bx)}{\partial \bW^L(0)},\frac{\partial f_{\bW(0)}(\bx')}{\partial \bW^L(0)}\Big\rangle_2\Big| \le \sup_{\bx\in\X}\Big\|\frac{\partial f_{\bW(0)}(\bx)}{\partial \bW^L(0)}\Big\|_2^2 \le 4.
    \end{align}
    where the last inequality used Lemma \ref{coro:bound-o} and condition $m \gtrsim dL^3\log(m/\delta)$.
    Then, for any $\boldsymbol{\alpha} = (\alpha_1,\ldots,\alpha_n)^\top\in\R^n$ with $\|\boldsymbol{\alpha}\|_2 = 1$, there holds $\boldsymbol{\alpha}^\top\bK^m\boldsymbol{\alpha} = \|\sum_{i=1}^n\alpha_iK^m_{\bx_i}\|^2_{\H_m}\le (\sum_{i=1}^n|\alpha_i|\|K^m_{\bx_i}\|_{\H_m})^2 \le (\sum_{i=1}^n|\alpha_i|\|K^m\|_{\infty}^\frac{1}{2})^2\le 4(\sum_{i=1}^n|\alpha_i|)^2 \le 4n$.
    This implies that $\|\bK^m\|_{op} \le 4n$.
    Since $\eta \le 1/4$ and $\bK^m$ is PSD, we know $\|\bfI - \frac{\eta}{n}\bK^m\|_{op}\le1$.

    Then, there holds
    \begin{align*}
        \|\bK^m-\bK\|_{op} &= \sup_{\|\boldsymbol{\alpha}\|_2 = 1} |\boldsymbol{\alpha}^\top (\bK^m - \bK)\boldsymbol{\alpha}| = \sup_{\|\boldsymbol{\alpha}\|_2 = 1} \Big|\sum_{i,j=1}^n \alpha_i\alpha_j\big(K^m(\bx_i,\bx_j) - K(\bx_i,\bx_j)\big)\Big|\\
        &\le  \|K^m-K\|_\infty\sup_{\|\boldsymbol{\alpha}\|_2 = 1}\sum_{i,j=1}^n |\alpha_i\alpha_j| = \|K^m-K\|_\infty\sup_{\|\boldsymbol{\alpha}\|_2 = 1}\Big(\sum_{i=1}^n |\alpha_i|\Big)\Big(\sum_{j=1}^n|\alpha_j|\Big)\\
        &\le n\|K^m-K\|_\infty.
    \end{align*}    
    Further, from \eqref{eq:update-Gmk-Fk}, we know $\bG_{k} = (\bfI - \frac{\eta}{n}\bK)\bG_{k-1} + \frac{\eta}{n}\bK\by$.
    Recursively applying this equation, we get $\bG_{k} = \frac{\eta}{n}\sum_{s=0}^{k-1}(\bfI - \frac{\eta}{n}\bK)^s\bK\by$.
    Analogous to the estimate of $\|\bfI - \frac{\eta}{n}\bK^m\|_{op}$, we can show that $\|\bK\|_{op} \le n$ and $\|\bfI - \frac{\eta}{n}\bK^m\|_{op}\le 1$ by noting $\|K\|_\infty\le1$ (see Property \ref{prop:bound-K}).
    Then, there holds
    \begin{align}\label{eq:Fk-y}
        \|\bG_k\|_2 &\le \Big\|\sum_{s=0}^{k-1}\Big(\bfI - \frac{\eta}{n}\bK\Big)^s\frac{\eta}{n}\bK\Big\|_{op}\|\by\|_2 \le  \sqrt{n}\sup_{t\in[0,1]}\Big|\sum_{s=0}^{k-1}(1-t)^st\Big|\nonumber\\
        &=\sqrt{n}\sup_{t\in[0,1]}\big(1 - (1-t)^k\big) \le \sqrt{n}.
    \end{align}
    Plugging the above estimates back into \eqref{eq:Gmk-Fk}, we have
    \begin{align*}
        \big\|\bG^m_{k+1} - \bG_{k+1}\big\|_2 &\le \frac{\eta}{n}\sum_{s=0}^k\|\bK^m-\bK\|_{op}\|\big(\|\bG_{k-s}\|_2 + \|\by\|_2\big) \le \eta T\sqrt{n}\|K^m-K\|_\infty,
    \end{align*}
    which completes the proof.
 \end{proof}

Based on the above lemma, we give the proof of Proposition~\ref{pro:gm-gT} as follows. 
    \begin{proof}[Proof of Proposition~\ref{pro:gm-gT}]
    For any $\bx\in\X$ and $k\in[T-1]$, from the definitions we know
    \begin{align*}
         |g^m_{k+1}(\bx) - g_{k+1}(\bx)| 
        & =  \Big|g^m_{k}(\bx) - g_{k}(\bx)- \frac{\eta}{n} \sum_{i=1}^n \big[\big(g^m_k(\bx_i) - g_k(\bx_i)\big)K^m(\bx_i,\bx)  + \big(g_k(\bx_i) - y_i\big)\big(K^m(\bx_i,\bx) - K(\bx_i,\bx)\big)\big]\Big| \nonumber\\
        & \le  \big|g^m_{k}(\bx) - g_{k}(\bx)\big| + \frac{\eta}{n}\sum_{i=1}^n\Big(\|K^m\|_\infty\big|g^m_k(\bx_i) - g_k(\bx_i)\big| + \|K^m-K\|_\infty\big|g_k(\bx_i) - y_i\big|\Big)\nonumber\\
        & \le \big|g^m_{k}(\bx) - g_{k}(\bx)\big| + \frac{\eta}{\sqrt{n}}\big(\|K^m\|_\infty\|\bG^m_k-\bG_k\|_2 + \|K^m-K\|_\infty\|\bG_k-\by\|_2\big),
    \end{align*}
    where the last inequality used Cauchy-Schwarz inequality.
    
    Combining Lemmas \ref{prop:gmk-fk}, \eqref{eq:bound_Km} and \eqref{eq:Fk-y} and with the above observation, we get
    \begin{align*}
        \|g^m_{k+1} - g_{k+1}\|_\infty \le \|g^m_{k} - g_{k}\|_\infty + 6\eta^2T\|K^m-K\|_\infty.
    \end{align*}
    Applying the above inequality recursively and noting that $g^m_{0} = g_{0}$, we have
    \begin{align*}
        \|g^m_{k+1} - g_{k+1}\|_\infty \le 6(\eta T)^2\|K^m-K\|_\infty.
    \end{align*}
    From Lemma \ref{lem:concentrationNTK} and the condition \eqref{eq:m_condition} we know $\|K^m-K\|_\infty \lesssim \frac{\sqrt{L}}{m^{\frac{1}{6}}}$.
Therefore, for any $k\in[T]$ 
    \begin{align*}
        \big\|\bS_mg^m_k - \bS g_k\big\|_\rho^2 &=  \int_\X |g^m_k(\bx) - g_k(\bx)|^2 d\rho_\X(\bx)  \le \|g^m_k - g_k\|_\infty^2 \le  36(\eta T)^4\|K^m-K\|^2_\infty \lesssim \frac{L(\eta T)^4}{m^{\frac{1}{3}}}.
    \end{align*}
  The desired result is obtained by setting $k=T$.
\end{proof}

\subsubsection{Proof of Proposition~\ref{pro:gT-frho}}\label{proof:gT-frho}

To estimate the last term $\big\|\bS g_T - f_{\rho} \big\|_\rho^2$ in \eqref{eq:decom-gd}, we first introduce an intermediate term.
Define the population iteration $h_k$ on $\H_K$ as 
\begin{align}\label{eq:population_iteration}
    h_{k+1} = h_k - \eta \int_\Z \big(\langle h_k, K_\bx\rangle_{\H_K} - y\big)K_\bx d\rho(\bz) \text{ \ with \ $h_0 = 0$.}
\end{align}
If we regard the population risk $\L(\cdot)$ as a functional on $\H_K$, then the population iteration $h_k$ can be viewed as the GD of $\L(\cdot)$ initialized at $h_0=0$.

\begin{lemma}\label{lemma:fh=frho}
    Let $\H$ be the closure of $\H_K$ in $\mathcal{L}^2_{\rho_\bx}$.
    Then, Assumption \ref{ass:frho_smth} implies $f_\rho\in\H$.
\end{lemma}
\begin{proof}
    Note that $\bL$ has the eigen-decomposition $\bL  f=\sum_{i=1}^\infty \lambda_i \langle f, \Phi_i\rangle_{\mathcal{L}^2_{\rho_\bx}} \Phi_i$.
    According to Assumption \ref{ass:frho_smth}, we know there exists a $g\in\mathcal{L}^2_{\rho_\bx}$ such that
    \begin{align*}
        f_\rho = \bL^\beta g = \sum_{i=1}^\infty \lambda_i^\beta\langle g,\Phi_i\rangle_{\mathcal{L}^2_{\rho_\bx}}\Phi_i = \sum_{i:\lambda_i\neq0}^\infty \lambda_i^\beta\langle g,\Phi_i\rangle_{\mathcal{L}^2_{\rho_\bx}}\Phi_i.
    \end{align*}
    Since for any $\lambda_i\neq0$, the associated eigenfunction $\Phi_i\in\H_K$ (see Chapter 4.5 in \cite{steinwart2008support}), we conclude that $f_\rho\in\H$.
\end{proof}
\begin{lemma}\label{lemma:gk-hk}
    Suppose Assumptions \ref{ass:effec_dim} and \ref{ass:frho_smth} hold.
    Assume $\eta \le 1$.
    For any $\delta_1,\delta_2\in(0,1/2)$, assume $\eta T \le n(9\log(n/\delta_2))^{-1}$. 
    Then, the following statements hold with probability at least $1-\delta_1-\delta_2$ over sampling.
    \begin{enumerate}[label=(\alph*), leftmargin=*]
        \item For the case $\beta \ge \frac{1}{2}$, there holds
        \begin{align*}
            \|\bS g_T - \bS h_T\|_\rho \le 4(B+1)(12 + 4\log(T) + \sqrt{2}\eta)\Big(\frac{\sqrt{\eta T}}{n} + \sqrt{\frac{2c_\gamma(\eta T)^\gamma}{n}}\Big)\log\big(\frac{4}{\delta_1}\big).
        \end{align*}
        \item For the case $\beta \in(0,\frac{1}{2})$, there holds
        \begin{align*}
            \|\bS g_T - \bS h_T\|_\rho &\le (12 + 4\log(T) + \sqrt{2}\eta)\bigg(2(6+B)\Big(\frac{\sqrt{\eta T}}{n} + \sqrt{\frac{2c_\gamma(\eta T)^\gamma}{n}}\Big) + \frac{4B\big((\eta T)^{1-\beta} + 1\big)}{n} \bigg)\log\Big(\frac{3T}{\delta_1}\Big).
        \end{align*}
    \end{enumerate}
\end{lemma}

\begin{proof}
    The proof is derived from Theorem 5 in \cite{lin2017optimal}, which provides upper bounds for $\|\S_\rho\nu_{k+1} - \S_\rho\mu_{k+1}\|_\rho$ with two iteration sequences $\{\nu_{k+1}\}$ and $\{\mu_{k+1}\}$. 
    We first show that their assumptions are satisfied in our setting, and then apply their results with our Lemma \ref{lemma:<2} by showing that $\S_\rho\nu_{k+1} - \S_\rho\mu_{k+1}$ is equivalent to $\bS g_k - \bS h_k$.

    Since we assume $|y|\le1$, their Assumption 1 is satisfied with $M=v=1$.
    Instead of using the notations $\bx, \langle \bx, \bx'\rangle_H $ and $ \S_\rho$ in \cite{lin2017optimal} for any $\bx,\bx'\in\X$, we use $K_\bx,$ $\langle K_\bx, K_{\bx'}\rangle_{\H_K} $ and $ \bS$ in our setting.
    Then, their $H_\rho$ is the same as our $\H_K$.
    Since $f_\H$ in \cite{lin2017optimal} is the projection of $f_\rho$ onto the closure of $H_\rho$ in $\mathcal{L}^2_{\rho_\bx}$, from Lemma \ref{lemma:fh=frho} we know their $f_\H$ is equivalent to our $f_\rho$.
    Hence, Assumption 2 in \cite{lin2017optimal} holds true with $\zeta = \beta$ and $R=B$ due to our Assumption \ref{ass:frho_smth}.
    Further, their Assumption 3 is guaranteed by Assumption \ref{ass:effec_dim}, their equation (3) holds true with $\kappa^2 = 1$ due to $\langle K_\bx, K_{\bx'}\rangle_{\H_K} = K(\bx,\bx') \le \|K\|_\infty \le 1$ (see Property \ref{prop:bound-K}). 
    Their equation (47) is guaranteed by Lemma \ref{lemma:<2} with $\kappa=1$, $\Gamma = n$, $\delta = \delta_2$, $\zeta_i = K_{\bx_i}$, $Q = \int_\X K_\bx\otimes K_\bx d\rho_{\bx}$.
    In addition,
    by taking the step-size $\eta_k = \eta$ for all $k\in[T]$, we know $\S_\rho\nu_{k+1} - \S_\rho\mu_{k+1}$ in \cite{lin2017optimal} is equivalent to our $\bS g_k - \bS h_k$.

    Then, combining above observations and Theorem 5 in \cite{lin2017optimal} with $\eta_k = \eta$, $\theta = 0$, $\lambda = (\eta T)^{-1}$, $\kappa = 1$, $M=v=1$, $R = B$, $\zeta = \beta$ , $m=n$ and $k = T$, we get the desired results.
\end{proof}

\begin{lemma}[Proposition 2 in \cite{lin2017optimal}]\label{lemma:hk-frho}
    Suppose Assumption \ref{ass:frho_smth} holds.
    Let $\eta \in(0,1]$ be the step size.
    For any $k\in\mathbb{N}$, there holds
    \begin{align*}
        \|\bS h_k - f_\rho\|_\rho \le B\Big(\frac{\beta}{2\eta k}\Big)^\beta.
    \end{align*}
\end{lemma}
\begin{proof}
    In the proof of Lemma \ref{lemma:gk-hk}, we already showed that $\mathcal{S}_\rho\mu_{k+1}$ and $f_\calH$ in \cite{lin2017optimal} are equivalent to our $\bS h_k$ and $f_\rho$.
    Then, by applying Proposition 2 in \cite{lin2017optimal} with $\eta_k = \eta$, $\kappa=1$, $R=B$ and $\zeta = \beta$, we get the desired results.
\end{proof}

Combining Lemma \ref{lemma:gk-hk} and Lemma \ref{lemma:hk-frho}, we give the proof of Proposition~\ref{pro:gT-frho}.  
\begin{proof}[Proof of Proposition~\ref{pro:gT-frho}]
    Note that $\|\bS g_T - f_\rho\|_\rho^2 \lesssim \|\bS g_T - \bS h_T\|_\rho^2 + \|\bS h_T - f_\rho\|_\rho^2$.
    The desired results are obtained by combining Lemma \ref{lemma:gk-hk} with $\delta_1 = \delta_2 = \frac{\delta}{2}$ and Lemma \ref{lemma:hk-frho}.
\end{proof}

\subsubsection{Proofs for Theorem~\ref{thm:excess-relu} and Corollary~\ref{cor:relu}}\label{proof:GD-results}
\begin{proof}[Proof of Theorem \ref{thm:excess-relu}]
  Combining Propositions~\ref{pro:f-flin}, \ref{pro:flin-gm}, \ref{pro:gm-gT} and \ref{pro:gT-frho} with $\delta$ replaced by $\frac{\delta}{4}$, with probability at least $1-L\exp\big(\O(dL\log(m)) - \Omega(m^{\frac{1}{3}})\big) - \delta$ over initialization $(\ba,\bW(0))$ and sampling, there holds
\[ \L (f_{\bW(T)} ) - \L(f_\rho) \lesssim      \frac{L^\frac{14}{3}(\eta T)^4}{m^{\frac{1}{3}}}+ \Big(\frac{ {\eta T}}{n^2} + \frac{(\eta T)^\gamma\!+\!(\eta T)^{1 - 2\beta}}{n}  \Big) \log^4\!\big(\frac{T}{\delta}\big)\nonumber  + (\eta T)^{-2\beta}. \]
The proof of the theorem is completed. 
\end{proof}

\begin{proof}[Proof of Corollary \ref{cor:relu}]
The proof is derived by Theorem \ref{thm:excess-relu} with $\delta$ replaced by $\delta/2$.
We first prove that the condition $n \ge \frac{16}{\delta}\big(\frac{36(2\beta+\gamma)}{\beta}\big)^{\frac{2\beta+\gamma}{\beta}}$ implies $\eta T\le \frac{n}{ 36\log({16n}/{\delta}) }$. Since $\eta T \le 2n^{\frac{1}{2\beta+\gamma}}$, the condition reduces to show $n^{\frac{2\beta}{2\beta+\gamma}} \ge 72\log(\frac{16n}{\delta})$, which is equivalent to showing
$
\big(\frac{16n}{\delta}\big)^{\frac{2\beta}{2\beta+\gamma}} \ge \frac{36(2\beta+\gamma)}{\beta}\big(\frac{16}{\delta}\big)^{\frac{2\beta}{2\beta+\gamma}}\log\big(\frac{16n}{\delta}\big)^{\frac{2\beta}{2\beta+\gamma}}.
$
From (9.17) and (9.18) in \cite{gyorfi2006distribution} we know $u > 2c\log(c)$ implies $u> c\log(u)$ for any $c\ge e$.
Setting $u = \big(\frac{16n}{\delta}\big)^{\frac{2\beta}{2\beta+\gamma}}$ and $c = \frac{36(2\beta+\gamma)}{\beta}\big(\frac{16}{\delta}\big)^{\frac{2\beta}{2\beta+\gamma}}$ and solving $u\ge c^2$, the desired result is obtained by noting $u \ge c^2 > 2c\log(c)$ for all $c\ge e$.
Combining this with $T=\lceil n^{\frac{1}{2\beta+\gamma}}\rceil$, we know $n \ge \max\big\{\big(\frac{36(2\beta+\gamma)}{\beta}\big)^{\frac{2\beta+\gamma}{\beta}}\frac{16}{\delta}, \eta^{-(2\beta + \gamma)}\big\}$ implies $1\le \eta T\le n(36\log(16n/\delta))^{-1}$. 
Similarly, setting $u = (m/\delta)^{\frac{1}{3}}$ and $c = 3\big(L^{22}d^2(\eta T)^7/\delta\big)^\frac{1}{3}$, and noting $\eta T \asymp n^{\frac{1}{2\beta+\gamma}}$, we know $m \gtrsim L^{22}d^2n^{\frac{7}{2\beta+\gamma}}\log^3(ndL/\delta)$ ensures condition \eqref{eq:m_condition} in Theorem~\ref{thm:excess-relu}. 

Noting that $m \gtrsim L^{14} n^{\frac{6\beta+12}{2\beta+\gamma}}$ ensures $\frac{L^\frac{14}{3}(\eta T)^4 }{m^{\frac{1}{3}}} \lesssim n^{-\frac{2\beta}{2\beta+\gamma}}$ and \eqref{eq:m_condition} implies $L\exp\big(\O(dL\log(m)) - \Omega(m^{\frac{1}{3}})\big) \le \delta/2$. 
Then, from Theorem~\ref{thm:excess-relu} we know with probability at least $1-\delta$ over initialization $(\ba,\bW(0))$ and sampling, there holds
\begin{align*}
    \L (f_{\bW(T)} ) - \L(f_\rho) \lesssim    n^{-\frac{2\beta}{2\beta+\gamma}}+ \Big(\frac{ {\eta T}}{n^2} + \frac{(\eta T)^\gamma\!+\!(\eta T)^{1 - 2\beta}}{n}  \Big)\log^2(T)\log^2\!\big(\frac{T}{\delta}\big)\nonumber  + (\eta T)^{-2\beta}.
\end{align*}
In addition, since $2\beta+\gamma > 1$ and $\eta T \ge 1$, there holds $(\eta T)^{1-2\beta} \le (\eta T)^\gamma$.
Plugging the choice of $\eta T \asymp n^{\frac{1}{2\beta+\gamma}}$ back into the above inequality, we get 
\begin{align*}
    \L (f_{\bW(T)} ) - \L(f_\rho) \lesssim  n^{-\frac{2\beta}{2\beta+\gamma}} \log^4(\frac{n}{\delta}).
\end{align*}
The proof is completed. 
\end{proof}

\subsection{Proofs for Stochastic Gradient Descent}\label{sec:proof-SGD}
In this subsection, we present all proofs for SGD. Section~\ref{proof:SGD-f-flin_sgd}, Section~\ref{proof:sgd-flin-fm} and Section~\ref{proof:sgd-fm-frho} provide detailed proofs for Proposition~\ref{pro:sgd-f-flin}, Proposition~\ref{pro:sgd-flin-fm} and Proposition~\ref{pro:sgd-fm-frho}, respectively. Proofs for Theorem~\ref{thm:excess-relu_sgd} and Corollary~\ref{cor:relu_sgd} are given in Section~\ref{proof:SGD-results}. 

\subsubsection{Proof of Proposition~\ref{pro:sgd-f-flin}}\label{proof:SGD-f-flin_sgd}
We first show that the trajectory of SGD with deep ReLU networks also falls inside local balls around the initialization $\bW(0)$. 
\begin{lemma}\label{lemma:wk-w0_sgd}
Let $\{\bW(k)\}$ be produced by \eqref{eq:sgd-update} with $\eta \le 1/5$.
Assume  \eqref{eq:m_condition_sgd} holds. 
Then, with probability at least $1-L\exp(\O(dL\log(m)) - \Omega(m^{\frac{1}{3}})) - \delta$ over initialization $(\ba,\bW(0))$,   for any $k\in[T]$, there holds 
\[ \big\| \bW(k) - \bW(0) \big\|_{op,\infty}^2 \le \big\|\bW(k) - \bW(0) \big\|_2^2 \le 4\eta k \]
and 
\[|f_{\bW(k)}(\bx) - y| \le CL^2\sqrt{\eta k} + 1 \ \text{ for any } z=(\bx,y)\in\Z. \]
\end{lemma}
\begin{proof}
The first part of the lemma is proved by induction.
It's obvious that $  \| \bW(k) - \bW(0) \|^2_2 \le 0$ holds with $k=0$. Assume, for all $t\in [k]$ with $k\le T-1$, $\| \bW(k) - \bW(0) \|^2_2 \le 4\eta k$ holds.
We will show that  $\| \bW(k+1) - \bW(0) \|^2_2 \le 4\eta(k+1)$.

    From the update rule \eqref{eq:sgd-update}, we know
    \begin{align}\label{eq:wk-w0_sgd}
       & \big\|\bW(k+1) - \bW(0)  \big\|_2^2 = \Big\|\bW(k) - \bW(0) -  \eta \frac{\partial  l(\bW(k);z_{i_k})}{\partial \bW(k)} \Big\|_2^2\nonumber\\
       &= \big\|\bW(k) - \bW(0)\big\|_2^2 + \eta^2  \Big\|\frac{\partial  l(\bW(k);z_{i_k})}{\partial \bW(k)} \Big\|_2^2   + 2\eta\Big\langle \bW(0)- \bW(k),  \frac{\partial  l(\bW(k);z_{i_k})}{\partial \bW(k)}  \Big\rangle_2\nonumber\\
       &= \big\|\bW(k) - \bW(0)\big\|_2^2 + 2\eta^2  l(\bW(k),z_{i_k}) \Big\|\frac{\partial f_{\bW(k)}(\bx_{i_k})}{\partial \bW(k)} \Big\|_2^2  + 2\eta\Big\langle \bW(0)- \bW(k),  \frac{\partial  l(\bW(k);z_{i_k})}{\partial \bW(k)}  \Big\rangle_2,
    \end{align}
   where in the last inequality we have used $\frac{\partial  l(\bW(k);z_{i_k})}{\partial \bW(k)}=(f_{\bW(k)}(\bx_{i_k}) - y_{i_k}) \frac{\partial f_{\bW(k)}(\bx_{i_k})}{\partial \bW(k)}$ and $(f_{\bW(k)}(\bx_{i_k}) - y_{i_k})^2 = 2 l(\bW(k);z_{i_k})$. 

Setting $R_{op} = 2\sqrt{\eta T}$.
By the induction assumption, there holds $\bW(k),\bW(0)\in\calB_{R_{op}}(\bW(0))$.
Then from Lemma \ref{prop:symm=0} (if $\ell < L-1$) and part $(c)$ of Lemma \ref{coro:bound-o} (if $\ell=L$) and  \eqref{eq:diff_derivative} in Lemma \ref{lem:almost-convex} with $\bW = \bW(k)$, we have
\begin{align}\label{eq:bound_partial_wk_sgd}
    \Big\|\frac{\partial f_{\bW(k)} (\bx_{i_k})}{\partial{\bW(k)}}\Big\|_2 &\le   \Big\|\frac{\partial f_{\bW(k)} (\bx_{i_k})}{\partial{\bW(k)}}- \frac{\partial f_{\bW(0)} (\bx_{i_k})}{\partial{\bW(0)}}\Big\|_2 +\Big\|\frac{\partial f_{\bW(0)} (\bx_{i_k})}{\partial{\bW^L(0)}}\Big\|_2\\
    &\le \sqrt{L}\max_{\ell\in[L]}\Big\|\frac{\partial f_{\bW(k)} (\bx_{i_k})}{\partial{\bW^\ell(k)}}- \frac{\partial f_{\bW(0)} (\bx_i)}{\partial{\bW^\ell(0)}}\Big\|_2 + 2  \le \epsilon_3 + 2
\end{align}
with $\epsilon_3 \lesssim L^\frac{7}{3}(\eta T)^{\frac{1}{6}}m^{-\frac{1}{6}}$. 
 
Further, from the induction assumption we know $|f_{\bW(k)}(\bx) - y| \le CL^2R_{op}$ and \eqref{eq:semi-smth} in Lemma \ref{lem:almost-convex} with $\bW = \bW(k)$, $\widetilde{\bW} = \bW(0)$ and $\|\bW(k)-\bW(0)\|_{op,\infty} \le 2\sqrt{\eta T}$ implies
\begin{align*}
    2\eta\Big\langle \bW(0)\!-\! \bW(k),   \frac{\partial  l(\bW(k);z_{i_k})}{\partial \bW(k)}  \Big\rangle_2 &\le 2\eta\big( l(\bW(0),z_{i_k})  -  l(\bW(k),z_{i_k})\big) + 2\eta\epsilon_2 
\end{align*}
with $\epsilon_2 \lesssim L^\frac{13}{3}(\eta T)^{\frac{7}{6}}m^{-\frac{1}{6}}.$

Plugging the above two estimates back into \eqref{eq:wk-w0_sgd}, we get
\begin{align*}
      \big\|\bW(k+ 1) - \bW(0)\big\|_2^2   
    &\le \big\|\bW(k) - \bW(0)\big\|_2^2 + 2\eta^2 l(\bW(k),z_{i_k})( \epsilon_3 + 2)^2  + 2\eta\big( l(\bW(0),z_{i_k})  -  l(\bW(k),z_{i_k})\big) + 2\eta \epsilon_2\\
    &\le \big\|\bW(k) - \bW(0)\big\|_2^2 + 10\eta^2 l(\bW(k),z_{i_k})  + 2\eta\big( l(\bW(0),z_{i_k})  -  l(\bW(k),z_{i_k})\big) + 2\eta\\
    &\le \big\|\bW(k) - \bW(0)\big\|_2^2 +  2\eta l(\bW(0),z_{i_k}) + 2\eta\\
    &\le 4\eta k + 3\eta \le 4\eta(k+1),
 \end{align*}
where in the second inequality we have used $\epsilon_3\le5-\sqrt{2}$ and $\epsilon_2\le1$ implied by \eqref{eq:m_condition_sgd}, in the third inequality we have used $ 10\eta^2  \le 2\eta$ by noting $\eta \le 1/5$, in the last second   inequality we have used $ l(\bW(0),z_{i_k})  \le 1/2$ by observing $f_{\bW(0)} = 0$  and  the induction assumption $\|\bW(k)- \bW(0)\|^2_2\le 4\eta k$. The first part of the lemma is proved. 

Combining Lemma \ref{lemma:o-o_0} with $\|\bW(k)- \bW(0)\|_{op,\infty}^2 \le \|\bW(k)- \bW(0)\|^2_2\le 4\eta T = R_{op}^2$, we know
\begin{align*}
    |f_{\bW(k)}(\bx) - y| &\le |f_{\bW(k)}(\bx) - f_{\bW(0)}(\bx)| + |f_{\bW(0)}(\bx) - y| \le \|\ba\|_2\|o^L_{k}(\bx) - o^L_0(\bx)\|_2 + 1\\
    &\le CL^2\sqrt{\eta k} + 1,
\end{align*}
which completes the proof for the second part of the lemma.
\end{proof}

The proof of Proposition~\ref{pro:sgd-f-flin} is presented as follows. 
\begin{proof}[Proof of Proposition~\ref{pro:sgd-f-flin}] 
The proof is similar to that of Proposition \ref{pro:f-flin}.
    Setting $R_{op} = 2\sqrt{\eta T}$.
    Combining Lemma \ref{lemma:wk-w0_sgd} and \eqref{eq:linear_approx} in Lemma \ref{lem:almost-convex} with $\widetilde{\bW} = \bW(k)$ and $\bW = \bW(0)$, we get the desired results. 
\end{proof}

\subsubsection{Proofs for Proposition~\ref{pro:sgd-flin-fm}}\label{proof:sgd-flin-fm}
Recall that the operator $f\otimes f$ on Hilbert space $H$ is defined by $(f\otimes f) g = \langle f, g\rangle_H f$ for all $f,g\in H$.
Based on Lemma~\ref{lemma:wk-w0_sgd}, we give the proof of Proposition~\ref{pro:sgd-flin-fm} as follows. 
\begin{proof}[Proof of Proposition~\ref{pro:sgd-flin-fm}] 
    Denote $\epsilon_k = f_{\bW(k+1)}^{\text{lin}} - f_{\bW(k)}^{\text{lin}} + \eta(f_{\bW(k) }^{\text{lin}}(\bx_{i_k}) - y_{i_k})K^m_{\bx_{i_k}} \in\H_m$.
From the update rule of $ f_{k}^{m}$ \eqref{eq:update-kernelSGD}, we know
\begin{align}
    f^{\text{lin}}_{\bW(k+1)} - f^m_{k+1} &= \big(f^{\text{lin}}_{\bW(k)} -f^m_k\big) -\eta\big(f^{\text{lin}}_{\bW(k)}(\bx_{i_k})- f^m_k(\bx_{i_k})\big)K^m_{\bx_{i_k}} + \epsilon_k \nonumber\\
    &= \big(f^{\text{lin}}_{\bW(k)} -f^m_k\big) - \eta\big\langle f^{\text{lin}}_{\bW(k)} -f^m_k, K^m_{\bx_{i_k}}\big\rangle_{\H_m}K^m_{\bx_{i_k}} + \epsilon_k\nonumber\\
    &= \big(\bfI - \eta K^m_{\bx_{i_k}}\otimes K^m_{\bx_{i_k}}\big)\big(f^{\text{lin}}_{\bW(k)} -f^m_k\big) + \epsilon_k\nonumber, 
\end{align}
where the second equality follows from the fact $f^{\text{lin}}_{\bW(k) } - f_k^m \in \H_m$ and the reproducing kernel property $f_{\bW(k)}^{\text{lin}}(\bx_{i_k}) - f^m_k(\bx_{i_k}) = \langle f_{\bW(k)}^{\text{lin}} - f^m_k,K^m_{\bx_{i_k}}\rangle_{\H_m}.$

Applying the above equality recursively, we get
\begin{align*}
   f^{\text{lin}}_{\bW(k+1)} - f^m_{k+1}
    &= \sum_{s=0}^{k}\prod_{a=s+1}^k \big(\bfI - \eta K^m_{\bx_{i_a}}\otimes K^m_{\bx_{i_a}}\big)\epsilon_{s},
\end{align*}
where we used the conventional notation $\prod_{k+1}^k = \bfI$ for any $k\in\mathbb{N}$.
Note that for any $a\in[k]$ and $i_a$, $\eta K^m_{\bx_{i_a}}\otimes K^m_{\bx_{i_a}}$ is self-adjoint and positive, and from \eqref{eq:bound_Km} we know $\eta\| K^m_{\bx_{i_a}}\otimes K^m_{\bx_{i_a}}\|_{op} = \eta\|K^m_{\bx_{i_a}}\|^2_{\H_m} \le \eta \|K^m\|_\infty \le 4\eta \le 1$.
Then, $\|\bfI - \eta K^m_{\bx_{i_a}}\otimes K^m_{\bx_{i_a}}\|_{op} \le 1$.

According to the above inequality, we have
\begin{align}\label{eq:sgd-flin-fm-sum}
    &\big\|f^{\text{lin}}_{\bW(k+1)} - f^m_{k+1}\big\|_\infty = \sup_{\bx\in\X} \big|\big\langle f^{\text{lin}}_{\bW(k+1)} - f^m_{k+1}, K^m_\bx\big\rangle_{\H_m}\big| \le \sup_{\bx\in\X} \big\|f^{\text{lin}}_{\bW(k+1)} - f^m_{k+1}\big\|_{\H_m} \|K^m_\bx\|_{\H_m}  \nonumber\\
    &\le  \big\|f^{\text{lin}}_{\bW(k+1)} - f^m_{k+1}\big\|_{\H_m} \sqrt{\|K^m\|_\infty} \le 2\sum_{s=0}^{k}\prod_{a=s+1}^k \big\|\bfI - \eta K^m_{\bx_{i_a}}\otimes K^m_{\bx_{i_a}}\big\|_{op} \big\|\epsilon_{s}\big\|_{\H_m}  \le 2\sum_{s=0}^{k} \big\|\epsilon_{s}\big\|_{\H_m},
\end{align}
where in the first equality we have used the reproducing kernel property and in the last second inequality we have used \eqref{eq:bound_Km} with $\sqrt{\|K^m\|_\infty}\le 2$.

Now, we turn to estimate $\|\epsilon_k\|_{\H_m}$.
For any $k\in[T]$, from the definition of $f^{\text{lin}}_{\bW }$ and the update rule of SGD \eqref{eq:sgd-update}, there holds
    \begin{align}
        &\epsilon_k(\bx) = f^{\text{lin}}_{\bW(k+1) }(\bx)  - f^{\text{lin}}_{\bW(k)} (\bx) + \eta\big(f^{\text{lin}}_{\bW(k)} (\bx_{i_k}) - y_{i_k}\big)K^m_{\bx_{i_k}}(\bx)\nonumber\\
        &= \Big\langle \frac{\partial f_{\bW(0)}(\bx)}{\partial\bW(0)}, \bW(k+1)- \bW(k) \Big\rangle_2 + \eta\big(f^{\text{lin}}_{\bW(k)} (\bx_{i_k}) - y_{i_k}\big)\Big\langle \frac{\partial f_{\bW(0)}(\bx_{i_k})}{\partial\bW(0)}, \frac{\partial f_{\bW(0)}(\bx)}{\partial\bW(0)} \Big\rangle_2 \nonumber\\
        &= \Big\langle \frac{\partial f_{\bW(0)}(\bx)}{\partial\bW^L(0)}, \bW^L(k+1)- \bW^L(k) \Big\rangle_2 + \eta\big(f^{\text{lin}}_{\bW(k)} (\bx_{i_k}) - y_{i_k}\big)\Big\langle \frac{\partial f_{\bW(0)}(\bx_{i_k})}{\partial\bW^L(0)}, \frac{\partial f_{\bW(0)}(\bx)}{\partial\bW^L(0)} \Big\rangle_2 \nonumber\\
        & = \eta\Big[\big(y_{i_k} - f_{\bW(k)}(\bx_{i_k})\big)\Big\langle \frac{\partial f_{\bW(k)}(\bx_{i_k})}{\partial\bW^L(k)}, \frac{\partial f_{\bW(0)}(\bx)}{\partial\bW^L(0)} \Big\rangle_2  + \big(f^{\text{lin}}_{\bW(k)}(\bx_{i_k}) - y_{i_k}\big) \Big\langle \frac{\partial f_{\bW(0)}(\bx_{i_k})}{\partial\bW^L(0)}, \frac{\partial f_{\bW(0)}(\bx)}{\partial\bW^L(0)} \Big\rangle_2\Big] \nonumber\\
        &=\bigg\langle \eta  \Big[\big(y_{i_k} - f_{\bW(k)}(\bx_{i_k})\big)\frac{\partial f_{\bW(k)}(\bx_{i_k})}{\partial\bW^L(k)} + \big(f^{\text{lin}}_{\bW(k)}(\bx_{i_k}) - y_{i_k}\big)\frac{\partial f_{\bW(0)}(\bx_{i_k})}{\partial\bW^L(0)}\Big] , \frac{\partial f_{\bW(0)}(\bx)}{\partial\bW^L(0)}\bigg\rangle_2 \nonumber\\
        & =: \Big\langle \Delta(k), \frac{\partial f_{\bW(0)}(\bx)}{\partial\bW^L(0)} \Big\rangle_2, \nonumber
    \end{align}
    where the second equality is due to $K^m_{\bx_{i_k}}(\bx) = K^m(\bx_{i_k},\bx)$, the third equality used $\frac{\partial f_{\bW(0)}}{\partial\bW^\ell(0)} = 0$ for $\ell\in[L-1]$ according to Lemma \ref{prop:symm=0}, and the fourth equality is according to the update rule \eqref{eq:sgd-update}, and in the last equality we define
    \begin{align*}
         \Delta(k)& := \eta  \Big[\big(y_{i_k} - f_{\bW(k)}(\bx_{i_k})\big)\frac{\partial f_{\bW(k)}(\bx_{i_k})}{\partial\bW^L(k)} + \big(f^{\text{lin}}_{\bW(k)}(\bx_{i_k}) - y_{i_k}\big)\frac{\partial f_{\bW(0)}(\bx_{i_k})}{\partial\bW^L(0)}\Big]\\
        & =  \eta \Big[\big(y_{i_k} -  f_{\bW(k)}(\bx_{i_k})\big)\Big(\frac{\partial f_{\bW\!(k)}(\bx_{i_k} )}{\partial\bW^L(k)}  - \frac{\partial f_{\bW (0)}(\bx_{i_k} )}{\partial\bW^L(0)}\Big)  +  \big(f^{\text{lin}}_{\bW (k)}(\bx_{i_k} ) - f_{\bW\!(k)}(\bx_{i_k} )\big)\frac{\partial f_{\bW (0)}(\bx_{i_k} )}{\partial\bW^L(0)}\Big].
    \end{align*}
    
    Let $\boldsymbol{\Delta}(k) = (0,\ldots,0,\Delta(k))\in\W$, then  $\epsilon_k(\bx) = \langle \boldsymbol{\Delta}(k),\Phi_m(\bx)\rangle_2$.
    There holds
    \begin{align}\label{eq:epsilon2}
        &\|\epsilon_k\|_{H_m} \le \|\Delta(k)\|_2\nonumber\\
        &\le \eta \Big[\big|y_{i_k} -  f_{\bW(k)}(\bx_{i_k})\big|\Big\|\frac{\partial f_{\bW(k)}(\bx_{i_k})}{\partial\bW^L(k)} - \frac{\partial f_{\bW(0)}(\bx_{i_k})}{\partial\bW^L(0)}\Big\|_2  + \big|f^{\text{lin}}_{\bW(k)}(\bx_{i_k})  -  f_{\bW(k)}(\bx_{i_k})\big|  \Big\|\frac{\partial f_{\bW(0)}(\bx_{i_k})}{\partial\bW^L(0)}\Big\|_2\Big] \nonumber\\
        &\le \eta \Big( \sup_{z\in\Z}|f_{\bW(k)}(\bx)-y|\sup_{\bx\in\X}\Big\|\frac{\partial f_{\bW(k)}(\bx)}{\partial\bW^L(k)}  -  \frac{\partial f_{\bW(0)}(\bx)}{\partial\bW^L(0)}\Big\|_2 + \big\|f^{\text{lin}}_{\bW(k)} - f_{\bW(k)}\big\|_\infty \sup_{\bx\in\X}\Big\|\frac{\partial f_{\bW(0)}(\bx)}{\partial\bW^L(0)}\Big\|_2\Big).
    \end{align}

    From part $(c)$ in Lemma \ref{coro:bound-o} we know $\sup_{\bx\in\X}\|\frac{\partial f_{\bW(0)}(\bx)}{\partial\bW^L(0)}\|_2 \le 2$.
    Setting $R_{op} = 2\sqrt{\eta T}$, from Lemma \ref{lemma:wk-w0_sgd} we know $\sup_{z\in\Z}|f_{\bW(k)}(\bx) - y| \le CL^2\sqrt{\eta k} + 1$ and $\bW(k) \in \calB_{R_{op}}(\bW(0))$ for any $k\in[T]$.
    Combining this and Lemma \ref{lem:almost-convex} with $R_{op} = 2\sqrt{\eta T}$, there holds
    \begin{align*}
        \sup_{z\in\Z}|f_{\bW(k)}(\bx)-y|\sup_{\bx\in\X}\Big\|\frac{\partial f_{\bW(k)}(\bx)}{\partial\bW^L(k)}  -  \frac{\partial f_{\bW(0)}(\bx)}{\partial\bW^L(0)}\Big\|_2 \lesssim \frac{L^\frac{10}{3}(\eta T)^\frac{2}{3}}{m^{\frac{1}{6}}}.
    \end{align*}
    According to Proposition \ref{pro:sgd-f-flin}, we know $\|f_{\bW(k)} - f_{\bW(k)}^{\text{lin}}\|_\infty \lesssim L^\frac{7}{3}(\eta T)^{\frac{2}{3}}m^{-\frac{1}{6}}$.
    Plugging the above estimates back into \eqref{eq:epsilon2}, we know with probability at least $1-L\exp(\O(dL\log(m)) - \Omega(m^{\frac{1}{3}}))-\delta$ over initialization $(\ba,\bW(0))$, there holds
    \begin{align*}
        \|\epsilon_k\|_{\H_m}\lesssim \frac{L^\frac{10}{3}\eta^{\frac{5}{3}}T^{\frac{2}{3}}}{m^{\frac{1}{6}}} .
    \end{align*}
Plugging the estimate of $\|\epsilon_s\|_{\H_m}$ back into \eqref{eq:sgd-flin-fm-sum},  with probability at least $1-L\exp(\O(dL\log(m)) - \Omega(m^{\frac{1}{3}}))-\delta$ over the random choice of $(\ba,\bW(0))$, there holds
\begin{align*}
    \|f^{\text{lin}}_{\bW(k) } - f_k^m\|_\infty \le \sum_{s=0}^k\|\epsilon_s\|_{\H_m} \lesssim \frac{L^\frac{10}{3}(\eta T)^{\frac{5}{3}}}{m^{\frac{1}{6}}}.
\end{align*}
This completes the proof of the proposition.
\end{proof}

\subsubsection{Proof of Proposition \ref{pro:sgd-fm-frho}}\label{proof:sgd-fm-frho}
To control $ \|\bS_mf^m_T - f_\rho\|^2_\rho$, we first control $ \|\bS_mf^m_T - \bS_mg^m_T \|^2_\rho$, i.e., the distance between the SGD and GD on $\H_m$.
\begin{lemma}\label{lem:sgd-fm-g}
     Let $\delta \in (0,1)$ and $T\in\mathbb{N}$. Suppose 
    $
        0 < \eta \le \frac{1}{32(\log (T) + 1)} \text{ and } \frac{1}{\eta T} \ge \frac{36}{n}\log\big(\frac{2n}{\delta}\big),
    $ and $m\gtrsim dL^3\log^3(m/\delta)$,
    Then, with probability at least $1-L\exp\big(\O(d\log(m)) - \Omega(m^{\frac{1}{3}})\big)- \delta$ over initialization $(\bW(0),\ba)$ and sampling, there holds
    $$
        \ebb_\A\big[\big\|\bS_m(f^m_T - g^m_T)\big\|^2_\rho\big] \lesssim  \eta(\log (T) \vee 1).
    $$
\end{lemma}
\begin{proof}
The lemma is proved by using Proposition 6 in \cite{lin2017optimal}, which provides upper bounds for $\|\S_\rho\omega_{T+1} - \S_\rho\nu_{T+1}\|_\rho$.
    We first show that their assumptions are satisfied in our setting, and then apply their results with our Lemma \ref{lemma:<2} by showing that $\S_\rho\nu_{T+1} - \S_\rho\mu_{T+1}$ is equivalent to $\bS_m f^m_T - \bS_m g^m_T$.

    Note we assume $|y|\le1$, then their Assumption 1 is satisfied with $M=v=1$.
    Instead of using the notations $\bx, \langle \bx, \bx'\rangle_H $ and $ \S_\rho$ in \cite{lin2017optimal} for any $\bx,\bx'\in\X$, we use $K^m_\bx,$ $\langle K^m_\bx, K^m_{\bx'}\rangle_{\H_m} $ and $ \bS_m$ in our setting.
    With probability at least $1-\delta/2$ over the random choice of $\bW(0)$, their equation (3) holds true with $\kappa^2 = 4$ due to $\langle K^m_\bx, K^m_{\bx'}\rangle_{\H_m} = K^m(\bx,\bx') \le \|K^m\|_\infty \le 4$ according to \eqref{eq:bound_Km}.
    Their equation (47) is guaranteed with probability at least $1-\delta/2$ over sampling by Lemma \ref{lemma:<2} with $\kappa=2$, $\Gamma = n$, $\zeta_i = K^m_{\bx_i}$, $Q = \int_\X K^m_\bx\otimes K^m_\bx d\rho_{\bx}$, and $\lambda = (\eta T)^{-1}$.
    In addition,
    by taking the batch-size $b=1$ and the step size $\eta_k = \eta$ for all $k\in[T]$, we know $\S_\rho\nu_{T+1} - \S_\rho\mu_{T+1}$ in \cite{lin2017optimal} is equivalent to our $\bS_m f^m_T - \bS_m g^m_T$.

    Then, combining above observations and Proposition 6 in \cite{lin2017optimal} with $\eta_k = \eta$, $\theta = 0$, $\lambda = (\eta T)^{-1}$, $\kappa = 2$, $M=v=1$ and $b = 1$, we get the desired results.
\end{proof}
Now, we present the proof of Proposition~\ref{pro:sgd-fm-frho}. 
\begin{proof}[Proof of Proposition~\ref{pro:sgd-fm-frho}]
    Note that
    \begin{align*}
        \ebb_\A[\|\bS_mf^m_T - f_\rho\|^2_\rho] \lesssim \ebb_\A[\|\bS_m f^m_T - \bS_m g^m_T\|^2_\rho] + \|\bS_m g^m_T - \bS g_T\|^2_\rho + \|\bS g_T - f_\rho\|^2_\rho.
    \end{align*}
    Then, the desired results are obtained by combining Lemma \ref{lem:sgd-fm-g} with $\delta$ replaced by $\delta/3$, Proposition~\ref{pro:gm-gT}with $\delta$ replaced by $\delta/3$, and Proposition \ref{pro:gT-frho} with $\delta$ replaced by $\delta/3$.
\end{proof}

\subsubsection{Proofs for Theorem~\ref{thm:excess-relu_sgd} and Corollary~\ref{cor:relu_sgd}}\label{proof:SGD-results}
\begin{proof}[Proof of Theorem \ref{thm:excess-relu_sgd}]
  Combining Propositions \ref{pro:sgd-f-flin}, \ref{pro:sgd-flin-fm} and  \ref{pro:sgd-fm-frho}  with $\delta$ replaced by $\delta/3$, the desired result is obtained.  
\end{proof}

\begin{proof}[Proof of Corollary \ref{cor:relu_sgd}]
    It's obvious that the inequality $\eta T\le n(36\log({24n}/{\delta}))^{-1}$ holds for $\eta=(72\log(24n/\delta))^{-1} n^{-\frac{2\beta}{2\beta+\gamma}} $   and $ T = \lceil n^{\frac{2\beta+1}{2\beta+\gamma}}\rceil$.
    Similar to the proof of Corollary \ref{cor:relu}, one can check that $n \ge (72(2\beta+\gamma))^{2(2\beta+\gamma)}(\frac{24}{\delta})$ implies $\eta T \ge 1$ and $\eta \le \frac{1}{32(\log(T)+1)}$.
    Note that the choices of $\eta$ and $T$ implies $\eta\log(T) + (\frac{ {\eta T}}{n^2} + \frac{(\eta T)^\gamma+(\eta T)^{1 - 2\beta}}{n} )\log^2(T)\log^2(\frac{T}{\delta})  + (\eta T)^{-2\beta} \lesssim n^{-\frac{2\beta}{2\beta+\gamma}}\log^2(n)\log^{2\beta}(\frac{n}{\delta}).$
    Further, according to the proof of Corollary \ref{cor:relu}, one can also show that $m \gtrsim L^{20}\max\{L^6d^3n^{\frac{7}{2\beta+\gamma}}\log^3(ndL/\delta), n^{\frac{6\beta+12}{2\beta+\gamma}}\}$ indicates \eqref{eq:m_condition_sgd} and $\frac{L^\frac{20}{3}(\eta T)^4}{m^{\frac{1}{3}}} \lesssim n^{-\frac{2\beta}{2\beta+\gamma}}$.
    In addition, note that condition \eqref{eq:m_condition_sgd} implies $L\exp\big(\O(dL\log(m)) - \Omega(m^{\frac{1}{3}})\big) \le \delta/2$.
    Combining the above observations with Theorem \ref{thm:excess-relu_sgd} with $\delta$ replaced by $\delta/2$ yields the desired results.
\end{proof}

\section{Discussion on Non-Symmetric Initialization}\label{sec:NTK}
In this subsection, we discuss the uniform concentration of the NTK with non-symmetric initialization and show that our symmetric trick does not affect the concentration properties of the NTK. We consider the following initialization
\begin{align}\label{eq:nonsymmetric-init}
       & \textit{for the first layer: } \ \ \bw_r^1(0) \overset{\text{}}{\sim} \mathcal{N}(0, \mathbf{I}_d), \   \textit{ for } \ell = 2,\ldots,L:   \  \bw_r^\ell(0) \overset{\text{i.i.d.}}{\sim} \mathcal{N}(0,   \mathbf{I}_m)  \textit{  for all } r\in[m],\nonumber\\
     & \textit{for the output layer:} \ \  a_r\overset{\text{i.i.d.}}{\sim} \textit{Unif }(\{-1,1\}) \textit{ for } r\in[m]. 
    \end{align}
  Indeed, the symmetric setting can be seen as a special case of the above general setting. We will show that for this general setting, the results of Lemma~\ref{lem:concentrationNTK} still hold with the same convergence rates. Note that $K^m$ and $K$ are different between the symmetric and non-symmetric settings. We first give their definitions as follows.  
  
As discussed in \cite{jacot2018neural,xu2024overparametrized}, the NTK $K:\X\times\X\to\R$ for deep ReLU networks with initialization \eqref{eq:nonsymmetric-init} is defined, for any $\bx,\bx'\in\X$, by
\begin{align}\label{eq:K_original}
    K(\bx,\bx') = \sum_{\ell=1}^L K^\ell(\bx,\bx') = \sum_{\ell=1}^L2\ebb[\sigma(U^{\ell-1}(\bx))\sigma(U^{\ell-1}(\bx')) ]\prod_{h=\ell}^L q^h(\bx,\bx'),
\end{align}
where $(U^{\ell}(\bx),U^{\ell}(\bx'))$ is a pair of bivariate normal variables defined iteratively by
\begin{align}\label{eq:def_Ux}
(U^{\ell}(\bx),U^{\ell}(\bx'))\sim \N(0,\Sigma^{\ell-1}(\bx,\bx'))
\end{align}
with
\begin{align*}
    \Sigma^{\ell-1}(\bx,\bx') = 2\left(\begin{aligned}
        &\ebb[\sigma^2(U^{\ell - 1}(\bx)) ] &&\ebb[\sigma(U^{\ell-1}(\bx))\sigma(U^{\ell-1}(\bx')) ]\\
        \ebb[&\sigma(U^{\ell-1}(\bx)) \sigma(U^{\ell-1}(\bx')) ] &&\ \ \ \ \ \ebb[\sigma^2(U^{\ell - 1}(\bx')) ]
    \end{aligned}
    \right)
\end{align*}
and
\begin{align*}
    \ebb[\sigma(U^0(\bx))\sigma(U^0(\bx')) ] = \langle \bx,\bx'\rangle_2 \ \text{ and }\ \Sigma^{0}(\bx,\bx') = \left(\begin{aligned}
        &1 &&\langle \bx,\bx'\rangle_2\\
        \langle \bx&,\bx'\rangle_2 &&\ \ \ 1
    \end{aligned}
    \right), 
\end{align*}
and $q^\ell(\bx,\bx') = \frac{\pi - \arccos(p^{\ell-1}(\bx,\bx'))}{\pi}$ with
\begin{align*}
    p^{\ell-1}(\bx,\bx') = \frac{\ebb[\sigma(U^{\ell-1}(\bx)) \sigma(U^{\ell-1}(\bx')) ]}{\sqrt{\ebb[\sigma^2(U^{\ell - 1}(\bx)) ]}\sqrt{\ebb[\sigma^2(U^{\ell - 1}(\bx')) ]}}.
\end{align*}
Note under the symmetric initialization \eqref{eq:initialization}, $K(\bx,\bx')$ degenerates to $K^L(\bx,
\bx')$.

Similar to Lemma \ref{coro:bound-o}, the following results still hold under non-symmetric initialization.
\begin{lemma}\label{coro:bound-o-general}
    The following statements hold with probability at least $1-\delta$ over initialization $\bW(0)$ for all $\ell\in[L].$
    \begin{enumerate}[label=(\alph*), leftmargin=*]
        \item Assume $m\gtrsim dL\log(\frac{1}{\delta}),$ there holds $\sup_{\bx\in\X}\big|\|o^\ell_0(\bx)\|_2 - 1\big| \le C\ell\sqrt{\frac{dL\log(m/\delta)}{m}}.$
        \item Assume $m\gtrsim dL\log(\frac{m}{\delta}),$ there holds $\sup_{\bx\in\X}\|\bV^\ell_{L,0}(\bx)\|_{op} \le \frac{CL}{\sqrt{m}}.$
        \item Assume $m\gtrsim dL^3\log(\frac{m}{\delta}),$ there holds $\sup_{\bx\in\X}\|\frac{\partial f_{\bW(0)}(\bx)}{\partial \bW^\ell(0)}\|_{2} \le CL.$
    \end{enumerate}
\end{lemma}
\begin{proof}
    The proofs of the first two parts are the same as those of Lemma \ref{coro:bound-o}.
    We only prove part $(c)$ here.
    According to the first two parts, for any $\ell\in[L],$ there holds
    \begin{align*}
      \sup_{\bx\in\X}\Big\|\frac{\partial f_\bW(0)(\bx)}{\partial \bW^\ell(0)}\Big\|_{2}  &= \sup_{\bx\in\X} \Big\|\bV^\ell_{L,0}(\bx) \ba \big(o^{\ell-1}_0(\bx)\big)^\top\Big\|_2 \le \sqrt{m}\sup_{\bx\in\X}\|\bV^\ell_{L,0}(\bx)\|_{op}\sup_{\bx\in\X}\|o^{\ell-1}_0(\bx)\|_{2}\\
      &\le  CL\Big(C\ell\sqrt{\frac{dL\log(m/\delta)}{m}} + 1\Big) \le CL,
    \end{align*}
    where the last inequality follows from the condition $m\gtrsim dL^3\log(\frac{m}{\delta})$.
    This completes the proof.
\end{proof}

Now, we give the concentration results of the general case.
The proof is similar to that of Lemma \ref{lem:concentrationNTK}.
Recall the definition of $K^\ell$ (see \eqref{eq:K_original}), similarly we define $K^{m,\ell}(\bx,\bx') = \big\langle \frac{\partial f_\bW(0)(\bx)}{\partial \bW^\ell(0)}, \frac{\partial f_\bW(0)(\bx')}{\partial \bW^\ell(0)} \big\rangle_2$ for all $\ell\in[L]$ and $\bx,\bx'\in X$.
\begin{lemma}\label{lem:concentration}
    Let $\delta\in(0,1)$.
    Assume $m \gtrsim dL^3\log(\frac{m}{\delta})$.
    With probability at least $1-L\exp(\O(dL\log(m))-\Omega(m^{\frac{1}{3}}))) - \delta$ over initialization $(\ba,\bW(0))$, for all $\ell\!\in\![L]$, there holds
    \begin{align*}
        \big\|K^{m,\ell} - K^\ell \big\|_\infty \lesssim  \sqrt{L} m^{-\frac{1}{6}}    + \sqrt{ dL\log(m) m^{-1}} + L^3 m^{-\frac{1}{3}}.
    \end{align*}
\end{lemma}
\begin{proof}
For all $\ell\in[L]$, instead of using the estimates $\sup_{\bx\in\X}\|\bV^\ell_{k,0}(\bx)\|_{op} \le c_0^{k-\ell}m^{-\frac{1}{2}}$ in the proof of Lemma 33 in \cite{xu2024overparametrized}, we employ more finer estimate $\sup_{\bx\in\X}\|\bV^\ell_{k,0}(\bx)\|_{op} \le CLm^{-\frac{1}{2}}$ in Lemma \ref{coro:bound-o-general}.
Then, the term $|(\bV^\ell_{L,0}(\bx)\ba)^\top\bV^\ell_{L,0}(\bx')\ba|$ can be controlled by $CL^2$.
Combining this with \eqref{eq:bound_Uell} yields that
\begin{align*}
     \big\|K^{m,\ell} - K^\ell\big\|_\infty 
    &=  \sup_{\bx,\bx'\in\X}\!\!\Big|\langle o^{\ell-1}_0(\bx), \! o^{\ell-1}_0(\bx')\rangle_2(\bV^\ell_{L,0}(\bx)\ba)^\top\bV^\ell_{L,0}(\bx')\ba \!-\! 2\ebb[\sigma(U^{\ell-1}(\bx))\sigma(U^{\ell-1}(\bx'))] \!\prod_{h=\ell}^Lq^h(\bx,\bx') \Big| \nonumber\\
    &\le \sup_{\bx,\bx'\in\X}\big|\langle o^{\ell-1}_0(\bx), o^{\ell-1}_0(\bx')\rangle_2 - 2\ebb[\sigma(U^{\ell-1}(\bx))\sigma(U^{\ell-1}(\bx'))]\big| \cdot \big|(\bV^\ell_{L,0}(\bx)\ba)^\top\bV^\ell_{L,0}(\bx')\ba\big|\nonumber\\
    &\quad + \sup_{\bx,\bx'\in\X}\big|2\ebb[\sigma(U^{\ell-1}(\bx))\sigma(U^{\ell-1}(\bx'))]\big| \cdot \Big|(\bV^\ell_{L,0}(\bx)\ba)^\top\bV^\ell_{L,0}(\bx')\ba - tr\big(\bV^\ell_{L,0}(\bx)^\top\bV^\ell_{L,0}(\bx')\big)\Big|\nonumber\\
    &\quad + \sup_{\bx,\bx'\in\X}\big|2\ebb[\sigma(U^{\ell-1}(\bx))\sigma(U^{\ell-1}(\bx'))]\big| \cdot \Big|tr\big(\bV^\ell_{L,0}(\bx)^\top\bV^\ell_{L,0}(\bx')\big) - \prod_{h=\ell}^Lq^h(\bx,\bx')\Big|\nonumber\\
    &\!\lesssim   L^2\sup_{\bx,\bx'\in\X}\big|\langle o^{\ell-1}_0(\bx), o^{\ell-1}_0(\bx')\rangle_2 - 2\ebb[\sigma(U^{\ell-1}(\bx))\sigma(U^{\ell-1}(\bx'))]\big| \nonumber\\
    &\quad + \sup_{\bx,\bx'\in\X} \Big|(\bV^\ell_{L,0}(\bx)\ba)^\top\bV^\ell_{L,0}(\bx')\ba - tr\big(\bV^\ell_{L,0}(\bx)^\top\bV^\ell_{L,0}(\bx')\big)\Big| + \sup_{\bx,\bx'\in\X} \Big|tr\big(\bV^\ell_{L,0}(\bx)^\top\bV^\ell_{L,0}(\bx')\big) - \prod_{h=\ell}^Lq^h(\bx,\bx')\Big|\nonumber\\
    &\!=: \calE_1^\ell + \calE_2^\ell + \calE_3^\ell, 
\end{align*}
The estimates of the above three terms $\calE_1, \calE_2 , \calE_3$ are given as follows.

\noindent\textbf{Estimate of $\calE_1^\ell$}: The estimate of $\calE_1^\ell$ follows the same proof steps as in Lemma 6 in \cite{xu2024overparametrized}. 
According to Lemma 6 in \cite{xu2024overparametrized}, one can get that  $\calE_1\lesssim  {LC^L}{m^{-\frac{1}{3}}}$. 
We improve this estimate  from $ {LC^L}{m^{-\frac{1}{3}}}$ to $ {L^3}{m^{-\frac{1}{3}}}$ by using more finer estimates of initialization terms. 
Specifically, instead of using their estimate $\sup_{\bx}\|o^\ell_0(\bx)\|_2 \le c_0^\ell$ in Lemma 30 of \cite{xu2024overparametrized}, we apply the tight estimate  $\sup_{\bx}\|o^\ell_0(\bx)\|_2 \le C$ according to part (a) of Lemma \ref{coro:bound-o-general}.
    In addition, we set $V_0$ to be a $c_0^{-L}m^{-2}$-net of the $S^{d-1}$ rather than a $m^{-2}$-net. Then, following the same steps of the proof of Lemma 6, with probability at least $1-L\exp(\O(dL\log(m)) - \Omega(m^{\frac{1}{3}})))$ over initialization $\bW(0)$, there holds
\begin{align}
        \calE_1^\ell \lesssim  {L^3}{m^{-\frac{1}{3}}}.\nonumber
    \end{align}

\noindent\textbf{Estimates of $\calE_2^\ell$}:
Similar to the proof of the estimate of $\calE_1$, by using more finer estimates $\sup_{\bx}\|o^\ell_0(\bx)\|_2 \le C$ and $\sup_{\bx\in\X}\|\bV^\ell_{L,0}(\bx)\|_{op} \le CLm^{-\frac{1}{2}}$, follows the same proof steps of Lemma 7 in \cite{xu2024overparametrized}, we can show that
\begin{align*}
    \calE_2^\ell \lesssim \frac{L^2}{m^{\frac{1}{3}}}.
\end{align*}

\noindent\textbf{Estimates of $\calE_3^\ell$}: 
Similar to the above arguments, we use the estimates $\sup_{\bx}\|o^\ell_0(\bx)\|_2 \le C$ and $\sup_{\bx\in\X}\|\bV^\ell_{L,0}(\bx)\|_{op} \le CLm^{-\frac{1}{2}}$ to improve the proof of Lemma 8 in \cite{xu2024overparametrized} and get
\begin{align*}
    \calE_3^\ell   \lesssim \frac{\sqrt{L}}{m^{\frac{1}{6}}} + \sqrt{\frac{dL\log(m)}{m}} + \frac{L^2}{m^{\frac{1}{3}}}.
\end{align*}
Combining the above estimates of $\calE_1^\ell,\calE_2^\ell,\calE_3^\ell$ completes the proof of this lemma. 
\end{proof}

\section{Conclusion}\label{sec:conclu}
In this paper, we prove that both GD and SGD with deep ReLU networks can achieve the minimax-optimal rates $\O(n^{-\frac{2\beta}{2\beta+\gamma}})$ of the excess risk when the network width satisfies $m\gtrsim \text{Poly}(L,n,d)$. Our results indicate that gradient descent methods with deep ReLU networks can achieve generalization performance that is at least comparable to classical gradient methods in the kernel setting. Several directions for future study remain. First, extending our results to deep networks with smooth activation would be a valuable next step. It would also be interesting to broaden the analysis to other architectures, such as convolutional networks and residual networks.

\setlength{\bibsep}{0.16cm}
\bibliographystyle{plainnat}
\bibliography{learning.bib}
\end{document}